\documentclass[lettersize,journal]{IEEEtran}
\IEEEoverridecommandlockouts
\usepackage{cite}
\usepackage{amsmath,amssymb,amsfonts,amsthm}
\usepackage{pifont}
\usepackage[ruled,vlined]{algorithm2e}
\SetKwInput{KwInput}{Input}
\SetKwInput{KwOutput}{Output}
\usepackage{caption}
\usepackage[labelformat=simple]{subcaption}
\usepackage[labelsep=period]{caption}

\DeclareCaptionLabelFormat{mylabel}{#1~#2.\quad}
\usepackage{float}
\usepackage{multirow}
\usepackage{graphicx}
\usepackage{textcomp}
\usepackage{hyperref}
\usepackage{xcolor}
\usepackage{tabularx,booktabs}
\newcolumntype{Y}{>{\centering\arraybackslash}X}
\usepackage[scr=boondox, cal=esstix]{mathalpha}
\usepackage{moresize}

\DeclareMathOperator*{\argmin}{arg\,min}
\newcommand{\cmark}{\ding{51}}
\newcommand{\xmark}{\ding{55}}
\usepackage{orcidlink}
\def\BibTeX{{\rm B\kern-.05em{\sc i\kern-.025em b}\kern-.08em
    T\kern-.1667em\lower.7ex\hbox{E}\kern-.125emX}}
\newtheorem{theorem}{Theorem}
\usepackage{fancyhdr}
\begin{document}

\fancyhead[L]{\footnotesize This paper has been accepted for publication in \textit{IEEE Transactions on Automation Science and Engineering}. DOI: \href{https://doi.org/10.1109/TASE.2025.3575237}{10.1109/TASE.2025.3575237}}

\clearpage
\thispagestyle{empty}
\begin{minipage}{\textwidth}
\Huge
IEEE Copyright Notice \\
\large
\\
\textcopyright~2025 IEEE. Personal use of this material is permitted. Permission from IEEE must be obtained for all other uses, in any current or future media, including reprinting/republishing this material for advertising or promotional purposes, creating new collective works, for resale or redistribution to servers or lists, or reuse of any copyrighted component of this work in other works.\\

This paper has been accepted for publication in \textit{IEEE Transactions on Automation Science and Engineering}. DOI: \href{https://doi.org/10.1109/TASE.2025.3575237}{10.1109/TASE.2025.3575237}

\normalsize
\end{minipage}
\newpage

\clearpage
\thispagestyle{fancy}
\setcounter{page}{1}

\title{MEF-Explore: Communication-Constrained Multi-Robot Entropy-Field-Based Exploration}

\author{Khattiya 
       Pongsirijinda\textsuperscript{\orcidlink{0000-0002-0168-8156}},
       Zhiqiang Cao\textsuperscript{\orcidlink{0000-0002-0823-7155}},
       Billy Pik Lik Lau\,\textsuperscript{\orcidlink{0000-0001-5133-2791}},~\IEEEmembership{Senior Member,~IEEE},\\
       Ran Liu\,\textsuperscript{\orcidlink{0000-0002-6343-4645}},~\IEEEmembership{Senior Member,~IEEE},
       Chau Yuen\,\textsuperscript{\orcidlink{0000-0002-9307-2120}},~\IEEEmembership{Fellow,~IEEE}, and U-Xuan Tan\,\textsuperscript{\orcidlink{0000-0002-5757-1379}},~\IEEEmembership{Senior Member,~IEEE}
\thanks{Corresponding author: Chau Yuen.
\\ \indent K. Pongsirijinda, Z. Cao, B. P. L. Lau, and U-X. Tan are with the Engineering Product Development, Singapore University of Technology and Design, Singapore 487372 {(email: \{khattiya\_pongsirijinda, zhiqiang\_cao\}@mymail.sutd.edu.sg, \{billy\_lau, uxuan\_tan\}@sutd.edu.sg}). R. Liu and C. Yuen are with the School of Electrical and Electronic Engineering, Nanyang Technological University, Singapore 639798 (email: \{ran.liu, chau.yuen\}@ntu.edu.sg).
}}

\maketitle

\begin{abstract}
Collaborative multiple robots for unknown environment exploration have become mainstream due to their remarkable performance and efficiency. However, most existing methods assume perfect robots' communication during exploration, which is unattainable in real-world settings. Though there have been recent works aiming to tackle communication-constrained situations, substantial room for advancement remains for both information-sharing and exploration strategy aspects. In this paper, we propose a Communication-Constrained \underline{M}ulti-Robot \underline{E}ntropy-\underline{F}ield-Based Exploration (MEF-Explore). The first module of the proposed method is the two-layer inter-robot communication-aware information-sharing strategy. A dynamic graph is used to represent a multi-robot network and to determine communication based on whether it is low-speed or high-speed. Specifically, low-speed communication, which is always accessible between every robot, can only be used to share their current positions. If robots are within a certain range, high-speed communication will be available for inter-robot map merging. The second module is the entropy-field-based exploration strategy. Particularly, robots explore the unknown area distributedly according to the novel forms constructed to evaluate the entropies of frontiers and robots. These entropies can also trigger implicit robot rendezvous to enhance inter-robot map merging if feasible. In addition, we include the duration-adaptive goal-assigning module to manage robots' goal assignment. The simulation results demonstrate that our MEF-Explore surpasses the existing ones regarding exploration time and success rate in all scenarios. For real-world experiments, our method leads to a 21.32\% faster exploration time and a 16.67\% higher success rate compared to the baseline.
\end{abstract}

\def\abstractname{Note to Practitioners}
\begin{abstract}
This paper is motivated by the problem of multi-robot exploration in unknown environments with communication constraints regarding information sharing between robots. The majority of existing methods have yet to consider that robots cannot always fully communicate when working in actual situations. To fill this gap, in this paper, we propose a new exploration method in which robots can share necessary information based on their communication status and explore the area efficiently based on the entropies of frontiers and robots. Experimental evaluation shows that the proposed approach leads to fast and successful exploration, which will be practical for real-world exploration tasks, such as in large buildings, warehouses, urban areas, and complex environments.
\end{abstract}

\begin{IEEEkeywords}
Communication constraints, cooperating robots, distributed exploration, entropy-field-based exploration, information sharing, multi-robot exploration.
\end{IEEEkeywords}

\section{Introduction}
\label{section:intro}

\begin{figure}
  \begin{center}
  \includegraphics[width=0.95\linewidth]{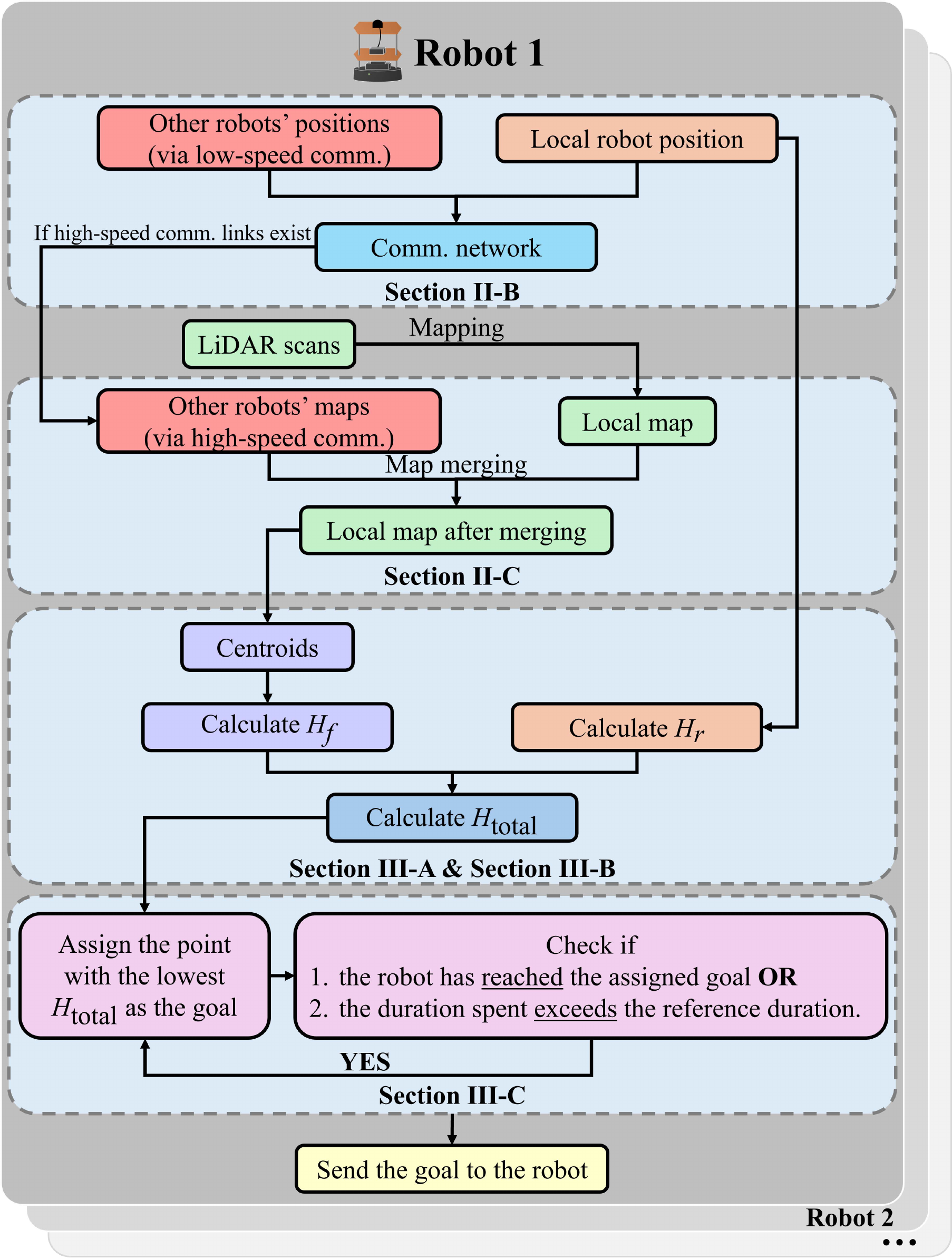}
  \caption{Overview of our proposed MEF-Explore}
  \vspace*{-6mm}
  \label{fig:exp-overview}
  \end{center}
\end{figure}

\IEEEPARstart{C}{ollaborative} multi-robot exploration systems have emerged as a prevailing approach for autonomous unknown area exploration owing to significant advancements and practical capabilities in various domains, such as search and rescue operations \cite{SAR, DIBNN}, environmental monitoring and patrolling \cite{Monitor, Patrol}, and several types of automation. So far, diverse methods have been proposed to enhance exploration performance; for instance, the conventional rapidly exploring random tree (RRT) strategy and its related developed strategies \cite{RRT, TempRRT, MultiRoom}, potential-field-based strategies \cite{MWF-CN, MMPF}, entropy-based strategies \cite{Entropy-Multi2015, Entropy-Multi2018, Entropy-Multi2019}, hierarchical-based strategies \cite{Hier, CURE}, and reinforcement-learning-based (RL-based) strategies \cite{RL2020, RL2022, CQLite}. Although the improvements are clearly shown by the exploration efficiency of the mentioned methods, most assume flawless communication between robots throughout the exploration, which is difficult to ensure in real-world scariness when infrastructure is unavailable.

\begin{table*}[htbp]
    \scriptsize
    \centering
    \caption{Communication-Constrained Multi-Robot Exploration Methods}
\begin{tabular}{lcccccc}
    \hline
    \multicolumn{1}{|c|}{{\textbf{Method}}}                              & \multicolumn{1}{c|}{\textbf{{Robot type}}}    & \multicolumn{1}{c|}{\textbf{\begin{tabular}[c]{@{}c@{}}Real-robot\\ deployment\end{tabular}}}   & \multicolumn{1}{c|}{{\textbf{Mapping}}}   & \multicolumn{1}{c|}{{\textbf{Approach}}}   & \multicolumn{1}{c|}{\textbf{{Exploration strategy}}}   & \multicolumn{1}{c|}{\textbf{{Communication requirement}}}                                                  \\ \hline
    \multicolumn{1}{|l|}{MEF-Explore (Proposed)}                             & \multicolumn{1}{c|}{UGVs}                                          & \multicolumn{1}{c|}{\cmark}                                                      & \multicolumn{1}{c|}{2D}                                        & \multicolumn{1}{c|}{Distributed}                                & \multicolumn{1}{c|}{Entropy-field-based}                                    & \multicolumn{1}{c|}{\begin{tabular}[c]{@{}c@{}}Low-speed comm.: Continuous\\ High-speed comm.: Opportunistic\end{tabular}}   \\ \hline
    \multicolumn{1}{|l|}{$\text{IR}^2$ \cite{IR2}}                           & \multicolumn{1}{c|}{UGVs}                                          & \multicolumn{1}{c|}{\cmark}                                                      & \multicolumn{1}{c|}{2D}                                        & \multicolumn{1}{c|}{Decentralized}                              & \multicolumn{1}{c|}{Deep-RL-based}                                          & \multicolumn{1}{c|}{Event-based}                                                                                             \\ \hline
    \multicolumn{1}{|l|}{MOCHA \cite{MOCHA}}                           & \multicolumn{1}{c|}{UGVs, UAVs}                              & \multicolumn{1}{c|}{\cmark}                                                & \multicolumn{1}{c|}{2D}                                  & \multicolumn{1}{c|}{Distributed}                          & \multicolumn{1}{c|}{Waypoint-based}                                   & \multicolumn{1}{c|}{Opportunistic}                                                                                     \\ \hline
    \multicolumn{1}{|l|}{Pursuit \cite{M-TARE}}                              & \multicolumn{1}{c|}{UGVs, UAVs}                                    & \multicolumn{1}{c|}{\cmark}                                                      & \multicolumn{1}{c|}{3D}                                        & \multicolumn{1}{c|}{Decentralized}                              & \multicolumn{1}{c|}{Hierarchical}                                           & \multicolumn{1}{c|}{Event-based}                                                                                             \\ \hline
    \multicolumn{1}{|l|}{Bartolomei's \cite{Fast-Forest}}                    & \multicolumn{1}{c|}{UAVs}                                          & \multicolumn{1}{c|}{\xmark}                                                      & \multicolumn{1}{c|}{3D}                                        & \multicolumn{1}{c|}{Decentralized}                              & \multicolumn{1}{c|}{Cost-based}                                             & \multicolumn{1}{c|}{Opportunistic}                                                                                           \\ \hline
    \multicolumn{1}{|l|}{RACER \cite{RACER}}                                 & \multicolumn{1}{c|}{UAVs}                                          & \multicolumn{1}{c|}{\cmark}                                                      & \multicolumn{1}{c|}{3D}                                        & \multicolumn{1}{c|}{Decentralized}                              & \multicolumn{1}{c|}{Hierarchical}                                           & \multicolumn{1}{c|}{Opportunistic}                                                                                           \\ \hline
    \multicolumn{1}{|l|}{DME-LC \cite{DME-LC}}                               & \multicolumn{1}{c|}{UUVs}                                          & \multicolumn{1}{c|}{\xmark}                                                      & \multicolumn{1}{c|}{3D}                                        & \multicolumn{1}{c|}{Decentralized}                              & \multicolumn{1}{c|}{Information-gain-based}                                 & \multicolumn{1}{c|}{Continuous}                                                                                              \\ \hline
    \multicolumn{1}{|l|}{Bramblett's \cite{CM-ERT}}                          & \multicolumn{1}{c|}{UGVs}                                          & \multicolumn{1}{c|}{\cmark}                                                      & \multicolumn{1}{c|}{2D}                                        & \multicolumn{1}{c|}{Decentralized}                              & \multicolumn{1}{c|}{Cost-based}                                             & \multicolumn{1}{c|}{Event-based}                                                                                             \\ \hline
    \multicolumn{1}{|l|}{Meeting-Merging-Mission \cite{MeetMergeMission}}    & \multicolumn{1}{c|}{UGVs, UAVs}                                    & \multicolumn{1}{c|}{\cmark}                                                      & \multicolumn{1}{c|}{3D}                                        & \multicolumn{1}{c|}{Decentralized}                              & \multicolumn{1}{c|}{Mission-based}                                          & \multicolumn{1}{c|}{Event-based}                                                                                             \\ \hline
    \multicolumn{1}{|l|}{MR-TopoMap \cite{MRTopo}}                           & \multicolumn{1}{c|}{UGVs}                                          & \multicolumn{1}{c|}{\cmark}                                                      & \multicolumn{1}{c|}{2D}                                        & \multicolumn{1}{c|}{Distributed}                                & \multicolumn{1}{c|}{Topology-based}                                         & \multicolumn{1}{c|}{Event-based}                                                                                             \\ \hline
    \multicolumn{1}{|l|}{GVGExp \cite{GVGExp}}                               & \multicolumn{1}{c|}{UGVs}                                          & \multicolumn{1}{c|}{\xmark}                                                      & \multicolumn{1}{c|}{2D}                                        & \multicolumn{1}{c|}{Distributed}                                & \multicolumn{1}{c|}{Topology-based}                                         & \multicolumn{1}{c|}{Event-based}                                                                                             \\ \hline
    \multicolumn{1}{|l|}{Klaesson's \cite{InterConn}}                  & \multicolumn{1}{c|}{UGVs}                                    & \multicolumn{1}{c|}{\xmark}                                                & \multicolumn{1}{c|}{2D}                                  & \multicolumn{1}{c|}{Centralized}                          & \multicolumn{1}{c|}{Cost-based}                                       & \multicolumn{1}{c|}{Opportunistic}                                                                                     \\ \hline
    \multicolumn{7}{l}{\begin{tabular}[c]{@{}l@{}}Note: UGVs stands for unmanned ground vehicles, UAVs stands for unmanned aerial vehicles, UUVs stands for unmanned underwater vehicles, and comm. is short\\ for communication. For the exploration approaches, both decentralized and distributed do not have any central authority controlling the entire system. Still, the dec-\\ entralized ones have a control center for passing some commands to robots, such as rendezvous in some methods. On the other hand, the controls of the distributed\\ methods are fully separated on each robot.\end{tabular}}
\end{tabular}
\vspace*{-4mm}
\label{table:related-works}
\end{table*}

According to the survey \cite{Survey2017} and the review \cite{Review2022}, communication for information exchange between robots is a crucial component of multi-robot exploration. Lack of good quality for such communication can result in significantly deteriorated exploration performance. Some works, such as \cite{DSE}, have considered distributed state estimation using intermittently connected multi-robot systems, allowing them to exchange information when they are close to each other. However, there is not much research investigating multi-robot exploration under communication-constrained situations to date. Some of the works presented in \cite{Survey2017} do not consider any explicit assumption regarding robot communication capability, which makes it challenging to design a suitable exploration method to match the communication conditions. Meanwhile, some centralized approaches require a base station for information relays, which can be unsuccessful when facing the issues of a single point of failure and communication overhead, especially in case of unreliable communication. Most of them, for example \cite{Explo2006, TreeExplo, TabooList, MeetOrNot}, are also mainly focused on simulation, so they have not yet devised their methods for practical real-robot deployment.

Apart from the stated works, other novel methods, \cite{MeetMergeMission, RACER, Fast-Forest, M-TARE}, also have recently been proposed, predominantly targeting unmanned aerial vehicles (UAVs). One of the works \cite{MOCHA} has also presented an exploration strategy specifically for the collaborative deployment of both UAVs and UGVs with opportunistic communication. Some techniques,  \cite{MeetMergeMission, CM-ERT, M-TARE}, tend to be semi-distributed since rounds of centralized rendezvous are mandatory for map exchange. Although in \cite{IR2}, robot rendezvous is decentralized, the exploration performance can occasionally be suboptimal since the rendezvous pattern entirely relies on what it has learned from specific training maps. Additionally, in \cite{GVGExp, DME-LC, MRTopo}, strict conditions are applied to control robots' communication manner, and these can sometimes lead to lengthy and unaccomplished exploration. Again, some methods, such as \cite{InterConn, GVGExp, DME-LC, Fast-Forest}, still concentrate exclusively on simulation so that more modifications might be needed for advanced real-world exploration tasks. We summarize the details of each method in Table \ref{table:related-works}. Note that their communication requirements are categorized similarly to those in \cite{Survey2017}. Hence, in keeping with the abovementioned gaps, these allude to the need for a distributed multi-robot exploration method for UGVs that can work efficiently within the confines of communication constraints in real-world scenarios.

Therefore, in the context of this study, we address the problem of efficiently coordinating multiple autonomous homogeneous robots to explore unknown environments under communication-constrained situations. The first goal is to ensure effective information sharing according to the available communication, as different types of information, such as robot positions and maps, necessitate varying communication capabilities. Since communication constraints can exacerbate the difficulties in exploration, the second goal is to achieve fast and successful multi-robot exploration in a distributed manner. Accomplishing these requires thorough analysis and carefully crafted strategic designs.

With these objectives in mind, we propose a Communication-Constrained \textbf{M}ulti-Robot \textbf{E}ntropy-\textbf{F}ield-Based Exploration (MEF-Explore), consisting of two key components: information-sharing and distributed exploration strategies. Firstly, we present a two-layer inter-robot communication-aware information-sharing strategy to handle unsteady communication status. We adopt the perception of dynamic graphs to portray the exploring robots as the graph nodes. Then, we consider the edge between them formed as the availability of high-speed communication, in which robots can exchange their positions and local maps. Conversely, if there is no edge linking between nodes, the robots can only communicate via low-speed communication, where they are only allowed to share their positions. The existence of the mentioned edges is based on a specific range, which will be explained later in this paper. Secondly, we introduce a novel entropy-field-based exploration strategy for robots to decide where to explore next in the unknown environment. The intuition behind this is mainly from the uncertainty affecting frontiers during exploration, which is caused by the nature of unknown areas and communication constraints. Here, we adopt the Shannon entropy introduced in \cite{Shannon-1, Shannon-2} to measure such uncertainty and integrate it with the attractive effects inspired by the potential field to construct novel reformed entropies. This way, each location in the environment is embedded with the composite of entropy and attractiveness, which substantially reflects if it is worth exploring for robots. We also intentionally formulate our entropies to activate some rendezvous at appropriate times based on each robot's individual exploration to elevate inter-robot map merging. In addition, to foster better exploration, we present a duration-adaptive goal-assigning module for evaluating the feasibility of each robot's goal candidates. For clarification, the simplified overview is illustrated in Fig. \ref{fig:exp-overview}. In summary, the main contributions of our proposed MEF-Explore are as follows:

\begin{itemize}
    \item To deal with communication constraints between robots, we propose a two-layer inter-robot communication-aware information-sharing strategy, where inter-robot map merging opportunistically happens when robots are within high-speed communication range and position exchange over low-speed communication is always possible.
    \item To handle the uncertainty of frontiers in each robot's map under communication-constrained situations, we propose a novel distributed entropy-field-based exploration strategy, combining the concepts of entropy and potential fields for efficient communication-constrained exploration. The duration-adaptive goal-assigning module is also designed to strategically control robots to explore the environment without overloading goals in a short period.
    \item We implement the MEF-Explore and conduct simulation and real-world experiments to compare its performance with other exploration methods regarding exploration time and success rate. For clarification, a video of the experiments has been attached as supplementary material.
    % \item We implement the MEF-Explore and conduct simulation and real-world experiments to compare its performance with other exploration methods regarding exploration time and success rate. For clarification, a video of the experiments can be accessed at \url{https://youtu.be/HognvtZsmBs}.
\end{itemize}

The novelties of our proposed MEF-Explore, distinguishing it from previous works, stem directly from the aforementioned contributions in both theoretical and practical aspects. Firstly, robots using the proposed information-sharing strategy effectively share their positions continuously via low-speed communication and maps opportunistically via high-speed communication. Unlike the previous works, our novel strategy considers the different capabilities of communication between robots during exploration, as shown in Table \ref{table:related-works}.Secondly, the entropies, which play the leading role in the proposed exploration strategy, are novelly constructed by combining Shannon entropy with different forms of attractive potentials. To the best of our knowledge, our novel entropies have never been introduced before, as other existing methods are mainly based on the variants of Shannon entropy \cite{Entropy-Multi2015, Entropy-Vallve2015, Entropy-Multi2018, Entropy2020}, the compound of Shannon and Rényi entropies \cite{Entropy-Carrillo2015, Entropy-Carrillo2017}, or the Fisher information \cite{Entropy-Multi2019}. Robots using the proposed strategy have better exploration performance under communication-constrained situations. Specifically, they select goals to explore the unknown area efficiently and rendezvous to share their maps at appropriate times without predetermination.

The upcoming sections will proceed as follows: Firstly, in Section \ref{section:info-sharing}, we describe the communication constraints considered in this work and then elucidate our proposed information-sharing strategy on how each robot communicates during area exploration. The exploration problem and our proposed entropy-field-based exploration strategy will then be explained in Section \ref{section:exploration}. Section \ref{section:simulation} will provide the evaluation metrics, simulation results, and related analyses. Then, we detail how we conduct the real-world deployment and discuss the experimental results in Section \ref{section:real-world}. Section \ref{section:challenges} will discuss the challenges that may occur during simulation and real-world deployment. Finally, in Section \ref{section:conclusion}, we provide the conclusion of this paper and the potential directions for future research.

\section{Two-Layer Inter-Robot Communication-Aware Information-Sharing Strategy}
\label{section:info-sharing}

\begin{figure}
  \begin{center}
  \includegraphics[width=0.9\linewidth]{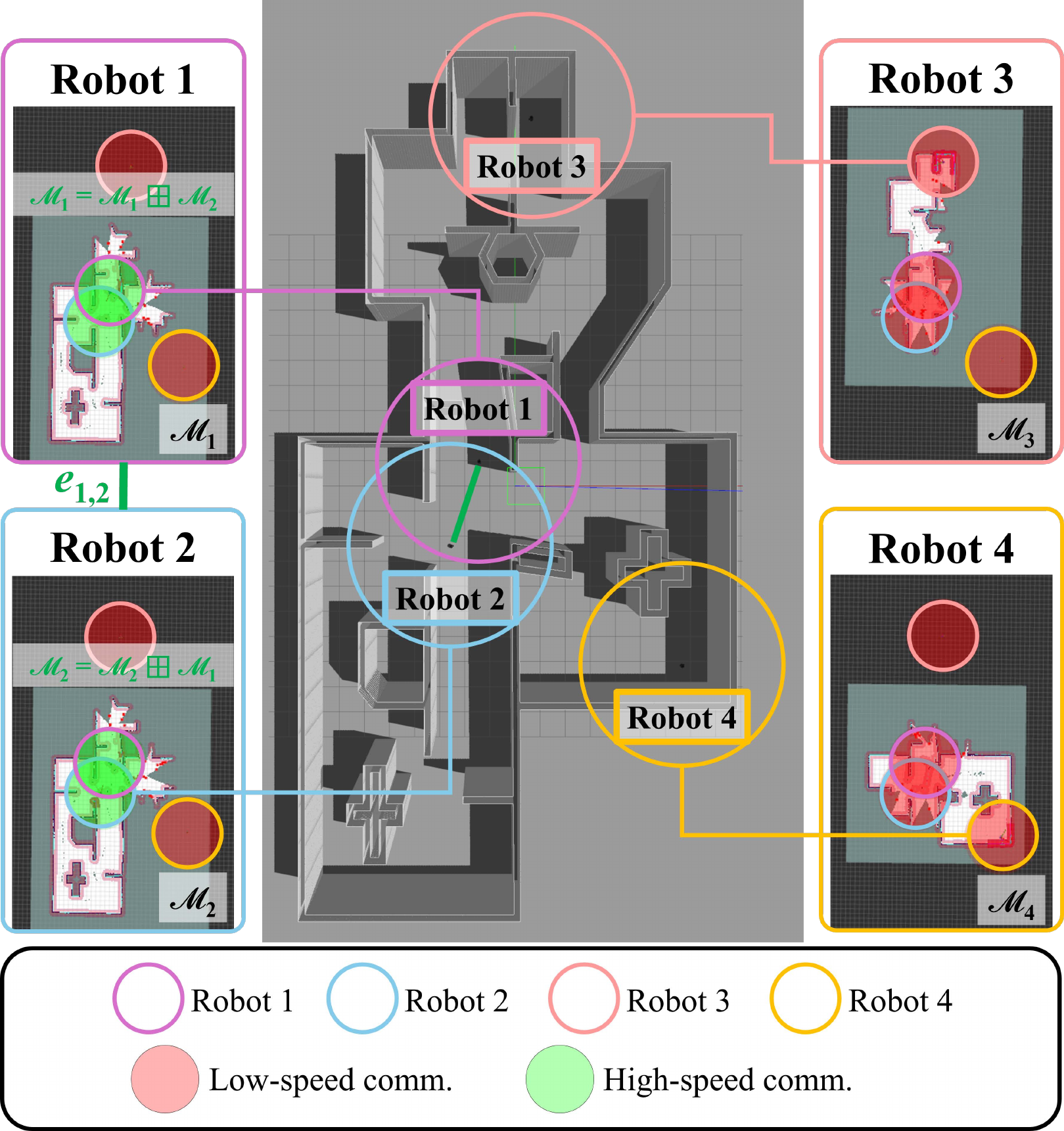}
  \setlength{\belowcaptionskip}{-10pt}
  \caption{Graphical representation of the proposed information-sharing strategy during four-robot exploration. $\mathscr{M}_i$ is each robot $i$'s local map, $\boxplus$ is the map merging operator, and the visualization of each robot is based on the current communication status between that robot and the others. Each circle represents the range of each robot's $r_\mathrm{comm}$. For instance, considering the top-left visualization for Robot 1's point of view, we can see the green-filled circle of Robot 2, which means a high-speed communication link $e_{1,2}$ is established between them. Meanwhile, Robot 3 and Robot 4 are red-filled, which means Robot 1 communicates with them via low-speed communication only.}
  \label{fig:info-sharing}
  \end{center}
\end{figure}

In this section, we provide details of the information-sharing strategy proposed for communication-constrained exploration. The first subsection will give the definition of communication constraints considered in this work. In the second subsection, we will explain the construction of the communication network, and in the third subsection, we will describe how the maps from different robots are merged based on their communication.

\subsection{Communication Constraints}
In this work, we focus on the communication constraints that prevent robots from always being able to share large amounts of information with each other during exploration. Specifically, all robot positions are shareable via ever-ready low-speed communication, but for the more extensive information like maps, they can only be shared via high-speed communication, which is available when the robots are within the specific range called $r_\mathrm{comm}$. To ensure continuous exploration without communication interruption, we propose an information-sharing strategy to control the sharing behavior of different robots' information based on the two layers of communication: low-speed and high-speed, which will be explained further in the following subsections.

\subsection{Two-Layer Communication Network}
For clarification, we represent the communication network as a dynamic graph to facilitate direct communication between robots. Formally, let $G=(V,E)$ be a graph with moving robots as nodes in $V = \{ 1, 2, ..., N_r\}$ where $N_r$ is the total number of robots. For any $i, j \in V$, each edge $e_{i,j} \in E$ connects robot $i$ and robot $j$ and acts for the high-speed communication between the corresponding robots. This edge is added to or removed from $G$ based on the following conditions:
\begin{align}
    &[d_\mathrm{cur} (i,j) < r_\mathrm{comm}] \land [e_{i,j} \notin E] \notag\\
    &\Rightarrow [E = E \cup \{ e_{i,j} \}] \land [G = (V, E)] \text{,} \\
    &[d_\mathrm{cur} (i,j) \geq r_\mathrm{comm}] \land [e_{i,j} \in E] \notag\\ &\Rightarrow [E = E \setminus \{ e_{i,j} \}] \land [G = (V, E)] \text{,}
\end{align}
where $d_\mathrm{cur}(i,j) = d(\mathrm{Pos}_\mathrm{cur}(i),\mathrm{Pos}_\mathrm{cur}(j))$ is the Euclidean distance between the robot $i$’s and robot $j$'s current positions.

Using a dynamic graph to represent the communication network provides a clear understanding of the current status of inter-robot communication during exploration. As shown in Fig. \ref{fig:info-sharing}, we can perceive the edge $e_{1,2}$ linked between Robot 1 and Robot 2, which means they are communicating at high speed. On the other hand, Robot 3 and Robot 4, which do not have any edges linked to any other robots, are under low-speed communication.

\subsection{Inter-Robot Map Merging Module}
\label{subsection:map-merging}
In this study, we utilize occupancy grid maps \cite{OGM} as the primary representation of the environment. By adhering to the widely used occupancy grid maps, we ensure that our proposed MEF-Explore is directly comparable with most existing works focused on exploration, which also use this map type. We define the robot $i$'s occupancy grid map $\mathscr{M}_i$ as follows:
\begin{equation}
    \mathscr{M}_i = \{ \mathscr{m}_{xy} \mid x \in \{1, 2, ..., N_\mathrm{row}\}, y \in \{1, 2, ..., N_\mathrm{col}\}\}\text{,}
\end{equation}
where $\mathscr{m}_{xy}$ is the grid cell at row $x$ and column $y$, and $N_\mathrm{row}$ and $N_\mathrm{col}$ are the numbers of rows and columns in the grid map, respectively. Note that the grid size depends on the map resolution. Each grid cell $\mathscr{m}_{xy}$ has a value $\mathrm{Occ}(\mathscr{m}_{xy})$ to indicate its occupancy where
\begin{itemize}
    \item $\mathrm{Occ}(\mathscr{m}_{xy}) = 0$: The grid cell is free, meaning it does not contain an obstacle.
    \item $\mathrm{Occ}(\mathscr{m}_{xy}) = 1$: The grid cell is occupied, meaning it contains an obstacle.
    \item $\mathrm{Occ}(\mathscr{m}_{xy}) = 0.5$: The grid cell is unknown.
\end{itemize}

\begin{algorithm}
\scriptsize
\DontPrintSemicolon
\SetAlgoCaptionSeparator{}
\SetKw{Continue}{continue}
  \KwInput{$N_r$: number of robots,
  \\ \qquad\quad $V = \{ i,j, ... \}$: the set of nodes (robots) of graph $G$,
  \\ \qquad\quad $E$: the set of edges of graph $G$
  \\ \qquad\quad $\mathrm{Pos}_\mathrm{cur}(i)$, $\mathrm{Pos}_\mathrm{cur}(j)$: current positions of robots $i$, $j$, and
  \\ \qquad\quad $\mathscr{M}_i, \mathscr{M}_j$: local maps of robots $i$, $j$}
  
  \KwOutput{$\mathscr{M}_i, \mathscr{M}_j$: local maps of robots $i$, $j$ after merging}
  
  \While{the exploration has not ended}{
  
  \For{any robots $i,j \in V$}
  {
    \uIf{$d_\mathrm{cur} (i,j) < r_\mathrm{comm}$ and $e_{i,j} \notin E$}{
    
        $E = E \cup \{ e_{i,j} \}$ 

        $G = (V, E)$
        
        $\mathscr{M}_i = \mathscr{M}_i \boxplus \mathscr{M}_j$ and $\mathscr{M}_j = \mathscr{M}_j \boxplus \mathscr{M}_i$
        
        }

    \uElseIf{$d_\mathrm{cur} (i,j) \geq r_\mathrm{comm}$ and $e_{i,j} \in E$}
    {
    
        $E = E \setminus \{ e_{i,j} \}$

        $G = (V, E)$
    
    }
    \Else{
        \Continue
    }
  }
  }
  
\KwRet{$\mathscr{M}_i, \mathscr{M}_j$}
\caption{Inter-robot map merging module}
\label{algo:comm}
\end{algorithm}

Moving on, we will delve into how maps of different robots are merged under communication-constrained situations. Let $\mathscr{M}_i$ and $\mathscr{M}_j$ be the local maps of robot $i$ and robot $j$, respectively. We define a new operator $\boxplus: \mathbb{R}^2 \times \mathbb{R}^2 \to \mathbb{R}^2$ to represent map merging, i.e., $\mathscr{M}_i \boxplus \mathscr{M}_j$ means merging $\mathscr{M}_j$ into $\mathscr{M}_i$. It is important to note the operator $\boxplus$ is not explicitly defined in this module, as its formulation depends on the specific map merger employed. This approach ensures the flexibility of this module, allowing it to accommodate a variety of map mergers without being limited by a predefined operator. We present some useful properties in Appendix \ref{appendix:prop}.

Within $r_\mathrm{comm}$ where high-speed communication is available between any robot $i$ and robot $j$, the edge $e_{i,j}$ exists in $G$. So, their maps will be updated with the $\boxplus$ product of both maps, which means
\begin{equation}
   \mathscr{M}_i = \mathscr{M}_i \boxplus \mathscr{M}_j \text{ and } \mathscr{M}_j = \mathscr{M}_j \boxplus \mathscr{M}_i \text{.}
\end{equation}
On the other hand, if the distance between robot $i$ and robot $j$ is beyond $r_\mathrm{comm}$, $e_{i,j}$ is removed. This transition signifies a shift in communication from high speed to low speed, where only the positions of the robots are shareable. In summary, inter-robot map merging is consistently initiated whenever high-speed communication is available based on the distance between them. The process is presented in Algorithm \ref{algo:comm}. 

It is to be noted that this strategy enables each robot to opportunistically share its map information without the need for predetermined events. With reference to Fig. \ref{fig:info-sharing}, Robot 1 and Robot 2 merge their maps when they are under high-speed communication. Also, based on the properties of our operator $\boxplus$, if a group of robots assembles as a connected subgraph, we recognize that their local maps are identical, even though some of them may not be directly linked together in the subgraph.

\section{Entropy-Field-Based Exploration Strategy}
\label{section:exploration}

Before introducing the proposed exploration strategy, we first define the exploration problem in the context of this work, which involves coordinating a group of $N_r$ homogeneous robots with identical capabilities and sensor range $d_s$ to explore an unknown environment efficiently under communication-constrained situations. As explained in Section \ref{section:info-sharing}, each robot has an occupancy grid map that is the outcome of the map from its individual exploration, merged with maps received from other robots when high-speed communication is available. Resultantly, different robots will not have the same maps during exploration. Therefore, the main objective is for at least one robot to have a fully explored map without any frontiers left using minimum time spent. Since frontiers are the boundary between explored and unexplored areas, their absence implies that all grid cells in the map have been classified as either free or occupied, with no cells remaining unknown, meaning all grid cells have an occupancy value equal to either 0 or 1, not 0.5. Hence, formally, the corresponding optimization problem is as follows:
\begin{equation}
    \text{Minimize } T \text{ subject to } \sum_{i=1}^{N_r} \gamma_i(T) \geq 1\text{,}
\end{equation}
where $\gamma_i(T)$ is the binary variable indicating whether robot $i$ achieves a fully explored map at time $T$. We have
\begin{equation}
    \gamma_i(T) = 1 \iff \forall \mathscr{m}_{xy} \in \mathscr{M}_i, \mathrm{Occ}(\mathscr{m}_{xy}) \neq 0.5\text{.}
\end{equation}
If otherwise, we have $\gamma_i(T) = 0$, meaning some grid cells in its map are still unknown.

In the following subsections, the proposed exploration strategy will be explained thoroughly, including components of our novel entropies and the module for the robot's goal assignment. 

\begin{figure*}
  \begin{center}
  \includegraphics[width=0.75\linewidth]{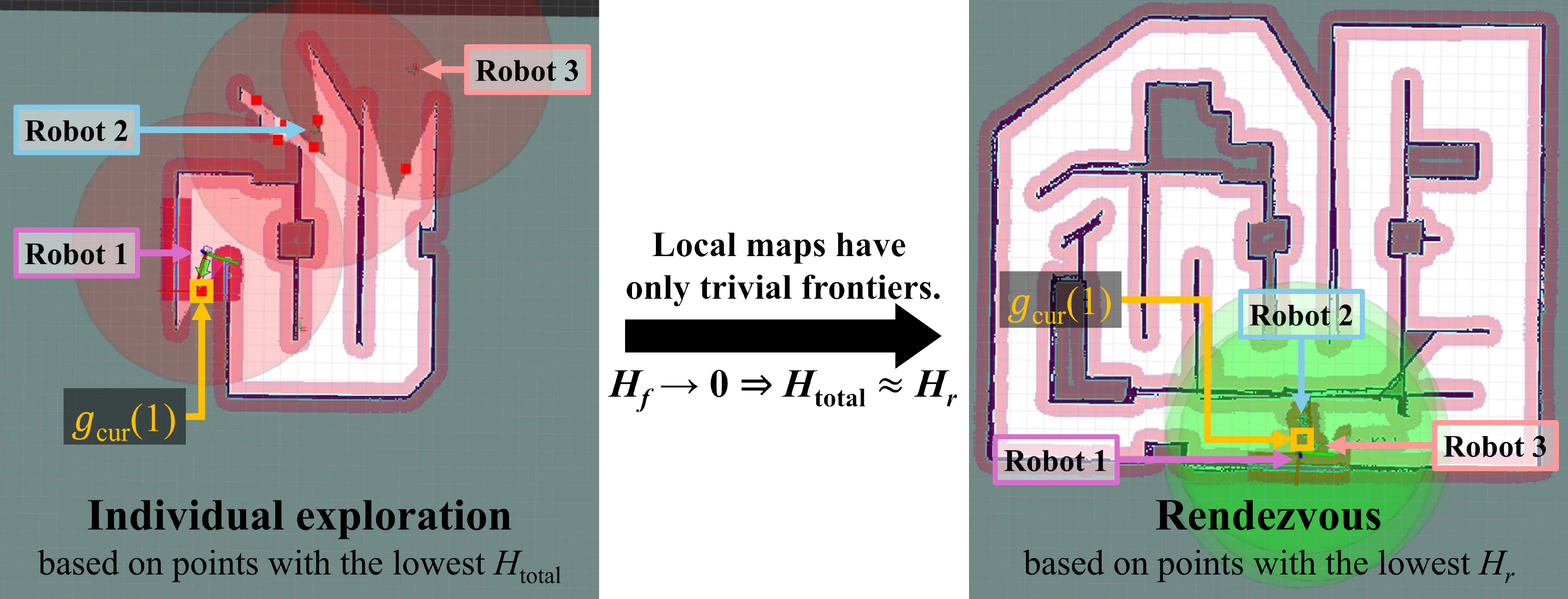}
  \setlength{\belowcaptionskip}{-10pt}
  \caption{Graphical representation of three-robot exploration by the proposed entropy-field-based exploration strategy from Robot 1's point of view when robots have individual exploration and then rendezvous. Red squares are frontier centroids, and ${g}_\mathrm{cur}(1)$ is the current goal of Robot 1.}
  \label{fig:exploration}
  \end{center}
\end{figure*}

\subsection{Total Entropy}
Initially, the continuity-based clustering method \cite{MMPF} is applied to the detected frontiers to generate the frontier centroids, which are the central points of the clusters obtained from clustering the boundaries between explored and unexplored areas. Then, the entropy will be calculated accordingly. We define the total entropy of robot $i$ for centroid $q$ at point $p$ as follows:
\begin{equation} \label{eq:H_total}
    H_\mathrm{total}(i,p,q) = H_f(p,q) + H_r(i,p)\text{,}
\end{equation}
where $H_f(p,q)$ is the entropy of the frontier centroid $q$ at point $p$ and $H_r(i,p)$ is the entropy of robot $i$ at point $p$. These two entropies will be described in detail in Section \ref{subsection:H_f} and Section \ref{subsection:H_r}.

Our motivation behind eq. (\ref{eq:H_total}) is to reform the standard information entropy by integrating it with the potential field to be the core component of the proposed exploration strategy. In particular, entropy can efficiently reflect the level of uncertainty, so we utilize it for multi-robot exploration with communication constraints. At the same time, the potential field can produce a powerful attractiveness for robots to explore environments collaboratively. Our novel reformed entropies concentrate on the frontiers of environments during exploration. So, this formation makes our entropies differ from existing entropy forms aforementioned in Section \ref{section:intro}. 

\begin{algorithm}
\scriptsize
\DontPrintSemicolon
\SetAlgoCaptionSeparator{}
  \KwInput{$p$, $q$: current positions of point $p$ and centroid $q$, \\
  \qquad\quad $\mathrm{Pos}_\mathrm{cur}(i)$: current position of robot $i$, \\ \qquad\quad $\mathrm{Pos}_\mathrm{pre}(i)$: previous position of robot $i$ before assigning the goal, \\
  \qquad\quad $C_q$: the total number of frontiers in the cluster that $q$ belongs to, 
  \\ \qquad\quad $N_C$: the total number of centroids, and\\
  \qquad\quad $N_r$: number of robots}
  \KwOutput{${g}_\mathrm{cur}(i)$: current goal for robot $i$}
  
  \While{the exploration has not ended}{
  
  \For{$l \leftarrow  1$ \textbf{to} $8$ (number of neighbors)}
  {
    $H_f = 0$, $H_r = 0$
    
    // calculate the entropy of frontiers
    
    \For{$m  \leftarrow  1$ \textbf{to} $C_q$}{
        $H_f -= \frac{k_f(N_r)C_q}{d^*(p,q)} \log (N_C C_q)$
        
        }
        
    // calculate the entropy of robots
    
    \For{$n \leftarrow 1$ \textbf{to} $N_r$}{
        \If{$d_\mathrm{cur}(i,p) < d_s$} {
        $H_r += \frac{k_r \sigma_r N_r}{d_\mathrm{cur}(i,p) - d_s}\log N_r + \chi^p_i(\alpha,\sigma_d)$

        }
    }
    
    // calculate the total entropy
    
    $H_\mathrm{total} = H_f + H_r$
    }

  // assign the robot's goal

  $\mathrm{Goal}_i = ({g}_\mathrm{cur}(i), {g}_\mathrm{new}(i))$
  
  \For{$p$ in robot $i$'s local map}{
       \If{$p$ has the minimum $H_\mathrm{total}(i,p,q)$}{

            ${g}_\mathrm{new}(i) = p$
       
       }
       \If{$\mathrm{Pos}_\mathrm{cur}(i) = {g}_\mathrm{cur}(i)$ or $t(\mathrm{Pos}_\mathrm{pre}(i) \to {g}_\mathrm{cur}(i)) \geq t_\mathrm{ref}(i)$}{

            ${g}_\mathrm{cur}(i) = {g}_\mathrm{new}(i)$
        
       }
    }
    }
\KwRet{${g}_\mathrm{cur}(i)$}
\caption{Entropy-field-based exploration strategy}
\label{algo:exp}
\end{algorithm}

Based on our proposed approach illustrated in Fig. \ref{fig:exploration}, in normal circumstances, robots explore the environment depending on $H_\mathrm{total}$, which includes $H_f$ and $H_r$. Since, in general, $H_f \gg H_r$, $H_f$ plays a dominant role in $H_\mathrm{total}$ values. Consequently, we have robots traveling to the goals mainly based on $H_f$. However, in addition to inter-robot map merging, which arises opportunistically based on the proposed information-sharing strategy, we intend to enhance exploration performance by having robots meet to exchange maps at the appropriate time, i.e., when each of them finished individual exploration. Hence, rooted in our well-constructed entropy, robot rendezvous will be spontaneously triggered by the effect of $H_r$ when the robots can detect just a small number of frontiers in their local map during exploration. In other words, our rendezvous mechanism is designed to enable robots to merge their maps more frequently, complementing the usual map merging that occurs opportunistically. We provide formal proof of this rendezvous manner in Appendix \ref{appendix:rendezvous}. 

While $H_\mathrm{total}$ is being computed, our proposed duration-adaptive goal-assigning module will also determine which points should be the goal candidates for each robot. For better understanding, we present the overall process in Algorithm \ref{algo:exp}. In the following subsections, we will detail all the components of the proposed exploration strategy.

\subsection{Entropy of Frontiers}
\label{subsection:H_f}
We integrate the effects of the Shannon entropy and the attractive potential to encourage robots to move toward frontier centroids. First, we consider the Shannon entropy of frontiers $\mathbb{H}_f$ as follows:
\begin{align} \label{eq:Shannon_H_f}
    \mathbb{H}_f &= -\sum_{m=1}^{N_f} \mathrm{Pr}(f_m) \log \mathrm{Pr} (f_m) \notag\\
    &= N_f \cdot \frac{1}{N_f} \bigg(-\log \bigg( \frac{1}{N_f} \bigg) \bigg) = \log N_f\text{,}
\end{align}
where $(f_m)_{m=1}^{N_f}$ are the frontiers, $\mathrm{Pr}(f_m)$ is probability that $f_m$ is selected as the robot's goal, and $N_f$ is the total number of frontiers.
We then formulate the entropy of the frontier centroid $q$ at point $p$ by incorporating the attractive potential as a multiplicative factor preceding $\mathbb{H}_f$ as follows:
\begin{equation} \label{eq:H_f}
    H_f(p,q) = \text{\small $-\frac{k_f(N_r)C_q}{d^*(p,q)} \mathbb{H}_f$} = \text{\small $-\frac{k_f(N_r)C_q}{d^*(p,q)} \log (N_C C_q)$}\text{,}
\end{equation}
where $k_f(N_r)$ is a scale factor based on the number of robot $N_r$, $C_q$ is the total number of frontiers in the cluster that centroid $q$ belongs to, $d^*$ is the 8-sector modified wavefront distance presented in \cite{MWF-CN}, and $N_C$ is the total number of centroids. Note that a special case of $\mathbb{H}_f = \log N_f$ from eq. (\ref{eq:Shannon_H_f}), which is $\mathbb{H}_f = \log (N_C C_q)$, is used in eq. (\ref{eq:H_f}) because we aim to make $H_f(p,q)$ more distinctly reflect each centroid $q$.

\subsection{Entropy of Robots}
\label{subsection:H_r}
In addition to the entropy of frontiers, we also formulate the entropy of robots by considering each robot's Shannon entropy. In addition, this entropy also induces robots to rendezvous if only a deficient number of frontier centroids can be detected on their local maps, as proven in Appendix \ref{appendix:rendezvous}. So, we first consider the Shannon entropy of robots $\mathbb{H}_r$ as follows:
\begin{align}
    \mathbb{H}_r &= -\sum_{m=1}^{N_r} \mathrm{Pr}(r_m) \log \mathrm{Pr} (r_m) \notag\\
    &= N_r \cdot \frac{1}{N_r} \bigg(-\log \bigg( \frac{1}{N_r} \bigg) \bigg) = \log N_r\text{,}
\end{align}
where $(r_m)_{m=1}^{N_r}$ are the robots, $\mathrm{Pr}(r_m)$ is probability that $r_m$ is selected, and $N_r$ is the total number of robots. Here, as we mainly focus on the surrounding area within each robot's sensor range (e.g., laser scan from 2D LiDAR), we then define the entropy of robot $i$ at point $p$ as follows:

\begin{align} \label{eq:H_r}
    H_r(i,p) &=
    \left\{
    \begin{array}{ll}
        \frac{k_r \sigma_r N_r}{d_\mathrm{cur}(i,p) - d_s}\mathbb{H}_r + \chi^p_i(\alpha,\sigma_d)&d_\mathrm{cur}(i,p) < d_s\\
        0 & \text{otherwise,}\\
    \end{array} 
    \right. \notag \\
    & =
    \left\{
    \begin{array}{ll}
        \text{\footnotesize $\frac{k_r \sigma_r N_r}{d_\mathrm{cur}(i,p) - d_s}\log N_r + \chi^p_i(\alpha,\sigma_d)$} &d_\mathrm{cur}(i,p) < d_s\\
        0 & \text{otherwise,}\\
    \end{array} 
    \right.
\end{align}
where $k_r$ is the scale factor, $\sigma_r$ is the relaxation distance, $N_r$ is the number of robots, $d_\mathrm{cur}(i,p) = d(\mathrm{Pos}_\mathrm{cur}(i),p)$ is the Euclidean distance between the robot $i$’s current position and point $p$, and $d_s$ is the sensor range. Note that since fluctuations can help reduce the local optima issues that make robots stuck during exploration, $\chi^p_i(\alpha,\sigma_d)$, which is the colored noise with noise color $\alpha$ and variance $\sigma_d$, is added to $H_r(i,p)$ in the same manner as the one for robots' potential in the previous work \cite{MWF-CN}.

\subsection{Duration-Adaptive Goal-Assigning Module}
In some existing exploration approaches such as \cite{MWF-CN}, \cite{MMPF}, \cite{Bench}, the goals for robots are always calculated and unceasingly assigned to them while exploring. This practice sometimes causes robots to be stuck as the generic navigation stack is not designed to handle many concurrent goals simultaneously. On the other hand, if the new goal is given just when the robots have completely arrived at the previously set goal, the exploration will be sub-optimal as there can be other better choices of goals that robots find during the travel, but they are not assigned. Therefore, we present a duration-adaptive goal-assigning module to cope with the mentioned situations. Let $\mathrm{Goal}_i = ({g}_\mathrm{cur}(i), {g}_\mathrm{new}(i))$ be a sequence of candidates for current and next goals of robot $i$, respectively. Note that unless the exploration has ended, ${g}_\mathrm{new}(i)$ is always available since we have
\begin{equation} \label{eq:g_new}
    {g}_\mathrm{new}(i) = \argmin_p H_\mathrm{total}(i,p,q)\text{.}
\end{equation}
We assign ${g}_\mathrm{cur}(i) = {g}_\mathrm{new}(i)$ when at least one of the following conditions is met:
\begin{align}
    &\mathrm{Pos}_\mathrm{cur}(i) = {g}_\mathrm{cur}(i)\text{,}\\
    &t(\mathrm{Pos}_\mathrm{pre}(i) \to {g}_\mathrm{cur}(i)) \geq t_\mathrm{ref}(i)\text{,}
\end{align}
where $\mathrm{Pos}_\mathrm{cur}(i)$ is the current position of robot $i$, $\mathrm{Pos}_\mathrm{pre}(i)$ is the previous position of robot $i$ before assigning ${g}_\mathrm{cur}(i)$, $t(\mathrm{Pos}_\mathrm{pre}(i) \to {g}_\mathrm{cur}(i))$ is the duration that robot $i$ spent to travel from $\mathrm{Pos}_\mathrm{pre}(i)$ to ${g}_\mathrm{cur}(i)$, and $t_\mathrm{ref}(i)$ is the reference duration. Consider
\begin{align}
    t(\mathrm{Pos}_\mathrm{pre}(i) \to {g}_\mathrm{cur}(i)) &\geq \frac{\mathrm{TravelDist}(\mathrm{Pos}_\mathrm{pre}(i), {g}_\mathrm{cur}(i))}{v_\mathrm{max}(i)} \notag\\
    &\geq \frac{k_\mathrm{ref}d(\mathrm{Pos}_\mathrm{pre}(i), {g}_\mathrm{cur}(i))}{v_\mathrm{max}(i)}\text{,}
\end{align}
where $\mathrm{TravelDist}(\mathrm{Pos}_\mathrm{pre}(i), {g}_\mathrm{cur}(i))$ is the function that calculates the actual distance that robot $i$ traveled from $\mathrm{Pos}_\mathrm{pre}(i)$ to ${g}_\mathrm{cur}(i)$, $v_\mathrm{max}(i)$ is the maximum velocity of robot $i$, $k_\mathrm{ref} < 1$ is the scale factor, and $d(\mathrm{Pos}_\mathrm{pre}(i), {g}_\mathrm{cur}(i))$ is the Euclidean distance between $\mathrm{Pos}_\mathrm{pre}(i)$ and ${g}_\mathrm{cur}(i)$. Hence, we select
\begin{equation} \label{eq:t_ref}
    t_\mathrm{ref}(i) = \frac{k_\mathrm{ref}d(\mathrm{Pos}_\mathrm{pre}(i), {g}_\mathrm{cur}(i))}{v_\mathrm{max}(i)} \text{.}
\end{equation}

Additionally, we would like to highlight that other existing goal-assignment strategies differ from our proposed one in terms of purposes and formulation. For example, some methods \cite{Goal-exp2020, Goal2023} leverage frontiers to form graphs to decide which goals should be assigned among all the robots. Some have purpose-built goal allocation that also might not be suitable in communication-constrained situations, such as generating safe trajectories \cite{Goal-Safe2020}, preventing robots from sticking to unapproachable goals \cite{TempRRT}, and matching only among an equal amount of robots and goals within the communication range \cite{Goal2024}.

\section{Simulation}
\label{section:simulation}

\begin{figure}
     \centering
     \begin{subfigure}[b]{0.322\linewidth}
         \centering
         \includegraphics[width=\linewidth]{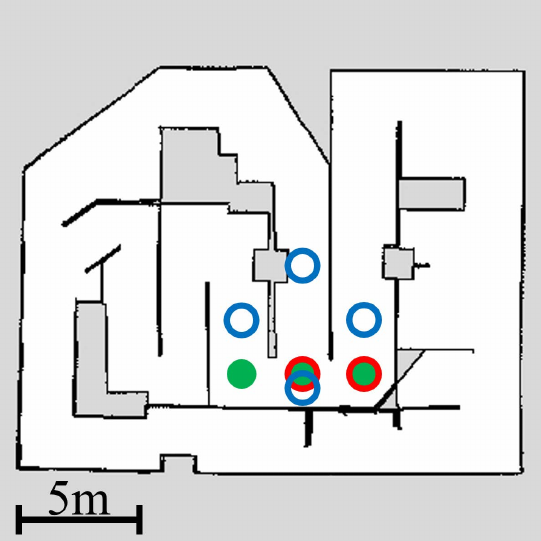}
         \caption{Map 1}
         \label{fig:map1}
     \end{subfigure}
     \begin{subfigure}[b]{0.6\linewidth}
         \centering
         \includegraphics[width=\linewidth]{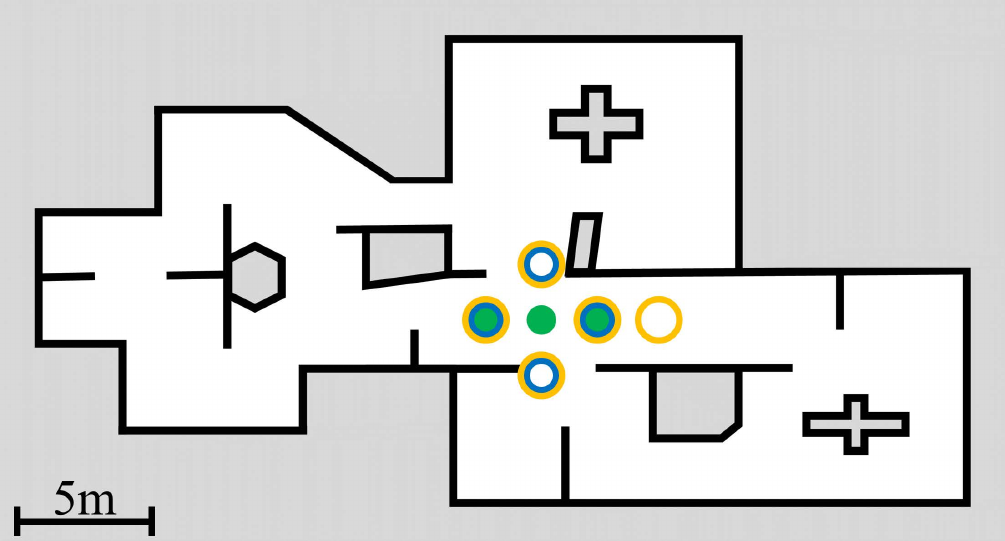}
         \caption{Map 2}
         \label{fig:map2}
     \end{subfigure}
     \begin{subfigure}[b]{\linewidth}
         \centering
         \includegraphics[width=\linewidth]{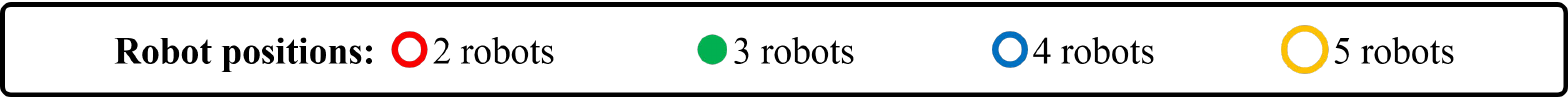}
     \end{subfigure}
        \caption{Environments for simulation. The red-outlined, green-filled, blue-outlined, and yellow-outlined circles indicate initial positions for two-robot, three-robot, four-robot, and five-robot exploration, respectively.}
        \label{fig:maps}
\end{figure}

\begin{figure*}
     \centering
     \begin{subfigure}[b]{0.9\linewidth}
         \centering
         \includegraphics[width=\linewidth]{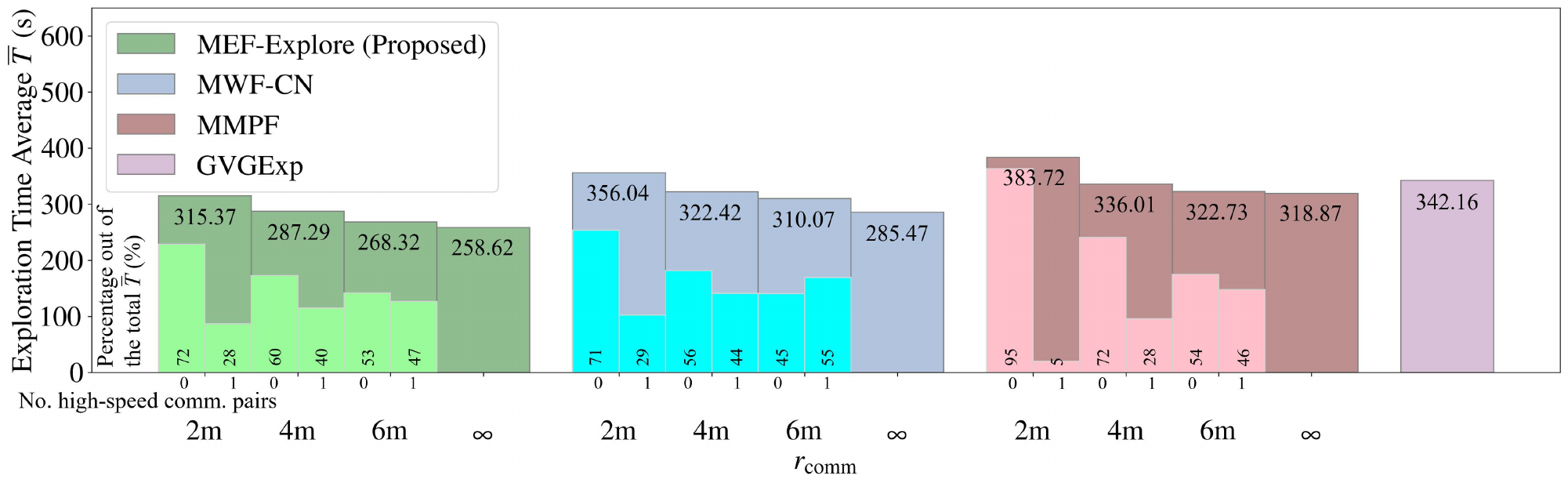}
         \caption{$\bar{T}$ (Map 1: 2 robots)}
         \label{fig:map1_2robots_Ttotal}
     \end{subfigure}
     \begin{subfigure}[b]{0.9\linewidth}
         \centering
         \includegraphics[width=\linewidth]{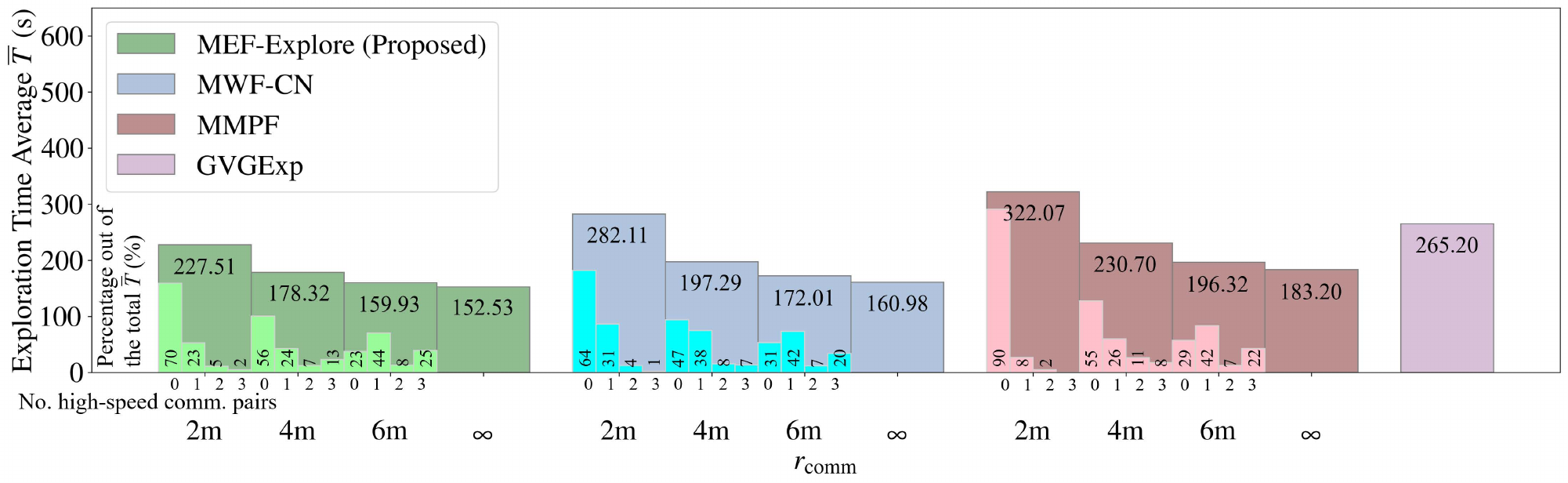}
         \caption{$\bar{T}$ (Map 1: 3 robots)}
         \label{fig:map1_3robots_Ttotal}
     \end{subfigure}
     \begin{subfigure}[b]{0.9\linewidth}
         \centering
         \includegraphics[width=\linewidth]{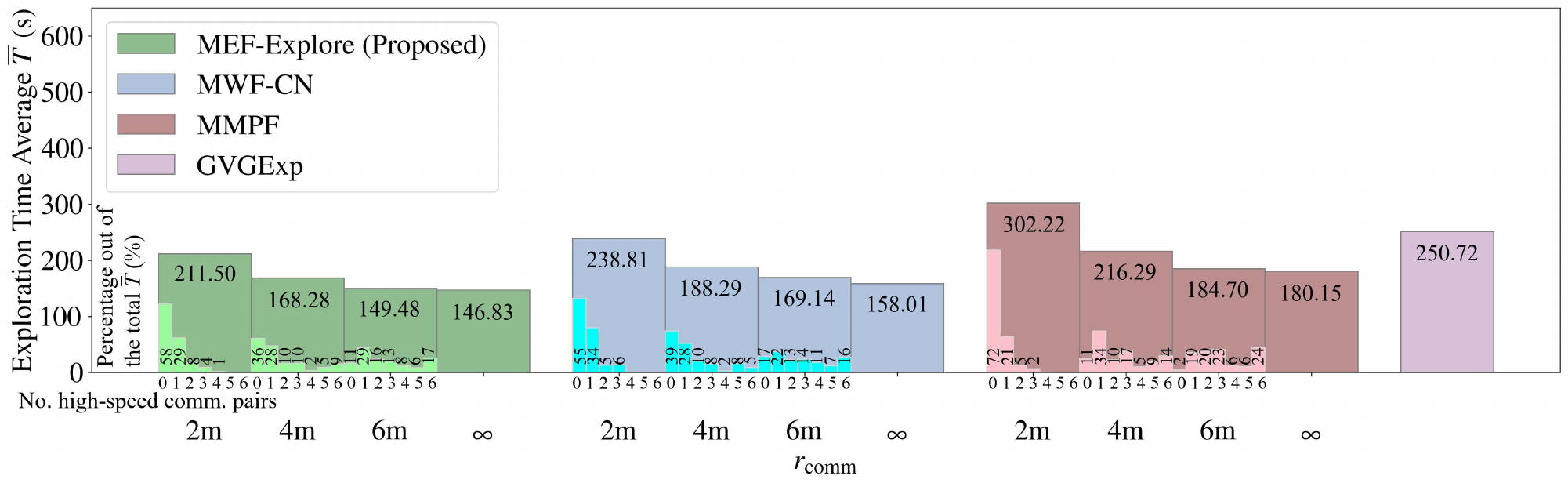}
         \caption{$\bar{T}$ (Map 1: 4 robots)}
         \label{fig:map1_4robots_Ttotal}
     \end{subfigure}
     \begin{subfigure}[b]{0.45\linewidth}
         \centering
         \includegraphics[width=\linewidth]{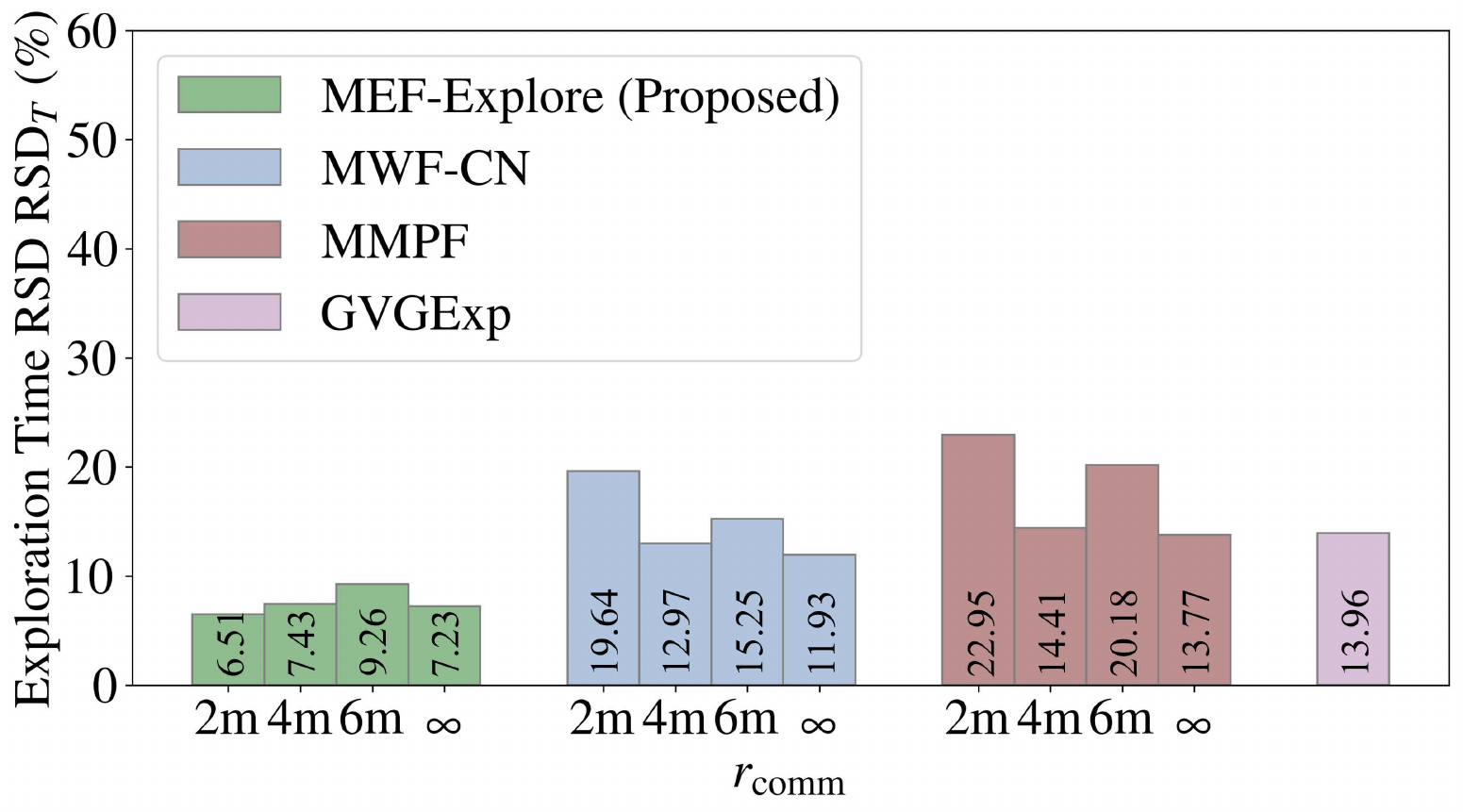}
         \caption{$\mathrm{RSD}_T$ (Map 1: 2 robots)}
         \label{fig:map1_2robots_RSD}
     \end{subfigure}
     \begin{subfigure}[b]{0.45\linewidth}
         \centering
         \includegraphics[width=\linewidth]{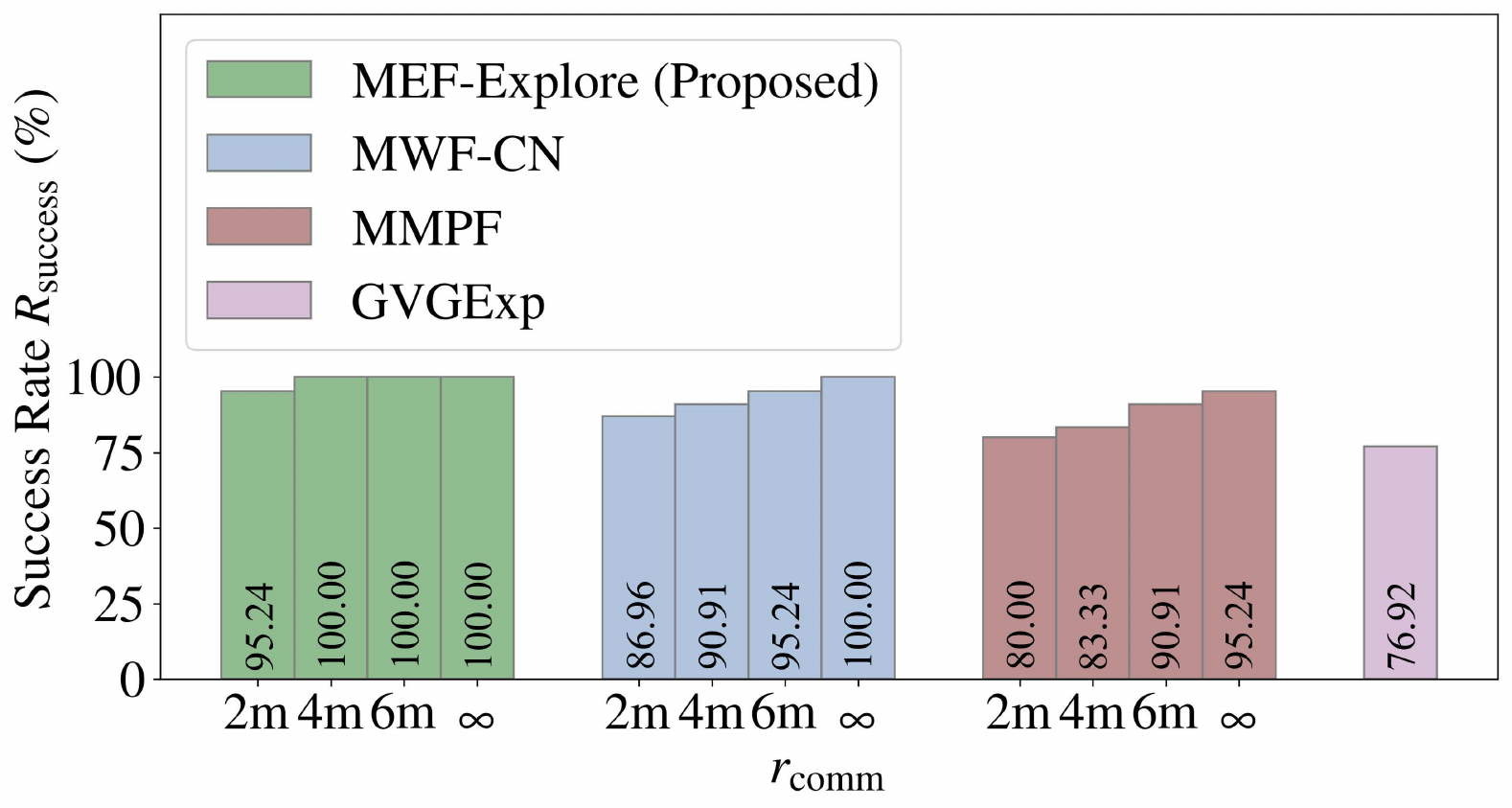}
         \caption{$R_\mathrm{Success}$ (Map 1: 2 robots)}
         \label{fig:map1_2robots_Success}
     \end{subfigure}
        \caption{Simulation results on Map 1 of our proposed MEF-Explore, the MWF-CN, the MMPF, and the GVGExp}
        \label{fig:map1_sim_results}
\end{figure*}

\begin{figure*}
    \centering
    \ContinuedFloat
    \begin{subfigure}[b]{0.45\linewidth}
         \centering
         \includegraphics[width=\linewidth]{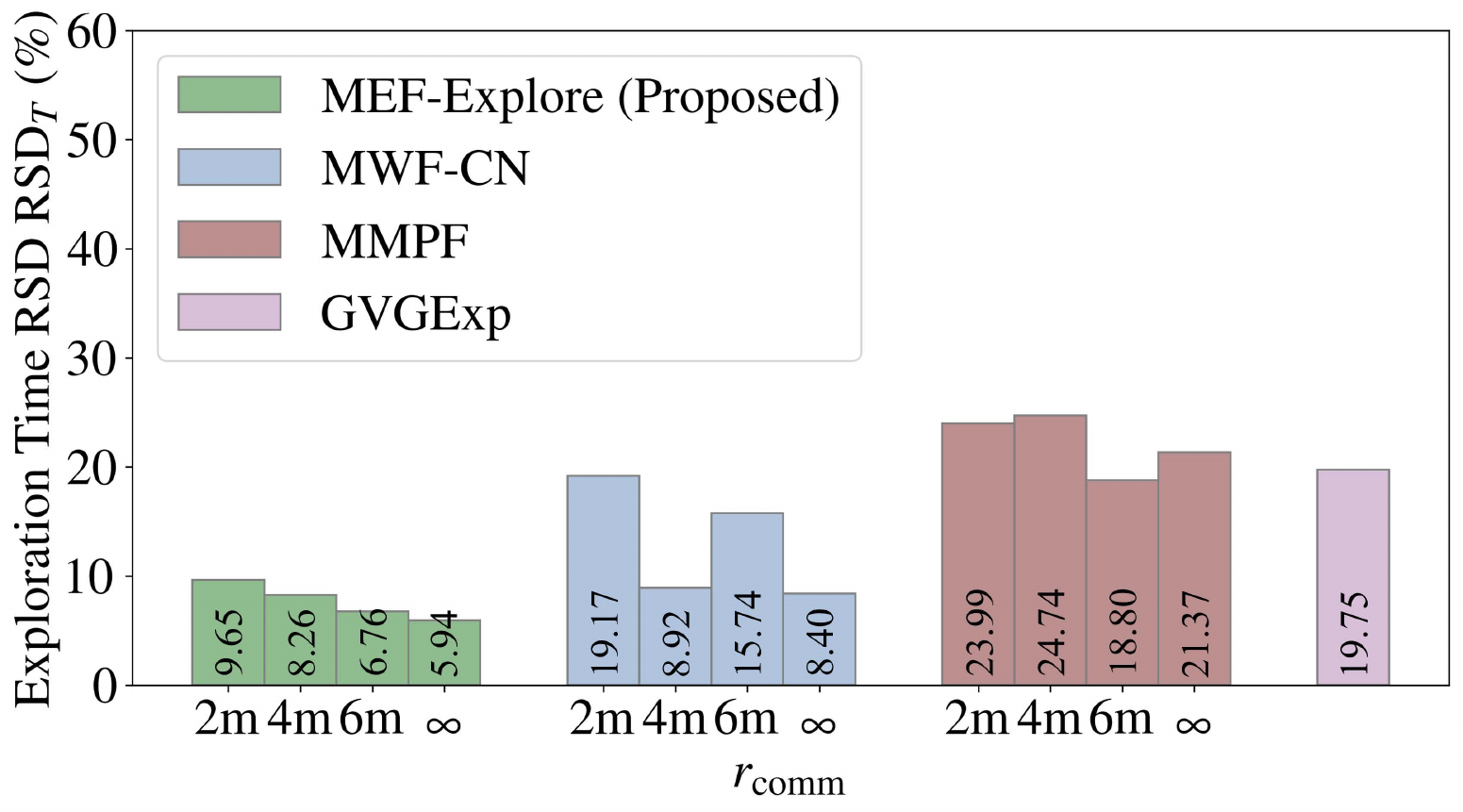}
         \caption{$\mathrm{RSD}_T$ (Map 1: 3 robots)}
         \label{fig:map1_3robots_RSD}
     \end{subfigure}
     \begin{subfigure}[b]{0.45\linewidth}
         \centering
         \includegraphics[width=\linewidth]{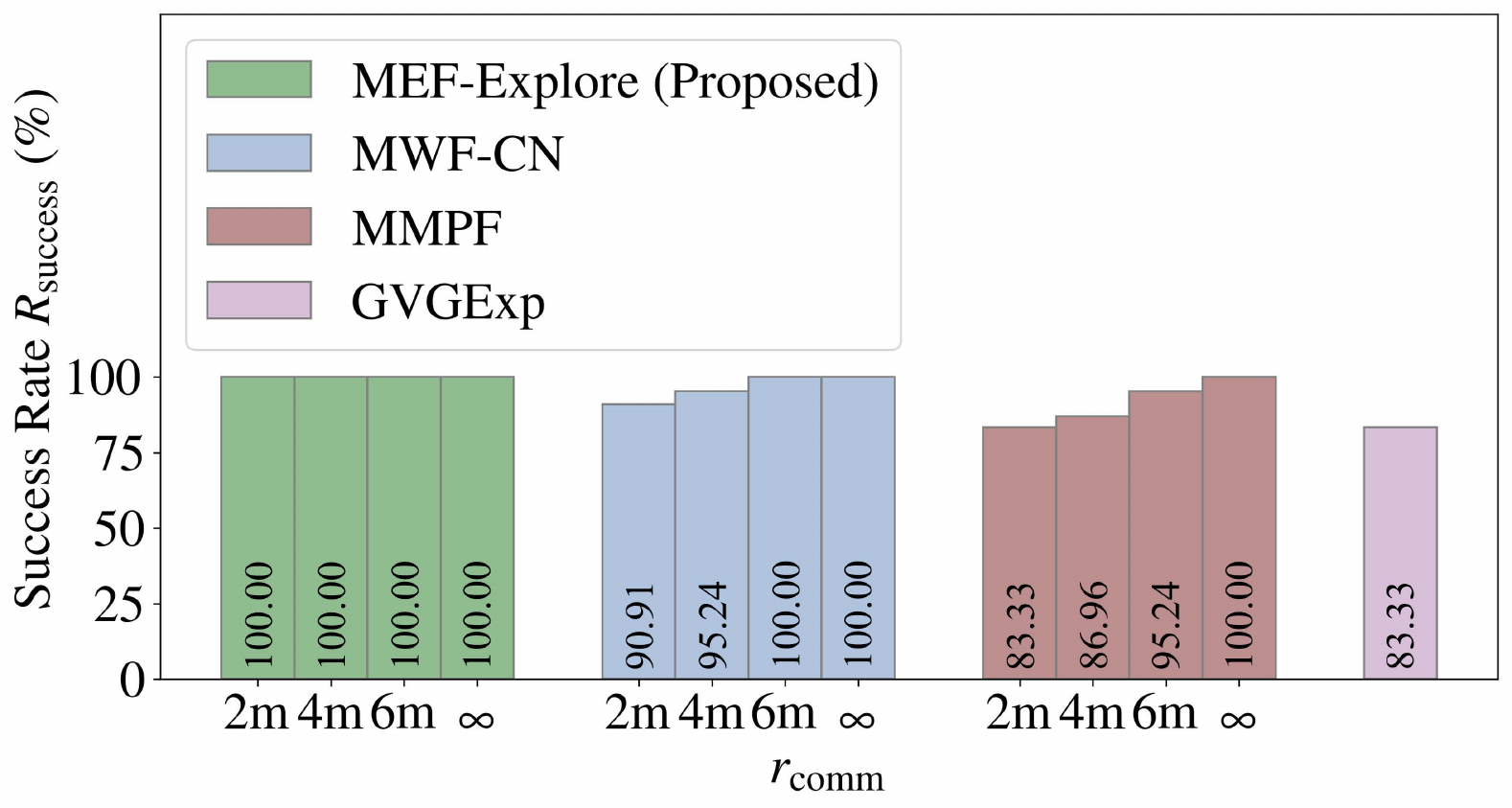}
         \caption{$R_\mathrm{Success}$ (Map 1: 3 robots)}
         \label{fig:map1_3robots_Success}
     \end{subfigure}
     \begin{subfigure}[b]{0.45\linewidth}
         \centering
         \includegraphics[width=\linewidth]{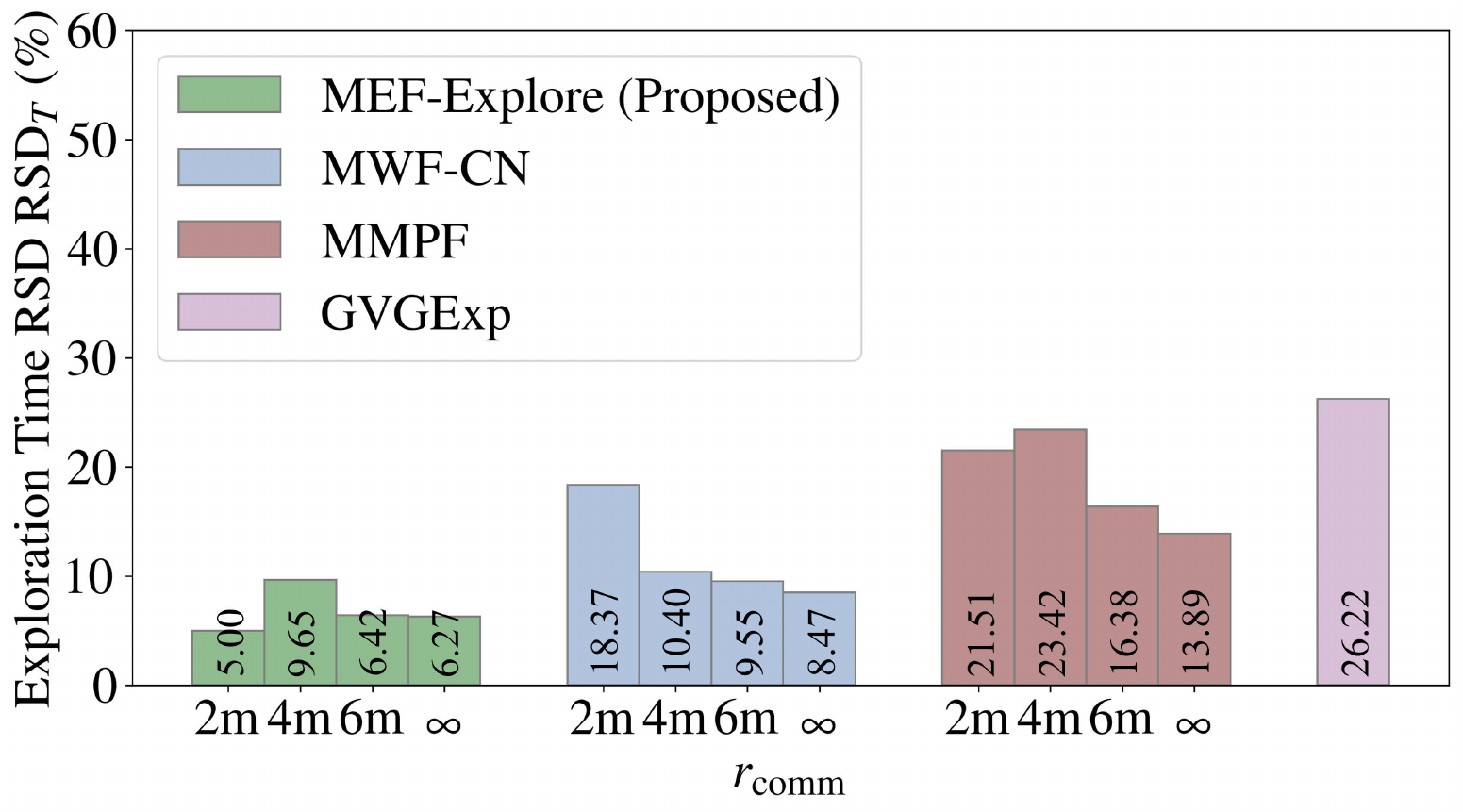}
         \caption{$\mathrm{RSD}_T$ (Map 1: 4 robots)}
         \label{fig:map1_4robots_RSD}
     \end{subfigure}
     \begin{subfigure}[b]{0.45\linewidth}
         \centering
         \includegraphics[width=\linewidth]{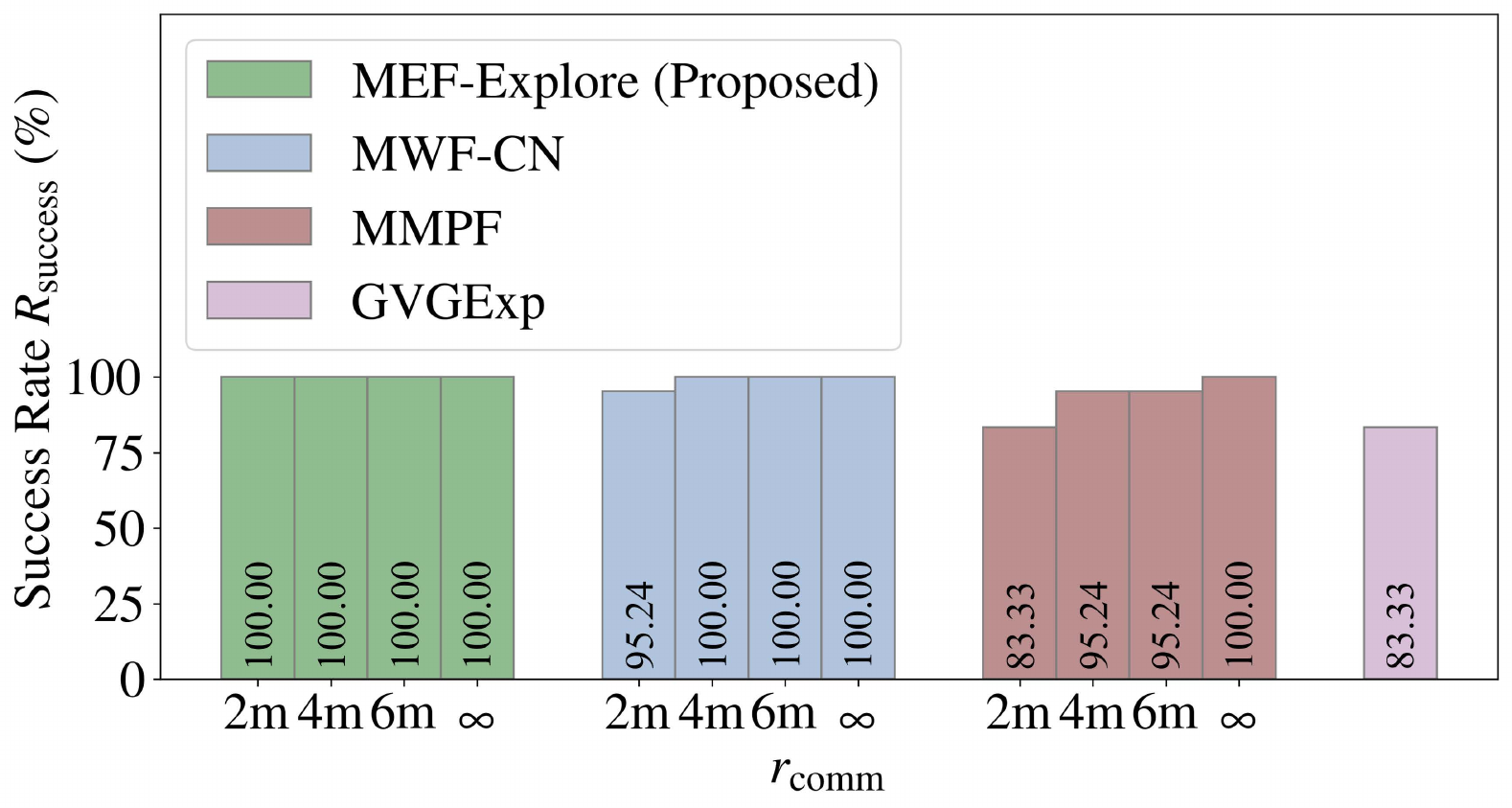}
         \caption{$R_\mathrm{Success}$ (Map 1: 4 robots)}
         \label{fig:map1_4robots_Success}
     \end{subfigure}
     \caption{Simulation results on Map 1 of our proposed MEF-Explore, the MWF-CN, the MMPF, and the GVGExp}
\end{figure*}

\begin{figure*}
     \centering
     \begin{subfigure}[b]{0.9\linewidth}
         \centering
         \includegraphics[width=\linewidth]{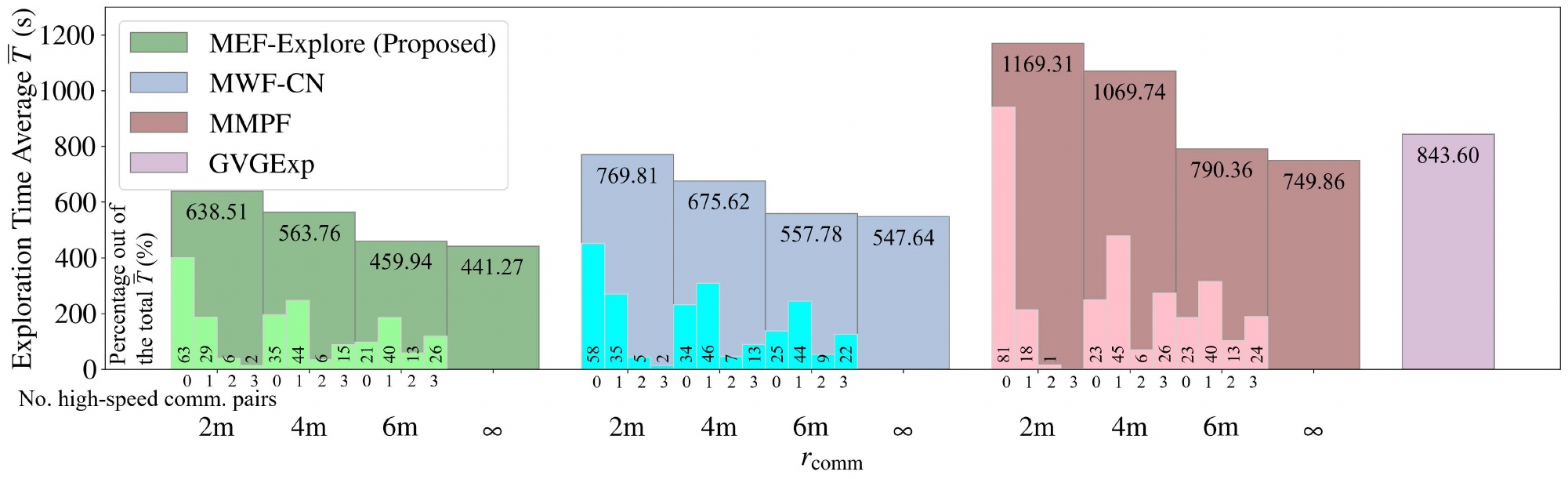}
         \caption{$\bar{T}$ (Map 2: 3 robots)}
         \label{fig:map2_3robots_Ttotal}
     \end{subfigure}
     \begin{subfigure}[b]{0.9\linewidth}
         \centering
         \includegraphics[width=\linewidth]{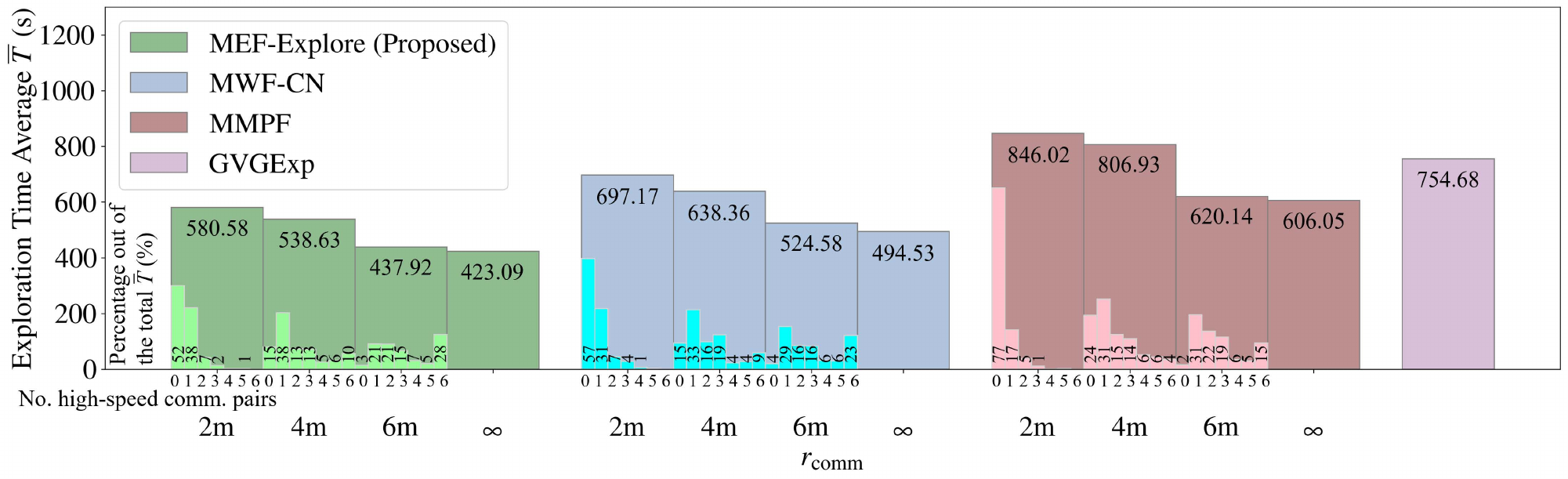}
         \caption{$\bar{T}$ (Map 2: 4 robots)}
         \label{fig:map2_4robots_Ttotal}
     \end{subfigure}
        \caption{Simulation results on Map 2 of our proposed MEF-Explore, the MWF-CN, the MMPF, and the GVGExp}
        \label{fig:map2_sim_results}
\end{figure*}

\begin{figure*}
    \centering
    \ContinuedFloat
    \begin{subfigure}[b]{0.9\linewidth}
         \centering
         \includegraphics[width=\linewidth]{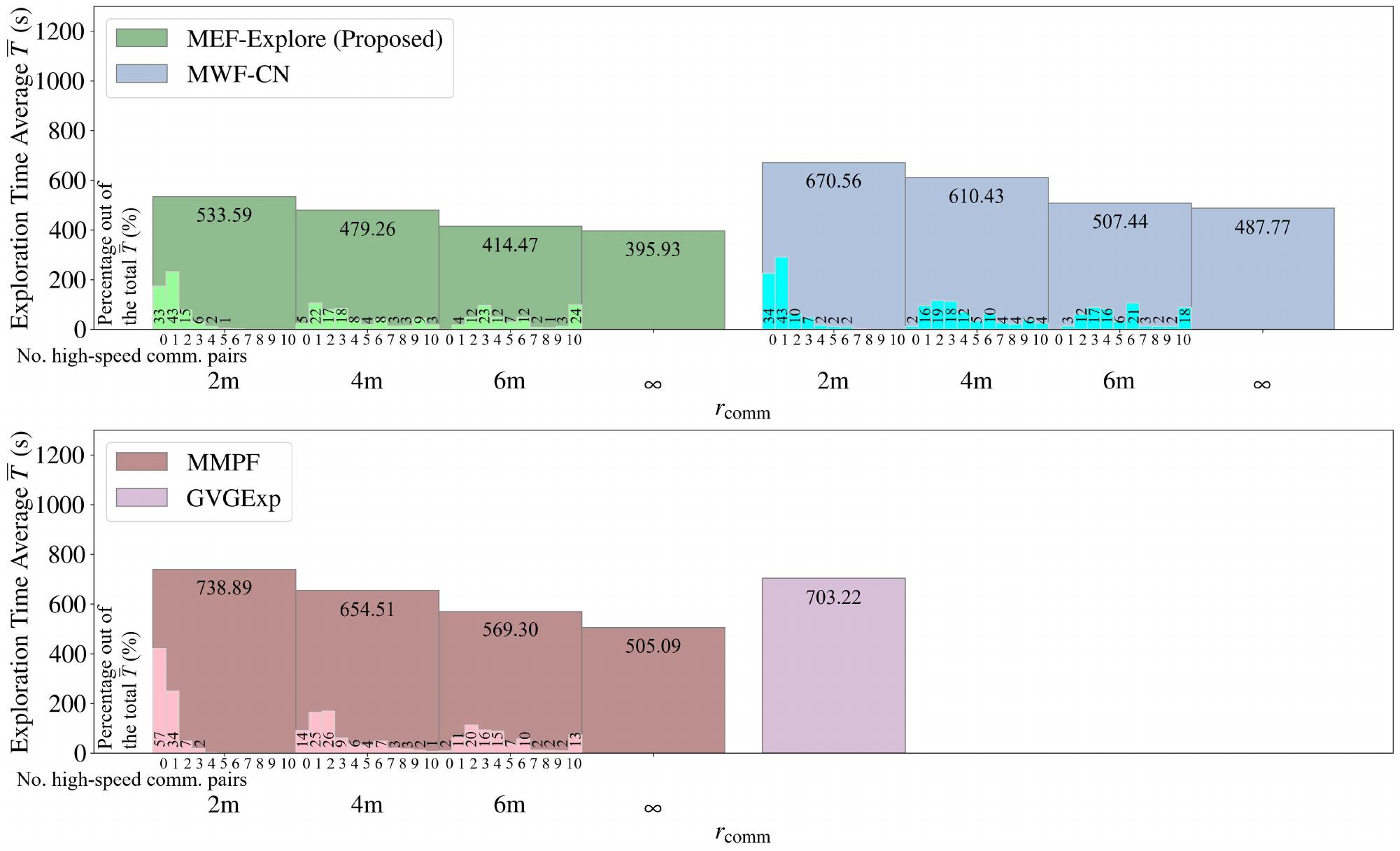}
         \caption{$\bar{T}$ (Map 2: 5 robots)}
         \label{fig:map2_5robots_Ttotal}
     \end{subfigure}
    \begin{subfigure}[b]{0.45\linewidth}
         \centering
         \includegraphics[width=\linewidth]{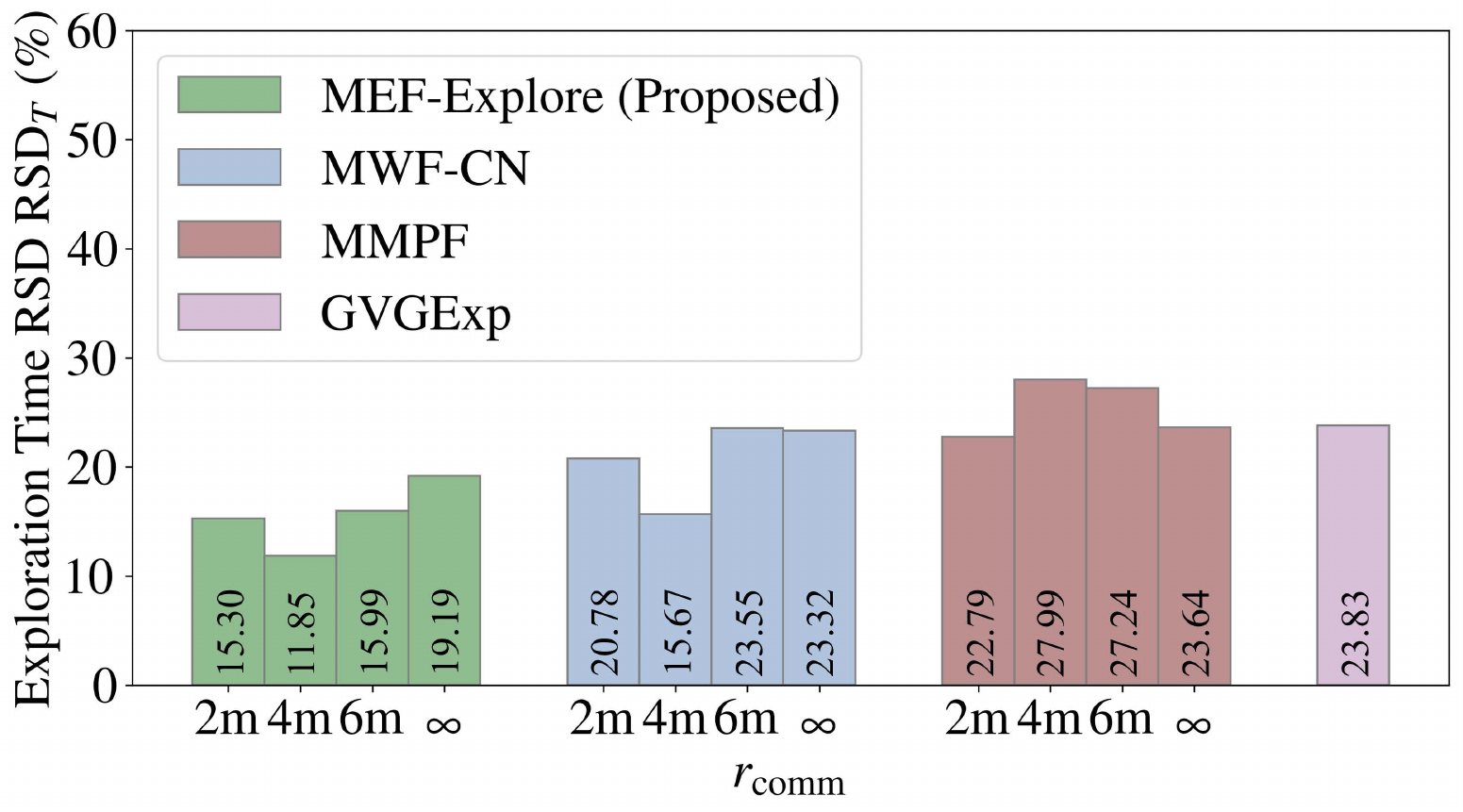}
         \caption{$\mathrm{RSD}_T$ (Map 2: 3 robots)}
         \label{fig:map2_3robots_RSD}
     \end{subfigure}
     \begin{subfigure}[b]{0.45\linewidth}
         \centering
         \includegraphics[width=\linewidth]{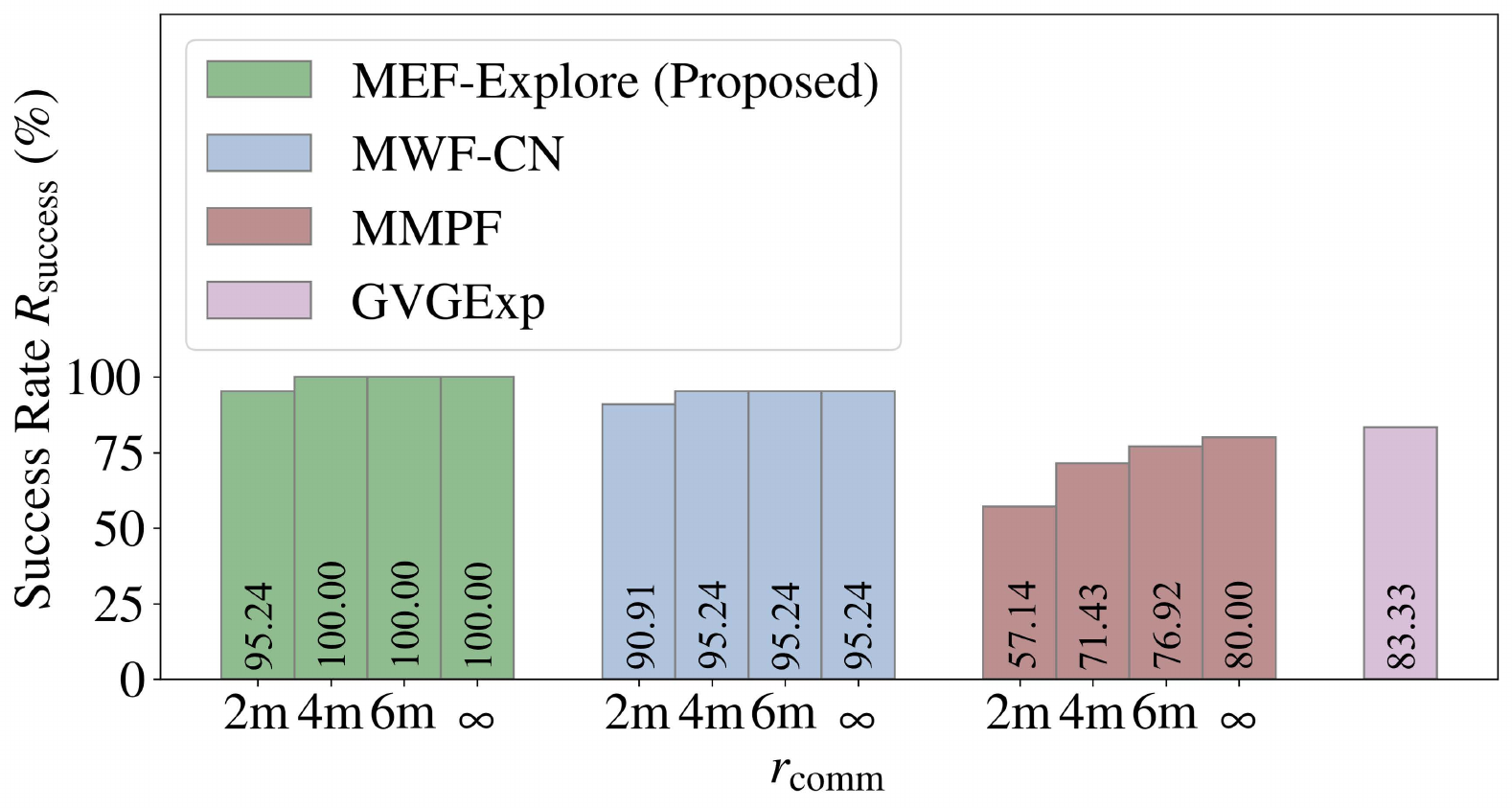}
         \caption{$R_\mathrm{Success}$ (Map 2: 3 robots)}
         \label{fig:map2_3robots_Success}
     \end{subfigure}
     \begin{subfigure}[b]{0.45\linewidth}
         \centering
         \includegraphics[width=\linewidth]{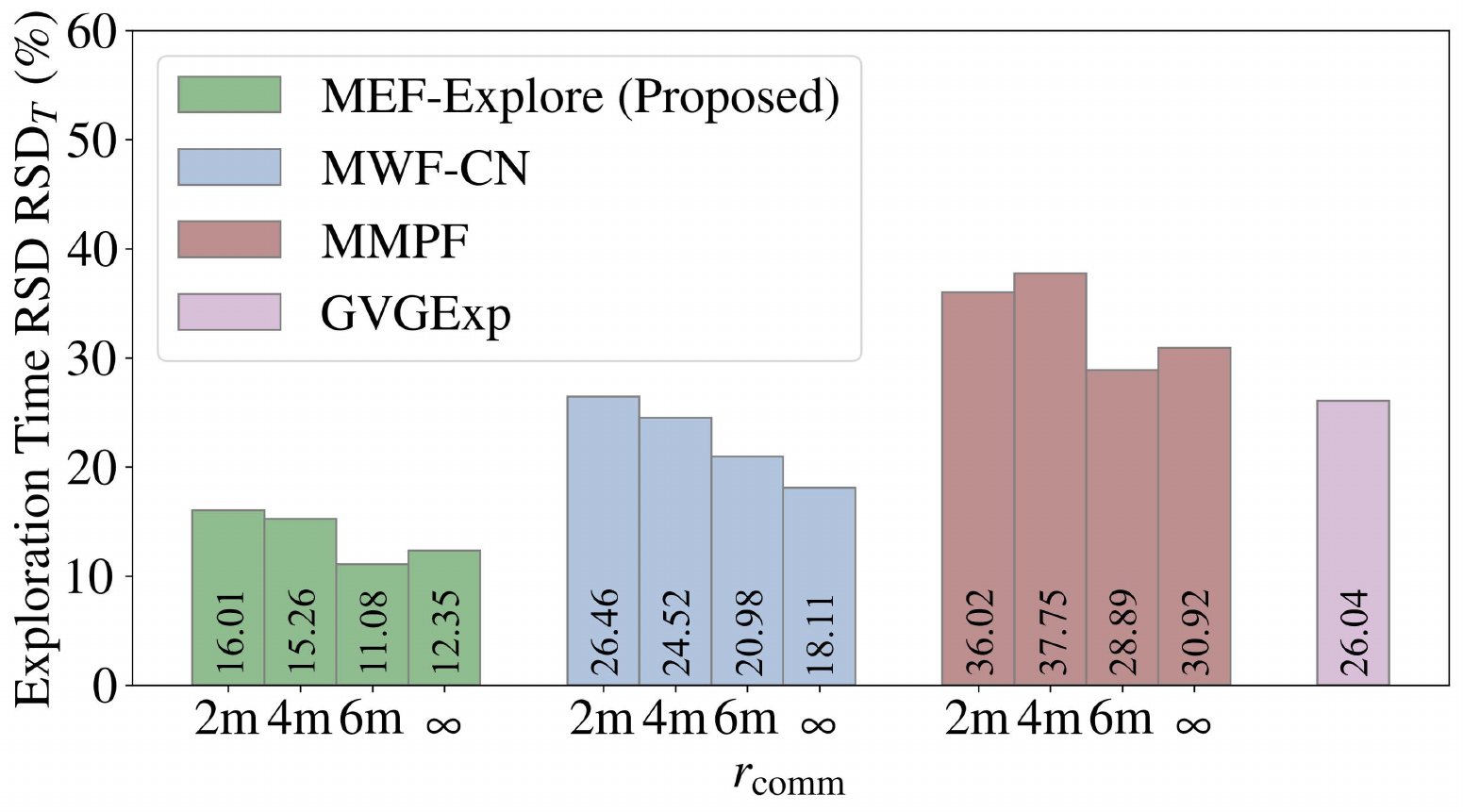}
         \caption{$\mathrm{RSD}_T$ (Map 2: 4 robots)}
         \label{fig:map2_4robots_RSD}
     \end{subfigure}
     \begin{subfigure}[b]{0.45\linewidth}
         \centering
         \includegraphics[width=\linewidth]{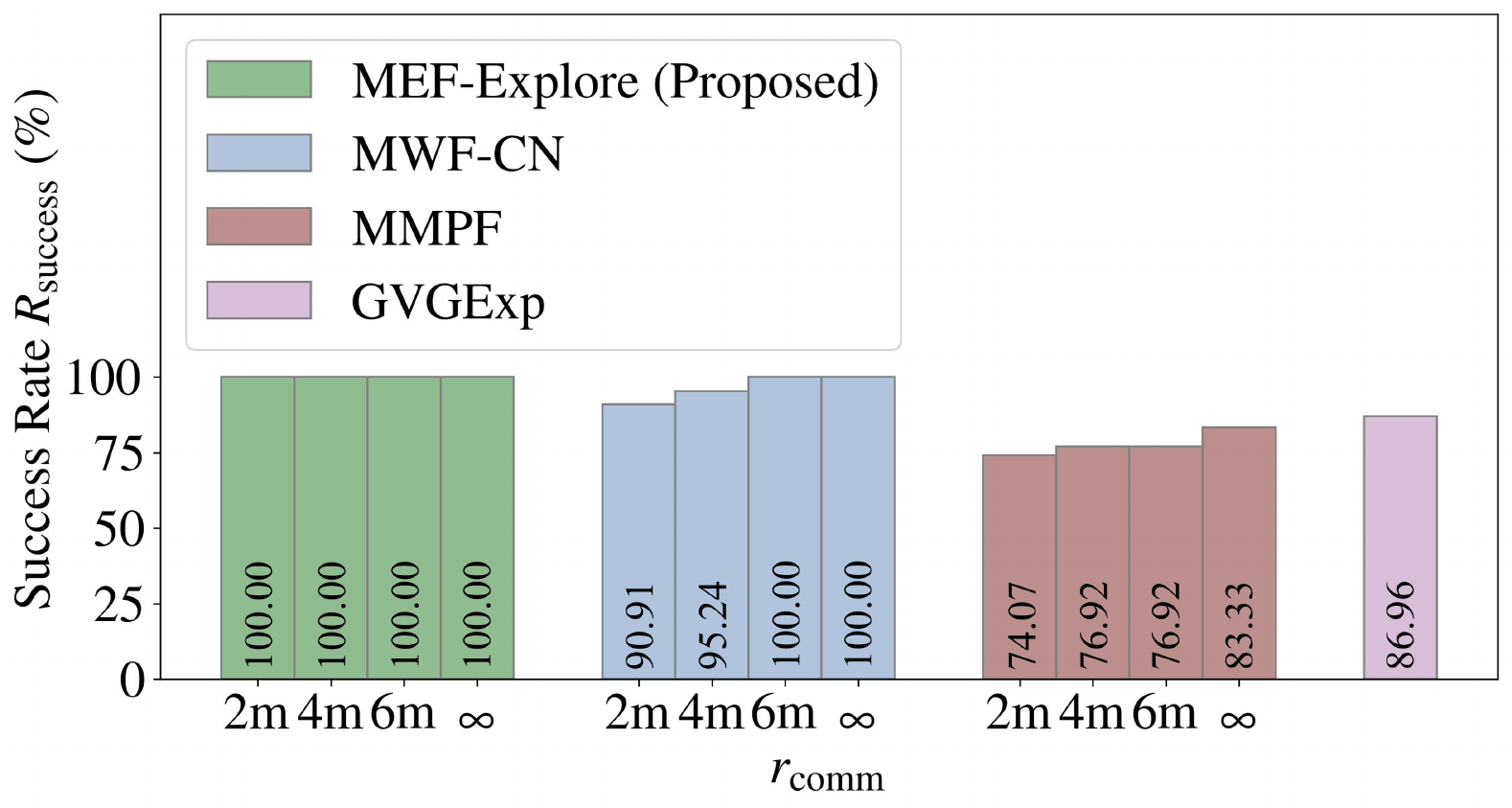}
         \caption{$R_\mathrm{Success}$ (Map 2: 4 robots)}
         \label{fig:map2_4robots_Success}
     \end{subfigure}
     \caption{Simulation results on Map 2 of our proposed MEF-Explore, the MWF-CN, the MMPF, and the GVGExp}
\end{figure*}

\begin{figure*}
    \centering
    \ContinuedFloat
    \begin{subfigure}[b]{0.45\linewidth}
         \centering
         \includegraphics[width=\linewidth]{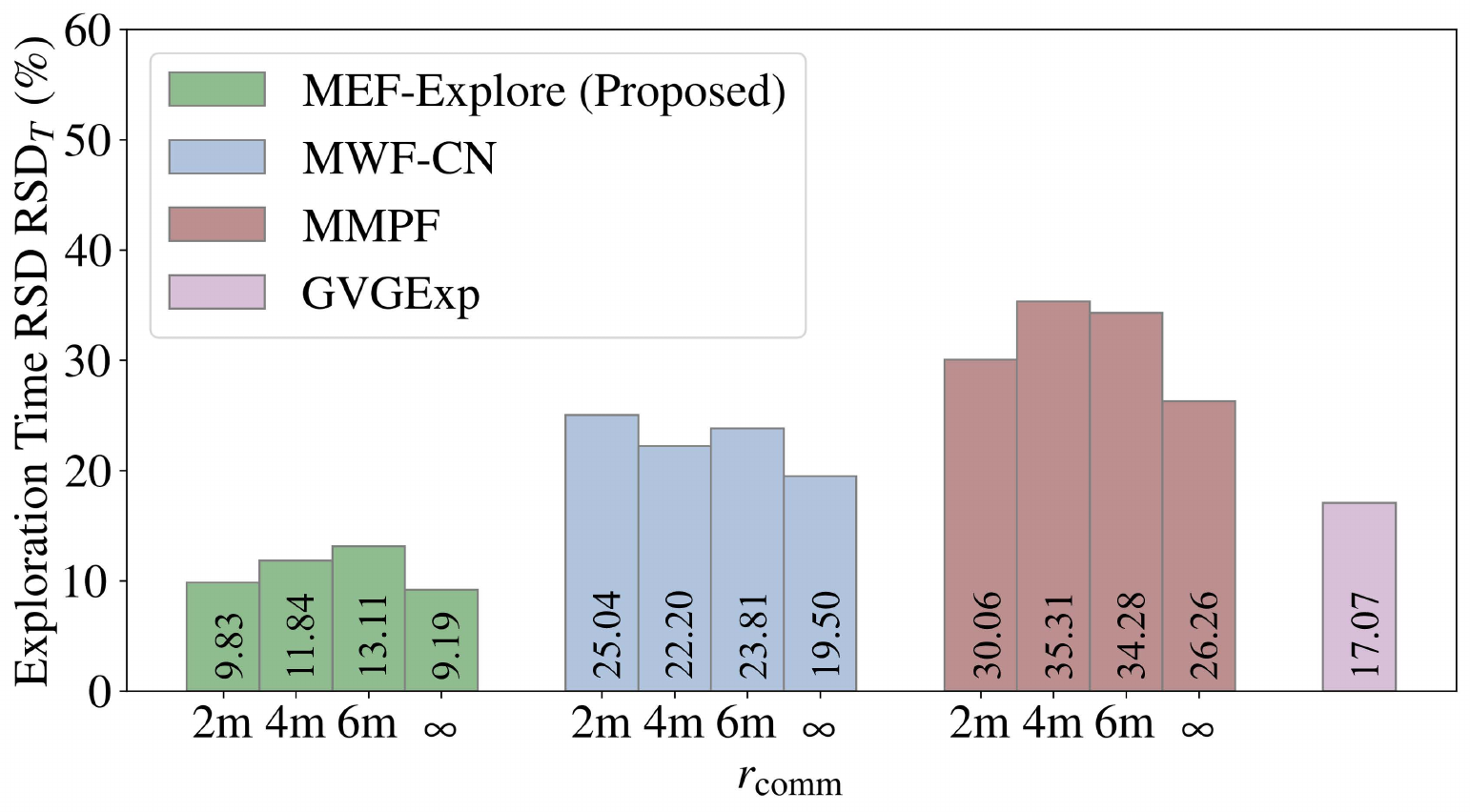}
         \caption{$\mathrm{RSD}_T$ (Map 2: 5 robots)}
         \label{fig:map2_5robots_RSD}
     \end{subfigure}
     \begin{subfigure}[b]{0.45\linewidth}
         \centering
         \includegraphics[width=\linewidth]{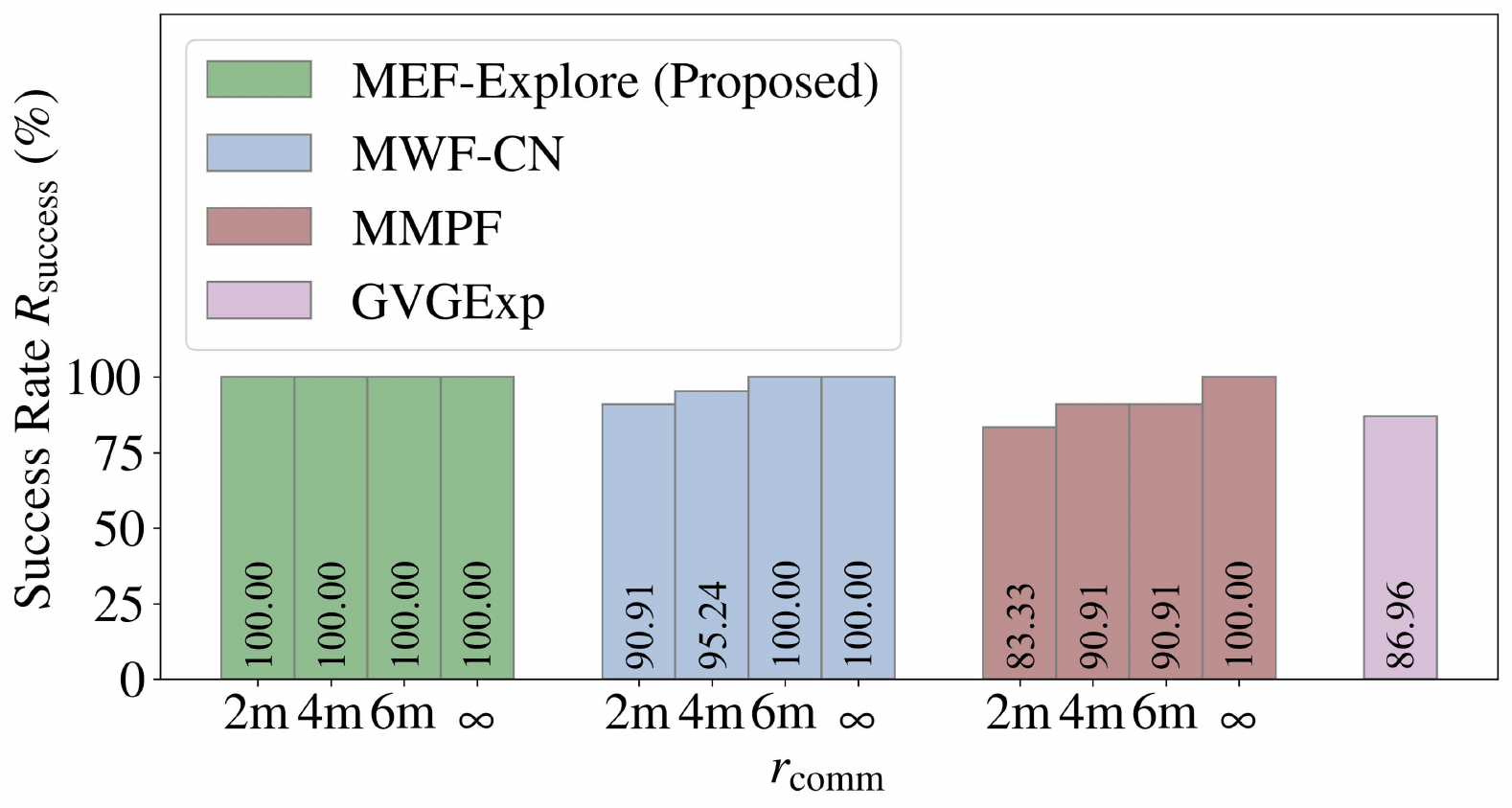}
         \caption{$R_\mathrm{Success}$ (Map 2: 5 robots)}
         \label{fig:map2_5robots_Success}
     \end{subfigure}
     \caption{Simulation results on Map 2 of our proposed MEF-Explore, the MWF-CN, the MMPF, and the GVGExp}
\end{figure*}

This section describes the simulation of the MEF-Explore. The first subsection details the evaluation metrics used to compare each exploration strategy's performance. Then, in the second subsection, we present and discuss the simulation results of different methods. The third subsection will be dedicated to the scalability analysis in case large numbers of robots are employed. Finally, in the fourth subsection, we analyze the stability of each method when inter-robot communication is unavailable.

The simulations are conducted using ROS Melodic with Ubuntu 18.04 on a Desktop PC with Xeon(R) CPU E5-1680 v3 @ 3.20GHz × 16 and 32GB RAM. We conduct multi-robot exploration utilizing TurtleBots (138mm × 178mm × 192mm) with 2D LiDAR performed by the Gazebo simulator \cite{Gazebo} in two simulated environments: Map 1 (218.6$\text{m}^2$) and Map 2 (492.9$\text{m}^2$), as shown in Fig. \ref{fig:maps}. We utilize GMapping \cite{GMapping} to create occupancy grid maps with a map resolution of 0.05m per grid cell and the modified multirobot\_map\_merge \cite{map_merge_2d} (the original version \cite{Orig-Merger}) as a map merger, which creates merged maps by stitching local maps based on rigid transformations between them that are estimated by feature detection. Hence, the process described in Section \ref{subsection:map-merging} will align with this specified map merger. Also, the TEB planner is employed for navigation.

We select the potential-field-based strategies: MWF-CN \cite{MWF-CN} and MMPF \cite{MMPF, Bench}, and the topology-based strategy: GVGExp \cite{GVGExp} as the benchmarks for our proposed MEF-Explore. Note that since the MWF-CN and the MMPF were designed based on perfect communication. To ensure a fair comparison, we also employ our information-sharing strategy to facilitate their ability to work under communication constraints. For the GVGExp, we utilize its originally presented communication and exploration strategies.

For parameters of our MEF-Explore, we have $k_f(N_r) = 2^{N_r - 3}$ in eq. (\ref{eq:H_f}) to scale the entropy values progressively. We select $k_\mathrm{ref} = 0.1$ in eq. (\ref{eq:t_ref}). For eq. (\ref{eq:H_r}), we choose $k_r = 1$ and follow some parameter selection in the previous work \cite{MWF-CN}, which is also applied to the MWF-CN: $\sigma_r=0.6$, $\alpha=2$, $\sigma_d=0.035$. Note that $\sigma_r$ reflects the preferable distance between robots. For the parameters of colored noise, which are the noise color $\alpha$ and the noise variance $\sigma_d$, $\alpha$ can trigger different behaviors of noise values, and  $\sigma_d$ is for adjusting the noise intensity.

\subsection{Evaluation Metrics}
\label{subsection:metrics}
Since our purpose is to tackle exploration speed and task success in communication-constrained situations, we select the metrics to assess the performance of each exploration strategy as follows:

\begin{enumerate}
    \item \textbf{Exploration time average \big($\bar{T}$\big):}
    $\bar{T}$ is the average of the exploration times spent by robots to explore 99\% of the environmental area. As mentioned in Section \ref{section:exploration}, we consider the time the first robot has fully explored the whole area to be the exploration time of each round.
    \item \textbf{Exploration time’s relative standard deviation \big($\mathrm{RSD}_T$\big):}
    This metric measures the consistency of exploration times spent by robots. The $\mathrm{RSD}_T$ is defined as follows:
    \begin{equation}
        \mathrm{RSD}_T = \frac{\mathrm{SD}_T}{\bar{T}} \times 100 \text{,}
    \end{equation}
    where $\mathrm{SD}_T$ is the exploration time’s standard deviation.
    
    \item \textbf{Success rate ($R_\text{success}$): }
    We define $R_\text{success}$ as follows:
    \begin{equation}
        R_\text{success} = \frac{N_\text{success}}{N_\text{total}} \times 100 \text{,}
    \end{equation}
    where $N_\text{success}$ is the number of successful exploration rounds and $N_\text{total}$ is the total number of exploration rounds. Note that we consider the exploration to be unsuccessful if all the robots are stuck and not moving for 120 seconds.
\end{enumerate}

\begin{table*}[ht]
    \scriptsize
    \centering
    \captionsetup{justification=centering}
    \caption{{Exploration Efficiency Improvement Percentages Between the Proposed Method (MEF-Explore)\\ and the Benchmarks (MWF-CN, MMPF, and GVGExp) Under Varying Numbers of Robots and $r_\mathrm{comm}$}}
    \begin{tabular}{cccccccccccccc}
    \cline{3-14}
                                                                                                                  & \multicolumn{1}{c|}{}                                                     & \multicolumn{12}{c|}{{\textbf{Map 1: Improvement of the proposed MEF-Explore compared to the benchmarks}}}                                                                                                                                                                                                                                                                                                                                                                                                                                   \\ \cline{2-14} 
    \multicolumn{1}{c|}{}                                                                                         & \multicolumn{1}{c|}{\textbf{$N_r$}}                                       & \multicolumn{4}{c|}{\textbf{2 robots}}                                                                                                          & \multicolumn{4}{c|}{\textbf{3 robots}}                                                                                                          & \multicolumn{4}{c|}{{\textbf{4 robots}}}                                                                                                                                                                               \\ \cline{2-14} 
    \multicolumn{1}{c|}{}                                                                                         & \multicolumn{1}{c|}{\textbf{$r_\mathrm{comm}$}}                           & \multicolumn{1}{c|}{\textbf{2m}} & \multicolumn{1}{c|}{\textbf{4m}} & \multicolumn{1}{c|}{\textbf{6m}} & \multicolumn{1}{c|}{\textbf{$\infty$}} & \multicolumn{1}{c|}{\textbf{2m}} & \multicolumn{1}{c|}{\textbf{4m}} & \multicolumn{1}{c|}{\textbf{6m}} & \multicolumn{1}{c|}{\textbf{$\infty$}} & \multicolumn{1}{c|}{{\textbf{2m}}} & \multicolumn{1}{c|}{{\textbf{4m}}} & \multicolumn{1}{c|}{{\textbf{6m}}} & \multicolumn{1}{c|}{{\textbf{$\infty$}}} \\ \hline
    \multicolumn{1}{|c|}{{}}                                                                 & \multicolumn{1}{c|}{\textbf{\tiny $\bar{T}$}}              & \multicolumn{1}{c|}{12.90\%}     & \multicolumn{1}{c|}{12.23\%}     & \multicolumn{1}{c|}{15.56\%}     & \multicolumn{1}{c|}{10.38\%}           & \multicolumn{1}{c|}{24.00\%}     & \multicolumn{1}{c|}{10.64\%}     & \multicolumn{1}{c|}{7.55\%}      & \multicolumn{1}{c|}{5.54\%}            & \multicolumn{1}{c|}{{12.91\%}}     & \multicolumn{1}{c|}{{11.89\%}}     & \multicolumn{1}{c|}{{13.15\%}}     & \multicolumn{1}{c|}{{7.62\%}}            \\ \cline{2-14} 
    \multicolumn{1}{|c|}{{}}                                                                 & \multicolumn{1}{c|}{\textbf{$\mathrm{RSD}_T$}}                            & \multicolumn{1}{c|}{13.12\%}     & \multicolumn{1}{c|}{5.55\%}      & \multicolumn{1}{c|}{6.00\%}      & \multicolumn{1}{c|}{4.70\%}            & \multicolumn{1}{c|}{9.52\%}      & \multicolumn{1}{c|}{0.65\%}      & \multicolumn{1}{c|}{8.98\%}      & \multicolumn{1}{c|}{2.47\%}            & \multicolumn{1}{c|}{{13.37\%}}     & \multicolumn{1}{c|}{{0.75\%}}      & \multicolumn{1}{c|}{{3.13\%}}      & \multicolumn{1}{c|}{{2.20\%}}            \\ \cline{2-14} 
    \multicolumn{1}{|c|}{\multirow{-3}{*}{{\textbf{MWF-CN \cite{MWF-CN}}}}} & \multicolumn{1}{c|}{\textbf{$R_\mathrm{success}$}}                        & \multicolumn{1}{c|}{8.28\%}      & \multicolumn{1}{c|}{9.09\%}      & \multicolumn{1}{c|}{4.76\%}      & \multicolumn{1}{c|}{0.00\%}            & \multicolumn{1}{c|}{9.09\%}      & \multicolumn{1}{c|}{4.76\%}      & \multicolumn{1}{c|}{0.00\%}      & \multicolumn{1}{c|}{0.00\%}            & \multicolumn{1}{c|}{{4.76\%}}      & \multicolumn{1}{c|}{{0.00\%}}      & \multicolumn{1}{c|}{{0.00\%}}      & \multicolumn{1}{c|}{{0.00\%}}            \\ \hline
    \multicolumn{1}{|c|}{{}}                                                                 & \multicolumn{1}{c|}{\textbf{\tiny $\bar{T}$}}              & \multicolumn{1}{c|}{21.67\%}     & \multicolumn{1}{c|}{16.96\%}     & \multicolumn{1}{c|}{20.28\%}     & \multicolumn{1}{c|}{23.30\%}           & \multicolumn{1}{c|}{41.56\%}     & \multicolumn{1}{c|}{29.37\%}     & \multicolumn{1}{c|}{22.75\%}     & \multicolumn{1}{c|}{20.11\%}           & \multicolumn{1}{c|}{{42.89\%}}     & \multicolumn{1}{c|}{{28.53\%}}     & \multicolumn{1}{c|}{{23.56\%}}     & \multicolumn{1}{c|}{{22.69\%}}           \\ \cline{2-14} 
    \multicolumn{1}{|c|}{{}}                                                                 & \multicolumn{1}{c|}{\textbf{$\mathrm{RSD}_T$}}                            & \multicolumn{1}{c|}{16.44\%}     & \multicolumn{1}{c|}{6.98\%}      & \multicolumn{1}{c|}{10.92\%}     & \multicolumn{1}{c|}{6.54\%}            & \multicolumn{1}{c|}{14.34\%}     & \multicolumn{1}{c|}{16.48\%}     & \multicolumn{1}{c|}{12.04\%}     & \multicolumn{1}{c|}{15.43\%}           & \multicolumn{1}{c|}{{16.51\%}}     & \multicolumn{1}{c|}{{13.77\%}}     & \multicolumn{1}{c|}{{9.97\%}}      & \multicolumn{1}{c|}{{7.62\%}}            \\ \cline{2-14} 
    \multicolumn{1}{|c|}{\multirow{-3}{*}{{\textbf{MMPF \cite{MMPF}}}}}     & \multicolumn{1}{c|}{\textbf{$R_\mathrm{success}$}}                        & \multicolumn{1}{c|}{15.24\%}     & \multicolumn{1}{c|}{16.67\%}     & \multicolumn{1}{c|}{9.09\%}      & \multicolumn{1}{c|}{4.76\%}            & \multicolumn{1}{c|}{16.67\%}     & \multicolumn{1}{c|}{13.04\%}     & \multicolumn{1}{c|}{4.76\%}      & \multicolumn{1}{c|}{0.00\%}            & \multicolumn{1}{c|}{{16.67\%}}     & \multicolumn{1}{c|}{{4.76\%}}      & \multicolumn{1}{c|}{{4.76\%}}      & \multicolumn{1}{c|}{{0.00\%}}            \\ \hline
    \multicolumn{1}{|c|}{{}}                                                                 & \multicolumn{1}{c|}{\textbf{\tiny $\bar{T}$}}              & \multicolumn{1}{c|}{8.49\%}      & \multicolumn{1}{c|}{19.10\%}     & \multicolumn{1}{c|}{27.52\%}     & \multicolumn{1}{c|}{32.30\%}           & \multicolumn{1}{c|}{16.57\%}     & \multicolumn{1}{c|}{48.72\%}     & \multicolumn{1}{c|}{65.82\%}     & \multicolumn{1}{c|}{73.87\%}           & \multicolumn{1}{c|}{{18.55\%}}     & \multicolumn{1}{c|}{{49.00\%}}     & \multicolumn{1}{c|}{{67.73\%}}     & \multicolumn{1}{c|}{{70.76\%}}           \\ \cline{2-14} 
    \multicolumn{1}{|c|}{{}}                                                                 & \multicolumn{1}{c|}{\textbf{$\mathrm{RSD}_T$}}                            & \multicolumn{1}{c|}{7.45\%}      & \multicolumn{1}{c|}{6.54\%}      & \multicolumn{1}{c|}{4.71\%}      & \multicolumn{1}{c|}{6.73\%}            & \multicolumn{1}{c|}{10.10\%}     & \multicolumn{1}{c|}{11.49\%}     & \multicolumn{1}{c|}{12.99\%}     & \multicolumn{1}{c|}{13.81\%}           & \multicolumn{1}{c|}{{21.22\%}}     & \multicolumn{1}{c|}{{16.57\%}}     & \multicolumn{1}{c|}{{19.80\%}}     & \multicolumn{1}{c|}{{19.95\%}}           \\ \cline{2-14} 
    \multicolumn{1}{|c|}{\multirow{-3}{*}{{\textbf{GVGExp \cite{GVGExp}}}}} & \multicolumn{1}{c|}{\textbf{$R_\mathrm{success}$}}                        & \multicolumn{1}{c|}{18.32\%}     & \multicolumn{1}{c|}{23.08\%}     & \multicolumn{1}{c|}{23.08\%}     & \multicolumn{1}{c|}{23.08\%}           & \multicolumn{1}{c|}{16.67\%}     & \multicolumn{1}{c|}{16.67\%}     & \multicolumn{1}{c|}{16.67\%}     & \multicolumn{1}{c|}{16.67\%}           & \multicolumn{1}{c|}{{16.67\%}}     & \multicolumn{1}{c|}{{16.67\%}}     & \multicolumn{1}{c|}{{16.67\%}}     & \multicolumn{1}{c|}{{16.67\%}}           \\ \hline
    \multicolumn{1}{l}{}                                                                                          & \multicolumn{1}{l}{}                                                      & \multicolumn{1}{l}{}             & \multicolumn{1}{l}{}             & \multicolumn{1}{l}{}             & \multicolumn{1}{l}{}                   & \multicolumn{1}{l}{}             & \multicolumn{1}{l}{}             & \multicolumn{1}{l}{}             & \multicolumn{1}{l}{}                   & \multicolumn{1}{l}{}                                    & \multicolumn{1}{l}{}                                    & \multicolumn{1}{l}{}                                    & \multicolumn{1}{l}{}                                          \\ \cline{3-14} 
                                                                                                                  & \multicolumn{1}{c|}{}                                                     & \multicolumn{12}{c|}{{\textbf{Map 2: Improvement of the proposed MEF-Explore compared to the benchmarks}}}                                                                                                                                                                                                                                                                                                                                                                                                                                   \\ \cline{2-14} 
    \multicolumn{1}{c|}{}                                                                                         & \multicolumn{1}{c|}{\textbf{$N_r$}}                                       & \multicolumn{4}{c|}{\textbf{3 robots}}                                                                                                          & \multicolumn{4}{c|}{\textbf{4 robots}}                                                                                                          & \multicolumn{4}{c|}{{\textbf{5 robots}}}                                                                                                                                                                               \\ \cline{2-14} 
    \multicolumn{1}{c|}{}                                                                                         & \multicolumn{1}{c|}{\textbf{$r_\mathrm{comm}$}}                           & \multicolumn{1}{c|}{\textbf{2m}} & \multicolumn{1}{c|}{\textbf{4m}} & \multicolumn{1}{c|}{\textbf{6m}} & \multicolumn{1}{c|}{\textbf{$\infty$}} & \multicolumn{1}{c|}{\textbf{2m}} & \multicolumn{1}{c|}{\textbf{4m}} & \multicolumn{1}{c|}{\textbf{6m}} & \multicolumn{1}{c|}{\textbf{$\infty$}} & \multicolumn{1}{c|}{{\textbf{2m}}} & \multicolumn{1}{c|}{{\textbf{4m}}} & \multicolumn{1}{c|}{{\textbf{6m}}} & \multicolumn{1}{c|}{{\textbf{$\infty$}}} \\ \hline
    \multicolumn{1}{|c|}{{}}                                                                 & \multicolumn{1}{c|}{\textbf{\tiny $\bar{T}$}}              & \multicolumn{1}{c|}{20.56\%}     & \multicolumn{1}{c|}{19.84\%}     & \multicolumn{1}{c|}{21.27\%}     & \multicolumn{1}{c|}{24.10\%}           & \multicolumn{1}{c|}{20.08\%}     & \multicolumn{1}{c|}{18.51\%}     & \multicolumn{1}{c|}{19.79\%}     & \multicolumn{1}{c|}{16.88\%}           & \multicolumn{1}{c|}{{25.67\%}}     & \multicolumn{1}{c|}{{27.37\%}}     & \multicolumn{1}{c|}{{22.43\%}}     & \multicolumn{1}{c|}{{23.20\%}}           \\ \cline{2-14} 
    \multicolumn{1}{|c|}{{}}                                                                 & \multicolumn{1}{c|}{\textbf{$\mathrm{RSD}_T$}}                            & \multicolumn{1}{c|}{5.48\%}      & \multicolumn{1}{c|}{3.81\%}      & \multicolumn{1}{c|}{7.56\%}      & \multicolumn{1}{c|}{4.13\%}            & \multicolumn{1}{c|}{10.46\%}     & \multicolumn{1}{c|}{9.25\%}      & \multicolumn{1}{c|}{9.90\%}      & \multicolumn{1}{c|}{5.76\%}            & \multicolumn{1}{c|}{{15.21\%}}     & \multicolumn{1}{c|}{{10.36\%}}     & \multicolumn{1}{c|}{{10.70\%}}     & \multicolumn{1}{c|}{{10.31\%}}           \\ \cline{2-14} 
    \multicolumn{1}{|c|}{\multirow{-3}{*}{{\textbf{MWF-CN \cite{MWF-CN}}}}} & \multicolumn{1}{c|}{\textbf{$R_\mathrm{success}$}}                        & \multicolumn{1}{c|}{4.33\%}      & \multicolumn{1}{c|}{4.76\%}      & \multicolumn{1}{c|}{4.76\%}      & \multicolumn{1}{c|}{4.76\%}            & \multicolumn{1}{c|}{9.09\%}      & \multicolumn{1}{c|}{4.76\%}      & \multicolumn{1}{c|}{0.00\%}      & \multicolumn{1}{c|}{0.00\%}            & \multicolumn{1}{c|}{{9.09\%}}      & \multicolumn{1}{c|}{{4.76\%}}      & \multicolumn{1}{c|}{{0.00\%}}      & \multicolumn{1}{c|}{{0.00\%}}            \\ \hline
    \multicolumn{1}{|c|}{{}}                                                                 & \multicolumn{1}{c|}{\textbf{\tiny $\bar{T}$}}              & \multicolumn{1}{c|}{83.13\%}     & \multicolumn{1}{c|}{89.75\%}     & \multicolumn{1}{c|}{71.84\%}     & \multicolumn{1}{c|}{69.93\%}           & \multicolumn{1}{c|}{45.72\%}     & \multicolumn{1}{c|}{49.81\%}     & \multicolumn{1}{c|}{41.61\%}     & \multicolumn{1}{c|}{43.24\%}           & \multicolumn{1}{c|}{{38.48\%}}     & \multicolumn{1}{c|}{{36.57\%}}     & \multicolumn{1}{c|}{{37.35\%}}     & \multicolumn{1}{c|}{{27.57\%}}           \\ \cline{2-14} 
    \multicolumn{1}{|c|}{{}}                                                                 & \multicolumn{1}{c|}{\textbf{$\mathrm{RSD}_T$}}                            & \multicolumn{1}{c|}{7.49\%}      & \multicolumn{1}{c|}{16.14\%}     & \multicolumn{1}{c|}{11.24\%}     & \multicolumn{1}{c|}{4.45\%}            & \multicolumn{1}{c|}{20.01\%}     & \multicolumn{1}{c|}{22.48\%}     & \multicolumn{1}{c|}{17.81\%}     & \multicolumn{1}{c|}{18.57\%}           & \multicolumn{1}{c|}{{20.23\%}}     & \multicolumn{1}{c|}{{23.48\%}}     & \multicolumn{1}{c|}{{21.17\%}}     & \multicolumn{1}{c|}{{17.07\%}}           \\ \cline{2-14} 
    \multicolumn{1}{|c|}{\multirow{-3}{*}{{\textbf{MMPF \cite{MMPF}}}}}     & \multicolumn{1}{c|}{\textbf{$R_\mathrm{success}$}}                        & \multicolumn{1}{c|}{38.10\%}     & \multicolumn{1}{c|}{28.57\%}     & \multicolumn{1}{c|}{23.08\%}     & \multicolumn{1}{c|}{20.00\%}           & \multicolumn{1}{c|}{25.93\%}     & \multicolumn{1}{c|}{23.08\%}     & \multicolumn{1}{c|}{23.08\%}     & \multicolumn{1}{c|}{16.67\%}           & \multicolumn{1}{c|}{{16.67\%}}     & \multicolumn{1}{c|}{{9.09\%}}      & \multicolumn{1}{c|}{{9.09\%}}      & \multicolumn{1}{c|}{{0.00\%}}            \\ \hline
    \multicolumn{1}{|c|}{{}}                                                                 & \multicolumn{1}{c|}{\textbf{\tiny $\bar{T}$}}              & \multicolumn{1}{c|}{32.12\%}     & \multicolumn{1}{c|}{49.64\%}     & \multicolumn{1}{c|}{83.41\%}     & \multicolumn{1}{c|}{91.17\%}           & \multicolumn{1}{c|}{29.99\%}     & \multicolumn{1}{c|}{40.11\%}     & \multicolumn{1}{c|}{72.33\%}     & \multicolumn{1}{c|}{78.37\%}           & \multicolumn{1}{c|}{{31.79\%}}     & \multicolumn{1}{c|}{{46.73\%}}     & \multicolumn{1}{c|}{{69.67\%}}     & \multicolumn{1}{c|}{{77.61\%}}           \\ \cline{2-14} 
    \multicolumn{1}{|c|}{{}}                                                                 & \multicolumn{1}{c|}{\textbf{$\mathrm{RSD}_T$}}                            & \multicolumn{1}{c|}{8.53\%}      & \multicolumn{1}{c|}{11.97\%}     & \multicolumn{1}{c|}{7.83\%}      & \multicolumn{1}{c|}{4.64\%}            & \multicolumn{1}{c|}{10.04\%}     & \multicolumn{1}{c|}{10.78\%}     & \multicolumn{1}{c|}{14.97\%}     & \multicolumn{1}{c|}{13.69\%}           & \multicolumn{1}{c|}{{7.24\%}}      & \multicolumn{1}{c|}{{5.23\%}}      & \multicolumn{1}{c|}{{3.95\%}}      & \multicolumn{1}{c|}{{7.88\%}}            \\ \cline{2-14} 
    \multicolumn{1}{|c|}{\multirow{-3}{*}{{\textbf{GVGExp \cite{GVGExp}}}}} & \multicolumn{1}{c|}{\textbf{$R_\mathrm{success}$}}                        & \multicolumn{1}{c|}{11.90\%}     & \multicolumn{1}{c|}{16.67\%}     & \multicolumn{1}{c|}{16.67\%}     & \multicolumn{1}{c|}{16.67\%}           & \multicolumn{1}{c|}{13.04\%}     & \multicolumn{1}{c|}{13.04\%}     & \multicolumn{1}{c|}{13.04\%}     & \multicolumn{1}{c|}{13.04\%}           & \multicolumn{1}{c|}{{13.04\%}}     & \multicolumn{1}{c|}{{13.04\%}}     & \multicolumn{1}{c|}{{13.04\%}}     & \multicolumn{1}{c|}{{13.04\%}}           \\ \hline
    \end{tabular}
    \label{table:sim-improve}
\end{table*}

\subsection{Benchmark and Comparison}
\label{subsection:benchmark}
In this subsection, we discuss the simulation results of exploration performance by our MEF-Explore, the MWF-CN, the MMPF, and the GVGExp. The initial positions of two, three, and four robots for Map 1 and three, four, and five robots for Map 2 are shown in Fig. \ref{fig:maps}. For the first three mentioned methods, we evaluate the impact of the cut-off range $r_\mathrm{comm}$ using the following ranges: 2m, 4m, 6m, and $\infty$ (representing perfect communication). Additionally, we assess the number of robot pairs with high-speed communication throughout the exploration to depict the distribution of communication status. As shown by the x-axis of Fig. \ref{fig:map1_2robots_Ttotal}, \ref{fig:map1_3robots_Ttotal}, \ref{fig:map2_3robots_Ttotal}, and \ref{fig:map2_4robots_Ttotal}, the number of pairs is from 0 to $\frac{N_r(N_r-1)}{2}$ depending on the number of robots $N_r$. As the GVGExp has its own communication strategy based on generalized Voronoi graphs, it is not based on the concept of low-speed and high-speed communication. So, we consider only the cases of two, three, and four robots for Map 1 and three, four, and five robots for Map 2 without taking $r_\mathrm{comm}$ into account. Note that the simulations of different methods are run for 20 rounds per each $r_\mathrm{comm}$, so the results for two simulated environments shown in Fig. \ref{fig:map1_sim_results} and \ref{fig:map2_sim_results} are calculated based on 1,560 rounds of successful exploration, and the number of unsuccessful rounds is also collected for computing $R_\mathrm{success}$.

The presented results show that, overall, more robots and larger cut-off ranges $r_\mathrm{comm}$ lead to shorter time robots spent on exploration and fewer failed rounds. Apart from the apparent reason that a higher number of robots can cover larger areas, we also observe that improved outcomes are from increasing chances of map information relays. The more extensive range of $r_\mathrm{comm}$ also implies the broader region with the availability of high-speed communication, which gives rise to better exploration time and success rate. Considering the performance of each method, our proposed MEF-Explore outperforms the MWF-CN, the MMPF, and the GVGExp for all metrics and $r_\mathrm{comm}$ ranges on both maps. Specifically, it achieves the best results in all test cases involving two, three, and four robots exploring Map 1, as well as three, four, and five robots exploring Map 2. These outcomes highlight the consistently superior performance of our proposed MEF-Explore across all evaluated scenarios. We present Table \ref{table:sim-improve} to summarize the efficiency improvement for $\bar{T}$, $\mathrm{RSD}_T$, and $R_\mathrm{success}$ of the MEF-Explore compared to the benchmark methods. 

Based on the observation from Table \ref{table:sim-improve}, we notice exploration by robots using the GVGExp requires a considerably longer exploration time $\bar{T}$, has higher $\mathrm{RSD}_T$, and completes the mission with less $R_\mathrm{success}$, compared to our MEF-Explore. The main reason behind this is its data exchange, which relies significantly on the generalized Voronoi graphs reconstructed anytime the robot's maps are updated. Its strict conditions also make inter-robot map sharing occur infrequently, so robots often spend extra time exploring.

For the methods that are deployed together with our proposed information-sharing strategy, the overall computational time of our MEF-Explore, the MWF-CN, and the MMPF grows exponentially based on $N_C$ to the power of the total exploration area $A$ or $\mathscr{O}\big(N_C^A\big)$, as their core calculation is based on the potential field. However, our MEF-Explore has the best performance. The key factors behind this achievement are our well-developed entropies, which can realistically represent the uncertainty throughout the exploration mission. For our strategy, the environmental areas are indicated by the combination of uncertainty and attractiveness levels, unlike the MWF-CN and the MMPF ones, which are mainly based on attractive and repulsive potential values. So, our robot goal candidates are chosen more efficiently, resulting in faster, more consistent, and more successful exploration.

Some rounds of our designed robot rendezvous that arise without predetermination also help speed up inter-robot map merging, which is beneficial for completing exploration. Meanwhile, since the MWF-CN and the MMPF do not have this rendezvous feature, robots exchange their maps via high-speed communication only when they are occasionally near each other. In other words, the robots using these two benchmarks meet less frequently during exploration than those using our MEF-Explore, which also meet during rounds of implicit rendezvous. Hence, those using the MWF-CN and the MMPF finish exploring the environments more slowly, and sometimes, their mission is not accomplished because they do not receive some parts of the map from other robots during exploration to merge together into the complete map. 

Furthermore, since our proposed duration-adaptive goal-assigning module can effectively control each robot's goal assignment, the exploration finishes more rapidly and successfully than when using the MWF-CN and the MMPF. Since, for both methods, robots' goals are assigned nonstop, they sometimes traverse jerky or get stuck because of frequent interruptions from the exploration stack to the navigation stack. Therefore, in light of the aforementioned components altogether, it is clear how the exploration performance of our MEF-Explore is better than the benchmark methods in all scenarios.

\begin{figure}
     \centering
     \begin{subfigure}[b]{0.322\linewidth}
         \centering
         \includegraphics[width=\linewidth]{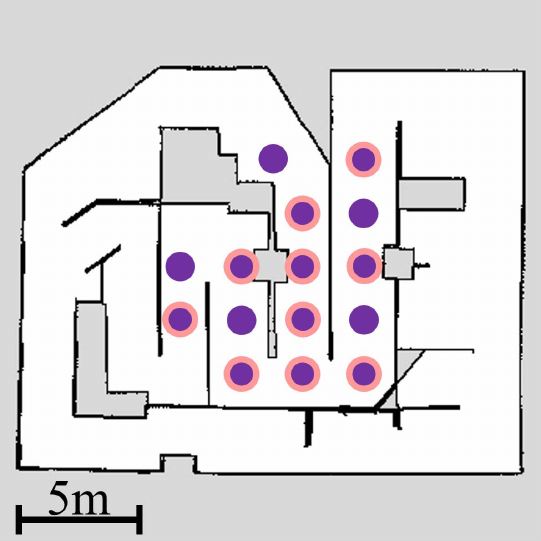}
         \caption{Map 1}
         \label{fig:map1-scalability}
     \end{subfigure}
     \begin{subfigure}[b]{0.6\linewidth}
         \centering
         \includegraphics[width=\linewidth]{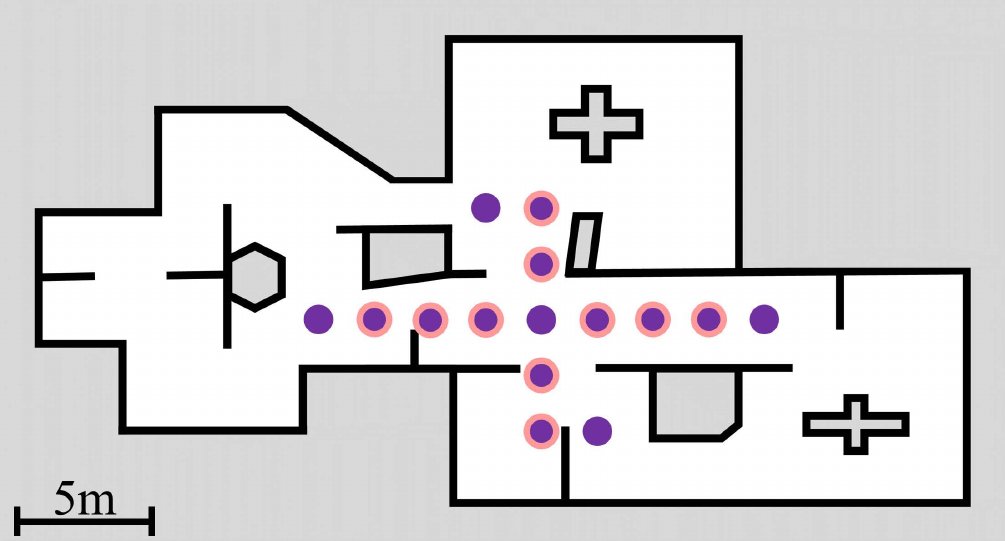}
         \caption{Map 2}
         \label{fig:map2-scalability}
     \end{subfigure}
     \begin{subfigure}[b]{\linewidth}
         \centering
         \includegraphics[width=\linewidth]{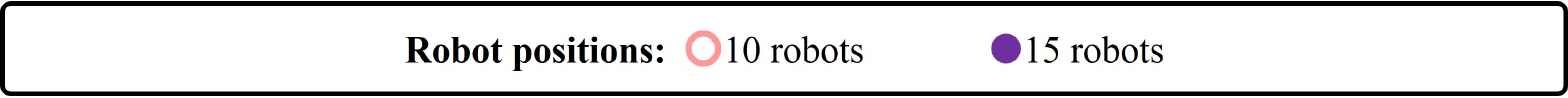}
     \end{subfigure}
        \caption{Initial positions for exploration with 10 and 15 robots}
        \label{fig:maps-scalability}
        \vspace*{-4mm}
\end{figure}

\begin{figure}[h]
     \centering
     \begin{subfigure}[b]{\linewidth}
         \centering
         \includegraphics[width=\linewidth]{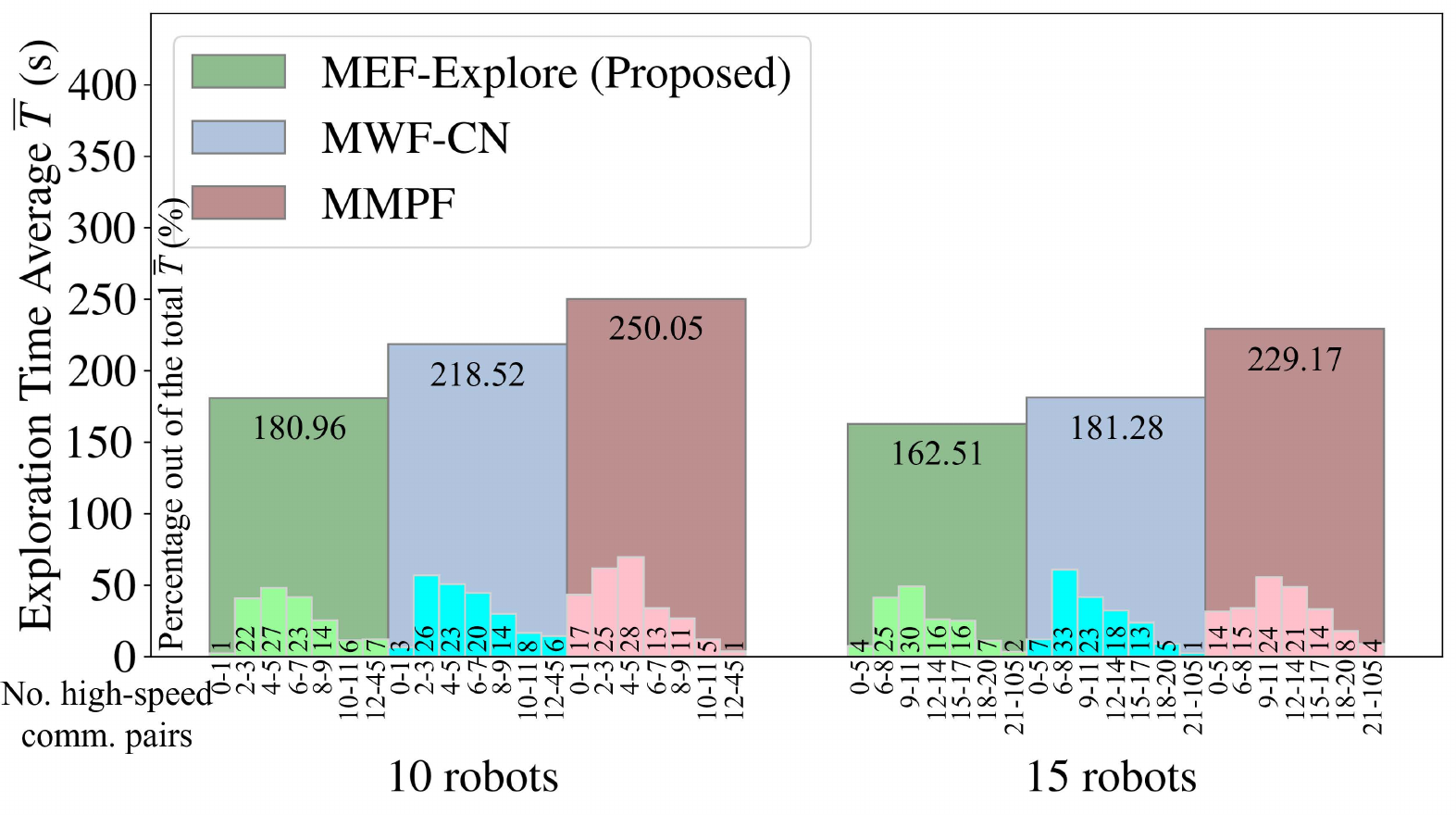}
         \caption{$\bar{T}$}
         \label{fig:map1_scalability_Ttotal}
     \end{subfigure}
     \begin{subfigure}[b]{0.49\linewidth}
         \centering
         \includegraphics[width=\linewidth]{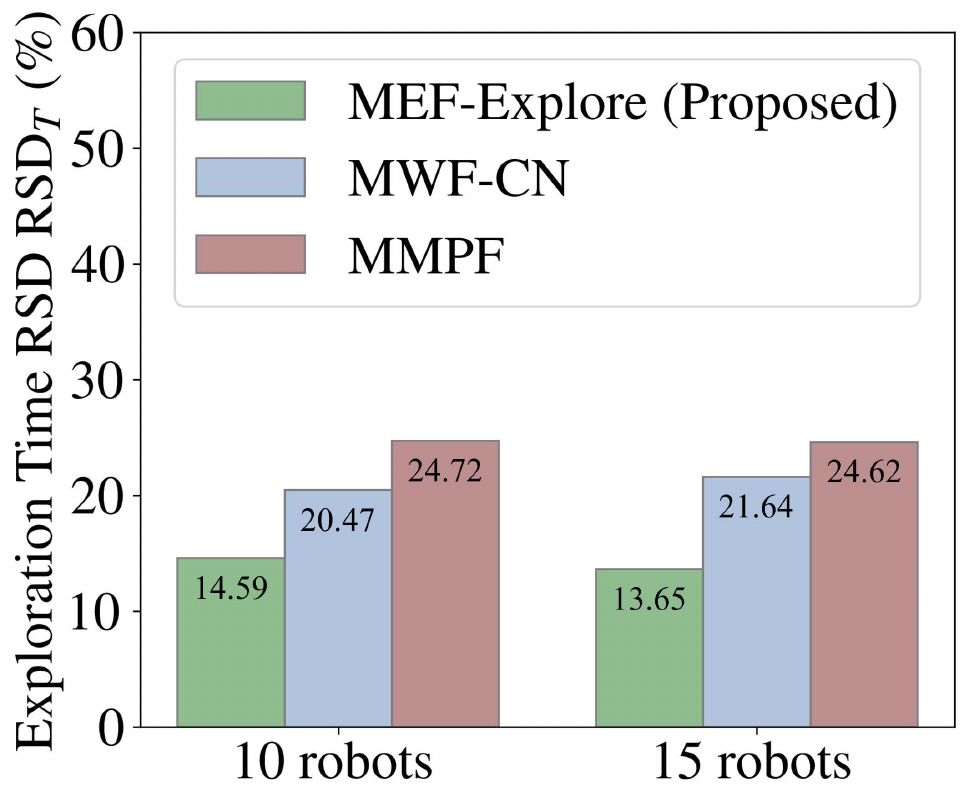}
         \caption{$\mathrm{RSD}_T$}
         \label{fig:map1_scalability_RSD}
     \end{subfigure}
     \begin{subfigure}[b]{0.49\linewidth}
         \centering
         \includegraphics[width=\linewidth]{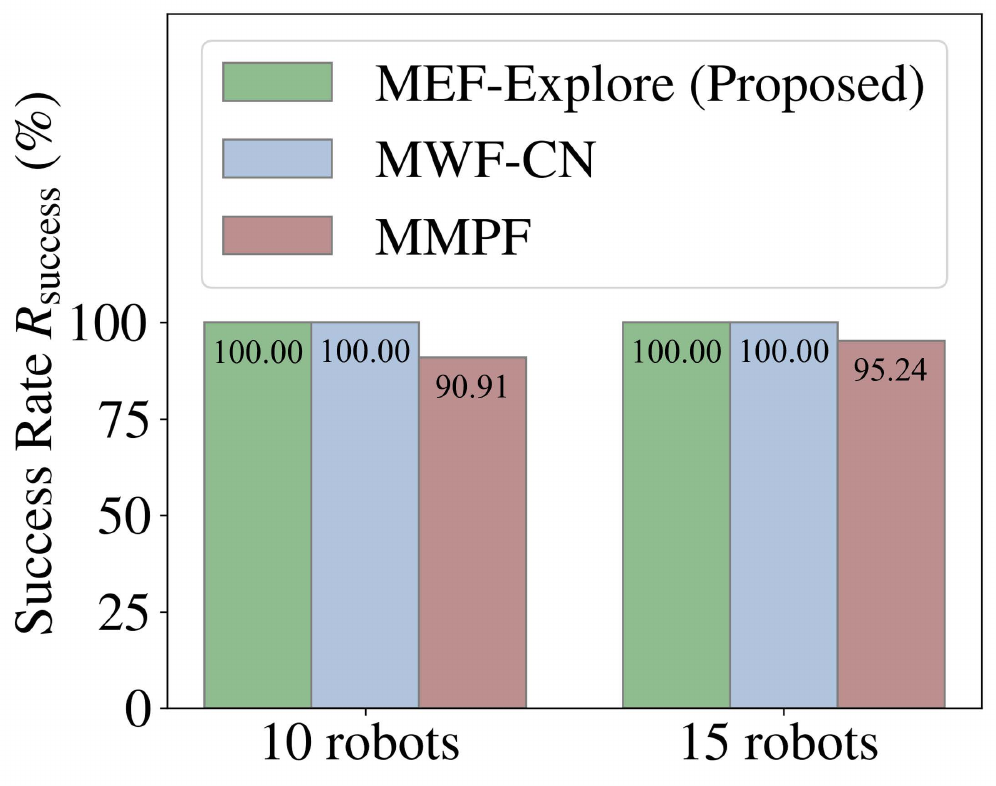}
         \caption{$R_\mathrm{Success}$}
         \label{fig:map1_scalability_Success}
     \end{subfigure}
        \caption{Simulation results on Map 1 of our proposed MEF-Explore, the MWF-CN, and the MMPF with 10 and 15 robots}
        \label{fig:scalability_map1}
\end{figure}

\subsection{Scalability with Large Numbers of Robots}
Ensuring that an exploration strategy remains effective as the number of robots increases is crucial for large-scale deployments. This subsection studies the scalability of the proposed MEF-Explore and the benchmarks, examining how the exploration performance is affected when scaling up the robot team size. To evaluate each method's scalability, we utilize 10 and 15 robots to explore Map 1 and Map 2 with $r_\mathrm{comm}=2\text{m}$. Robots' initial positions are shown in Fig. \ref{fig:maps-scalability}. 

\begin{figure}
     \centering
     \begin{subfigure}[b]{\linewidth}
         \centering
         \includegraphics[width=\linewidth]{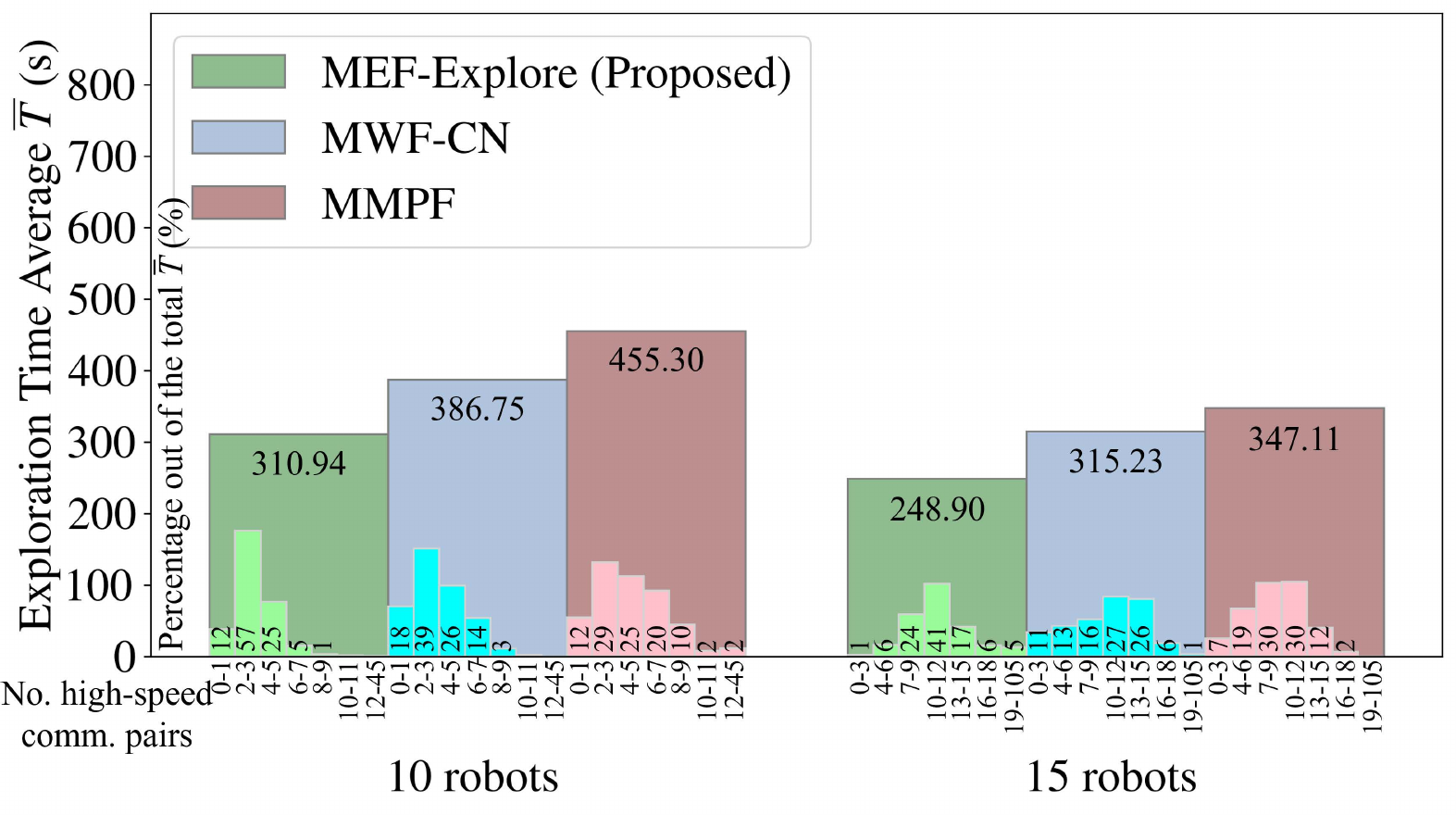}
         \caption{$\bar{T}$}
         \label{fig:map2_scalability_Ttotal}
     \end{subfigure}
     \begin{subfigure}[b]{0.49\linewidth}
         \centering
         \includegraphics[width=\linewidth]{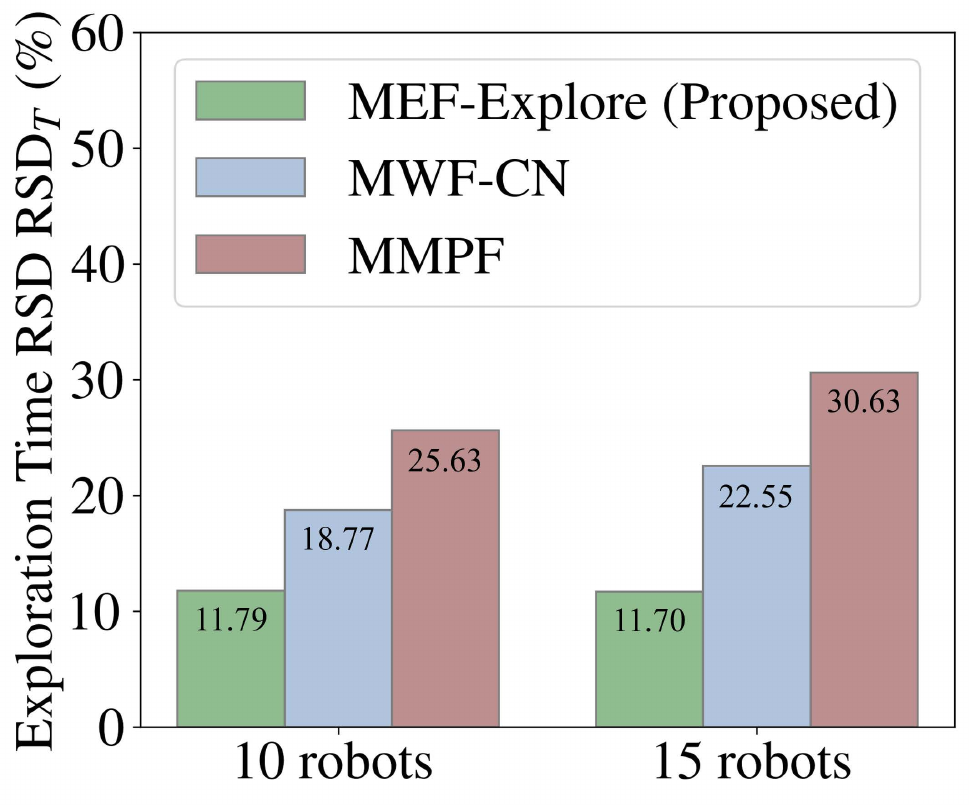}
         \caption{$\mathrm{RSD}_T$}
         \label{fig:map2_scalability_RSD}
     \end{subfigure}
     \begin{subfigure}[b]{0.49\linewidth}
         \centering
         \includegraphics[width=\linewidth]{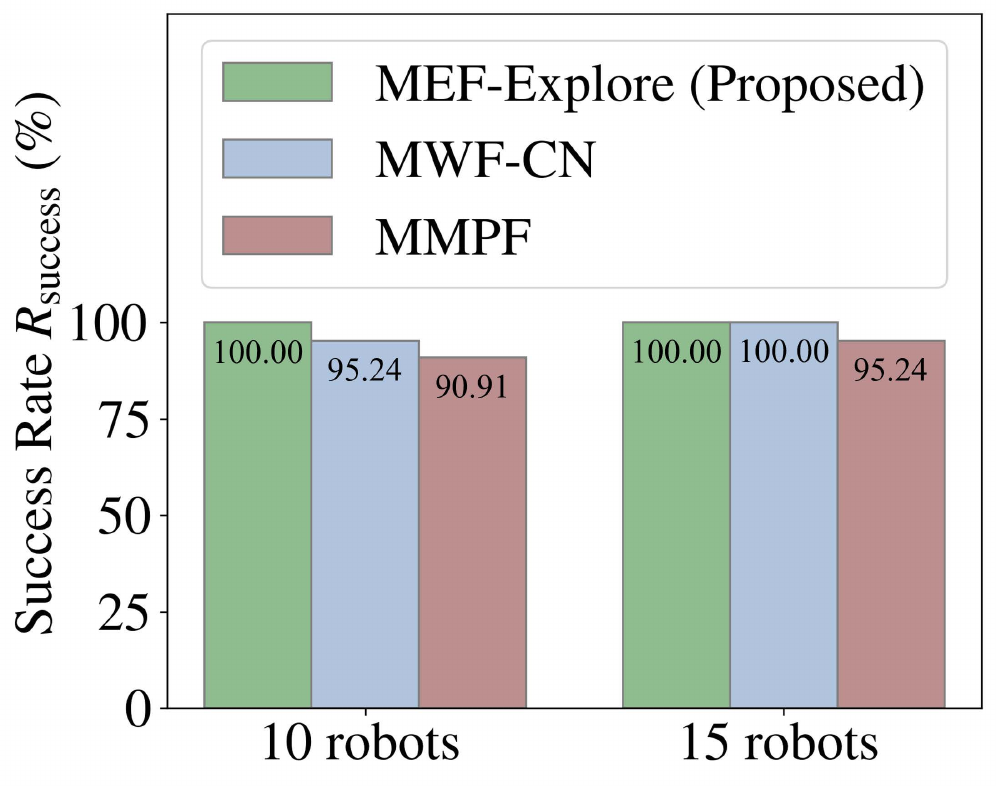}
         \caption{$R_\mathrm{Success}$}
         \label{fig:map2_scalability_Success}
     \end{subfigure}
        \caption{Simulation results on Map 2 of our proposed MEF-Explore, the MWF-CN, and the MMPF with 10 and 15 robots}
        \label{fig:scalability_map2}
\end{figure}

Note that we have attempted to implement the GVGExp with 10 and 15 robots; however, it is unable to fully explore the environments. The GVGExp was not designed to accommodate such large team sizes as its exploration strategy relies on generalized Voronoi graphs, which may become increasingly fragmented and inefficient as the number of robots grows, leading to coordination challenges. Additionally, the method imposes strict inter-robot communication requirements, which can become a limiting factor in larger teams, potentially causing delays or conflicts in task allocation. Therefore, we will evaluate only the exploration performance of the proposed MEF-Explore and the remaining benchmarks, the MWF-CN and the MMPF. The results are calculated based on 240 rounds of successful exploration (20 rounds for each scenario), and we also collect the number of unsuccessful rounds to compute $R_\mathrm{success}$.

According to Fig. \ref{fig:scalability_map1} and \ref{fig:scalability_map2}, the proposed MEF-Explore and the benchmarks successfully complete the exploration tasks in both simulated environments with 10 and 15 robots, demonstrating their scalability under large robot teams. Notably, the MEF-Explore consistently outperforms all benchmarks across all evaluated metrics. These results substantiate the performance advancements from using the proposed MEF-Explore, as previously discussed in Section \ref{subsection:benchmark}. Specifically, our novel entropies, implicit robot rendezvous, and efficient goal-assigning module continue to perform effectively even if the number of robots surges significantly.

Compared to the MWF-CN and the MMPF, for the case of Map 1, the MEF-Explore leads robots to explore 20.75\% and 38.18\% faster when using 10 robots and 11.55\% and 41.01\% faster when using 15 robots, as shown in Fig. \ref{fig:map1_scalability_Ttotal}. The exploration times are 5.88\% and 10.13\% more consistent when using 10 robots and 7.99\% and 10.97\% more consistent when using 15 robots, as shown in Fig. \ref{fig:map1_scalability_RSD}. Also, the exploration by the MEF-Explore is 9.09\% and 4.76\% more successful than the MMPF when using 10 and 15 robots, as shown in Fig. \ref{fig:map1_scalability_Success}. Similarly, for the case of Map 2, compared to the MWF-CN and the MMPF, the exploration by MEF-Explore is 24.38\% and 46.43\% faster when using 10 robots and 26.65\% and 39.46\% faster when using 15 robots, as shown in Fig. \ref{fig:map2_scalability_Ttotal}. The exploration times are 6.98\% and 13.84\% more consistent when using 10 robots and 10.85\% and 18.93\% more consistent when using 15 robots, as shown in Fig. \ref{fig:map2_scalability_RSD}. The MEF-Explore also leads to 4.76\% and 9.09\% more successful exploration compared to the MWF-CN and the MMPF when using 10 robots and 4.76\% more successful exploration compared to the MMPF when using 15 robots, as shown in Fig. \ref{fig:map2_scalability_Success}. These results confirm that the proposed MEF-Explore exhibits superior scalability and delivers the best performance when deployed with large numbers of robots.

\begin{figure}
     \centering
     \begin{subfigure}[b]{0.9\linewidth}
         \centering
         \includegraphics[width=\linewidth]{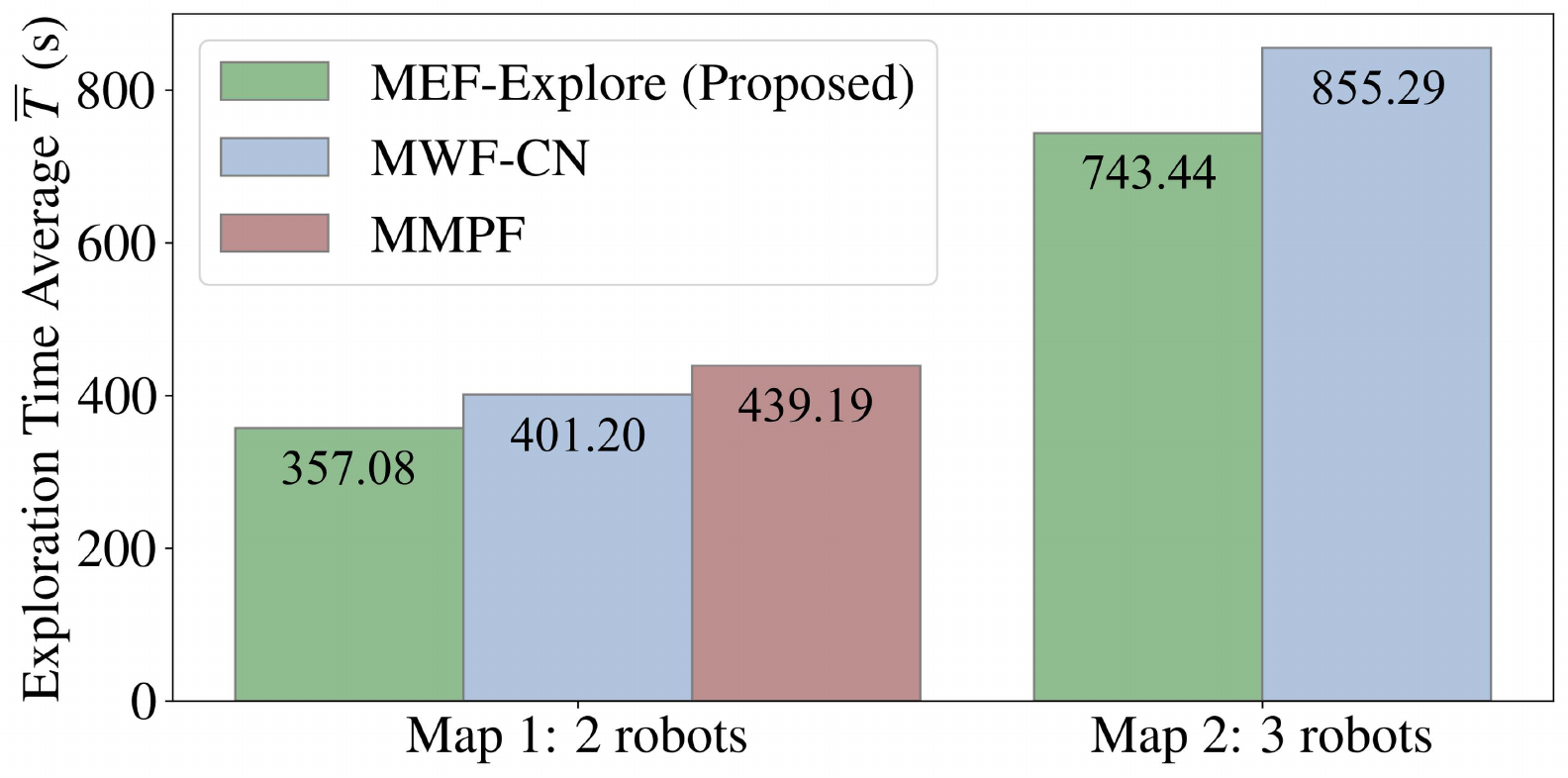}
         \caption{$\bar{T}$ (no inter-robot communication)}
         \label{fig:stability_Ttotal}
     \end{subfigure}
     \begin{subfigure}[b]{0.9\linewidth}
         \centering
         \includegraphics[width=\linewidth]{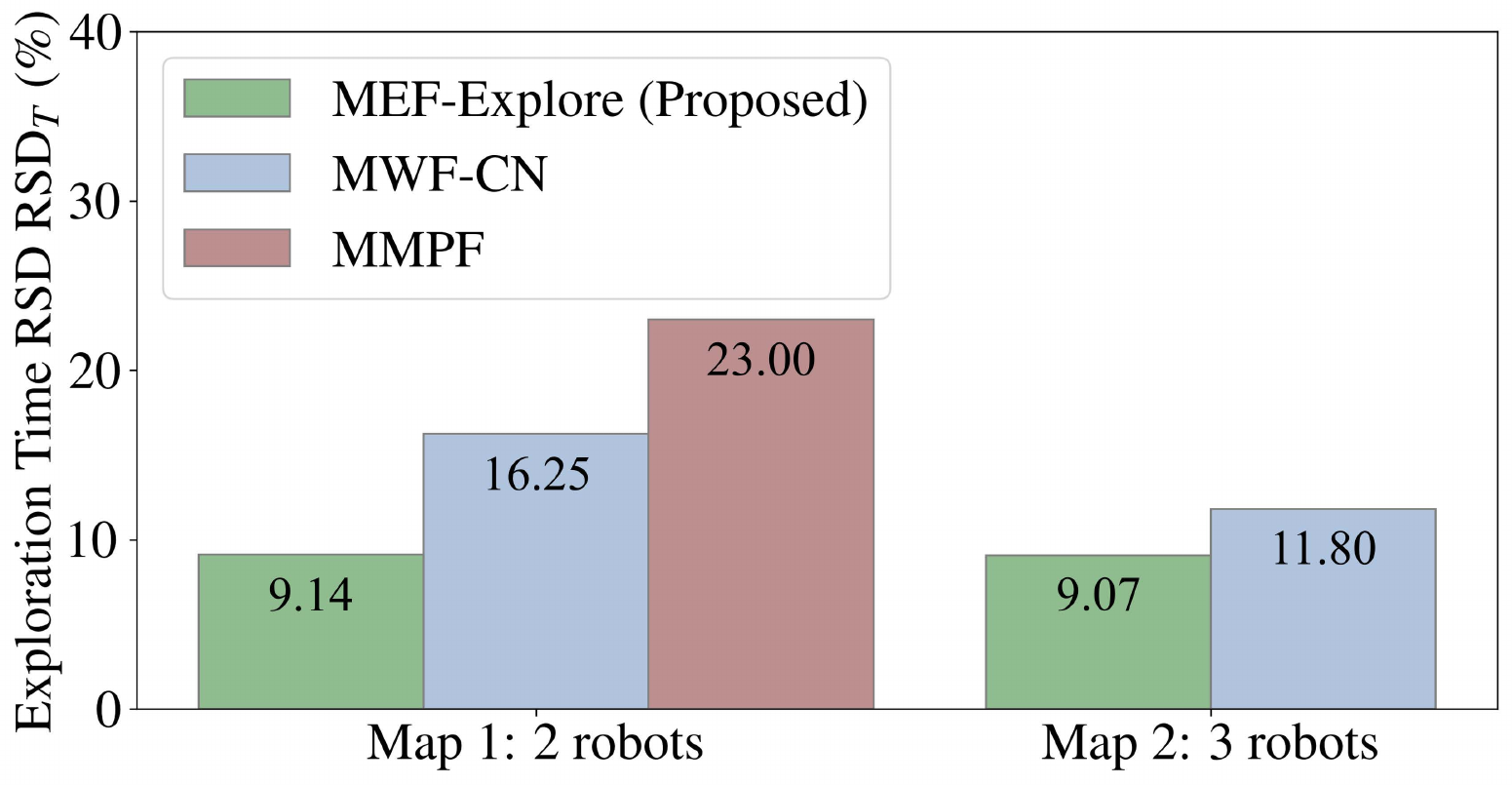}
         \caption{$\mathrm{RSD}_T$ (no inter-robot communication)}
         \label{fig:stability_RSD}
     \end{subfigure}
     \begin{subfigure}[b]{0.9\linewidth}
         \centering
         \includegraphics[width=\linewidth]{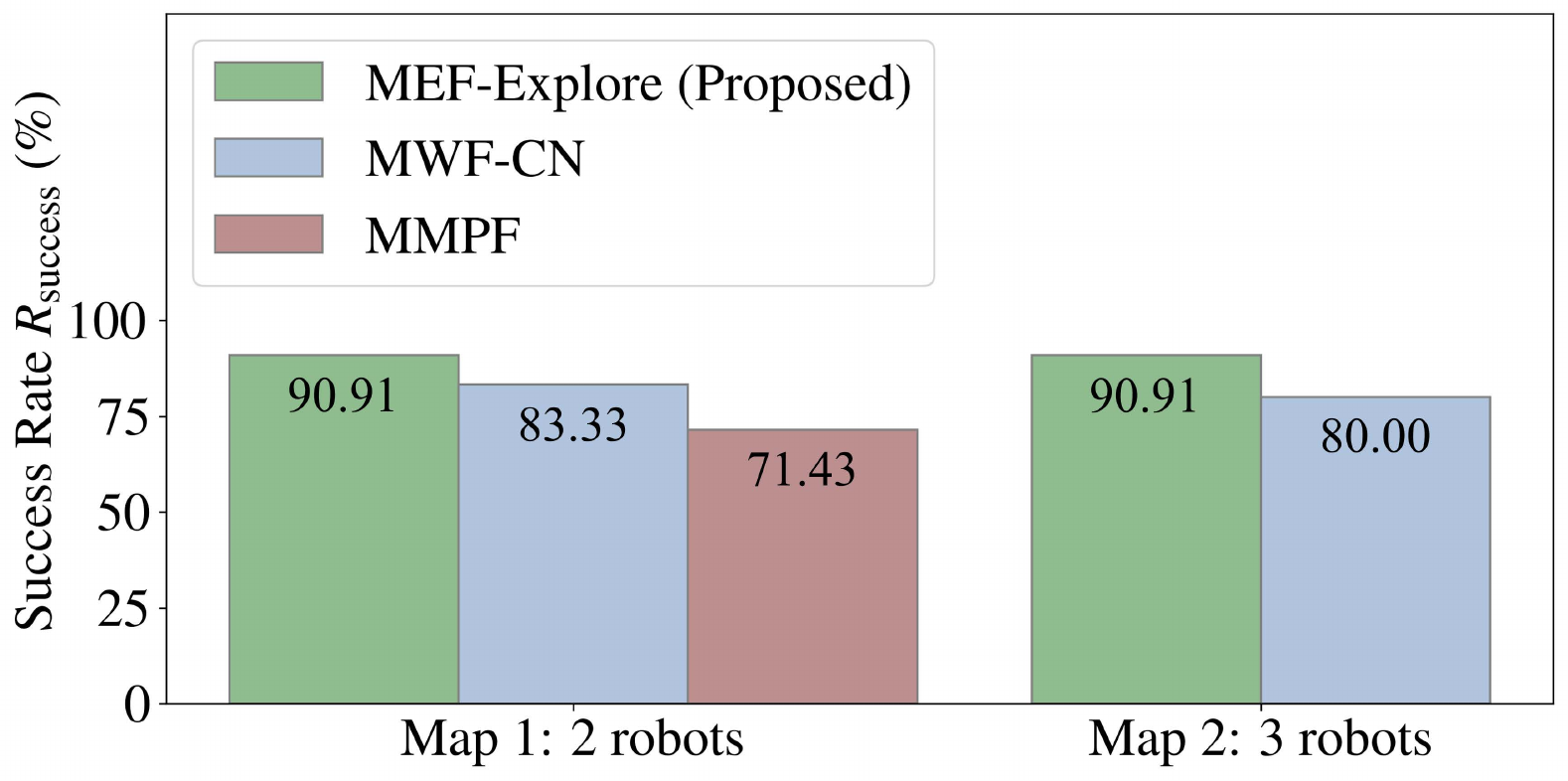}
         \caption{$R_\mathrm{Success}$ (no inter-robot communication)}
         \label{fig:stability_Success}
     \end{subfigure}
        \caption{Simulation results on Map 1 and Map 2 of our proposed MEF-Explore, the MWF-CN, and the MMPF without inter-robot communication}
        \label{fig:stability}
\end{figure}

\begin{figure*}
     \centering
     \begin{subfigure}[b]{0.2\linewidth}
         \centering
         \includegraphics[width=\linewidth]{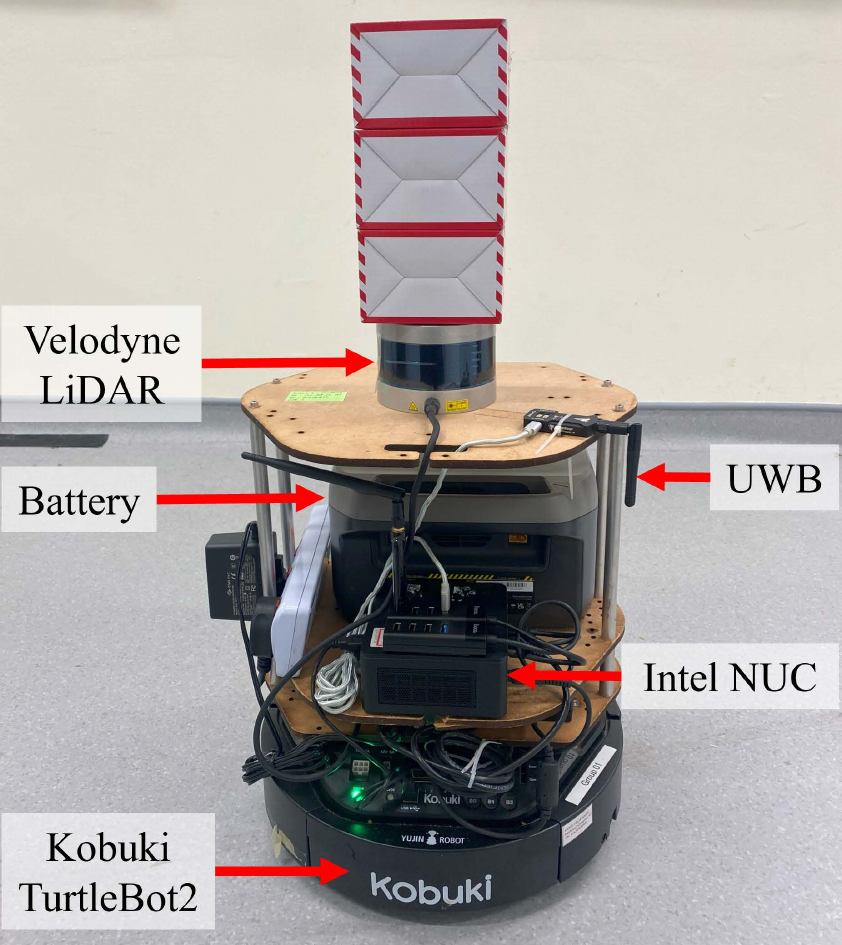}
         \caption{TurtleBot setup}
         \label{fig:turtlebot}
     \end{subfigure}
     \begin{subfigure}[b]{0.39\linewidth}
         \centering
         \includegraphics[width=\linewidth]{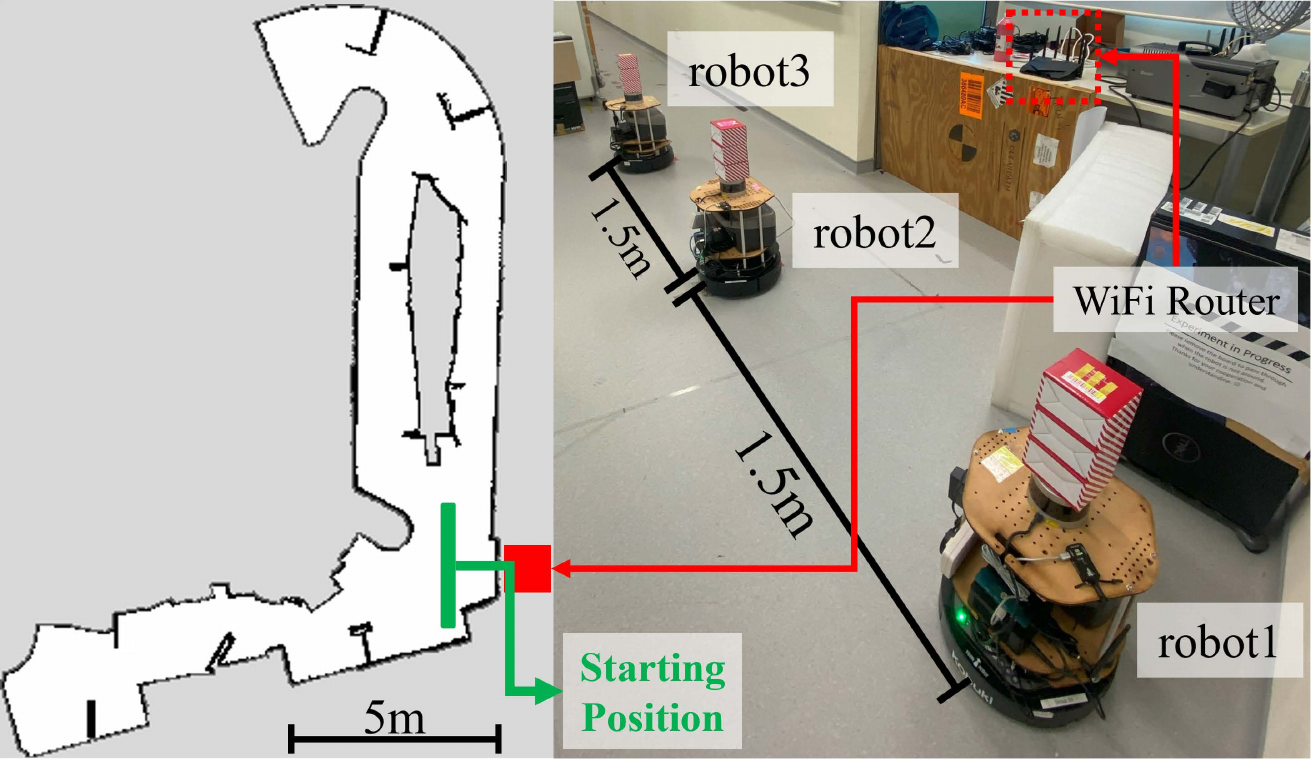}
         \caption{Real Environment}
         \label{fig:real-env}
     \end{subfigure}
        \caption{Robot and environment setup for three-robot exploration in real-world environment}
        \label{fig:robot+env}
\end{figure*}

\begin{figure*}
     \centering
     \begin{subfigure}[b]{0.45\linewidth}
         \centering
         \includegraphics[width=\linewidth]{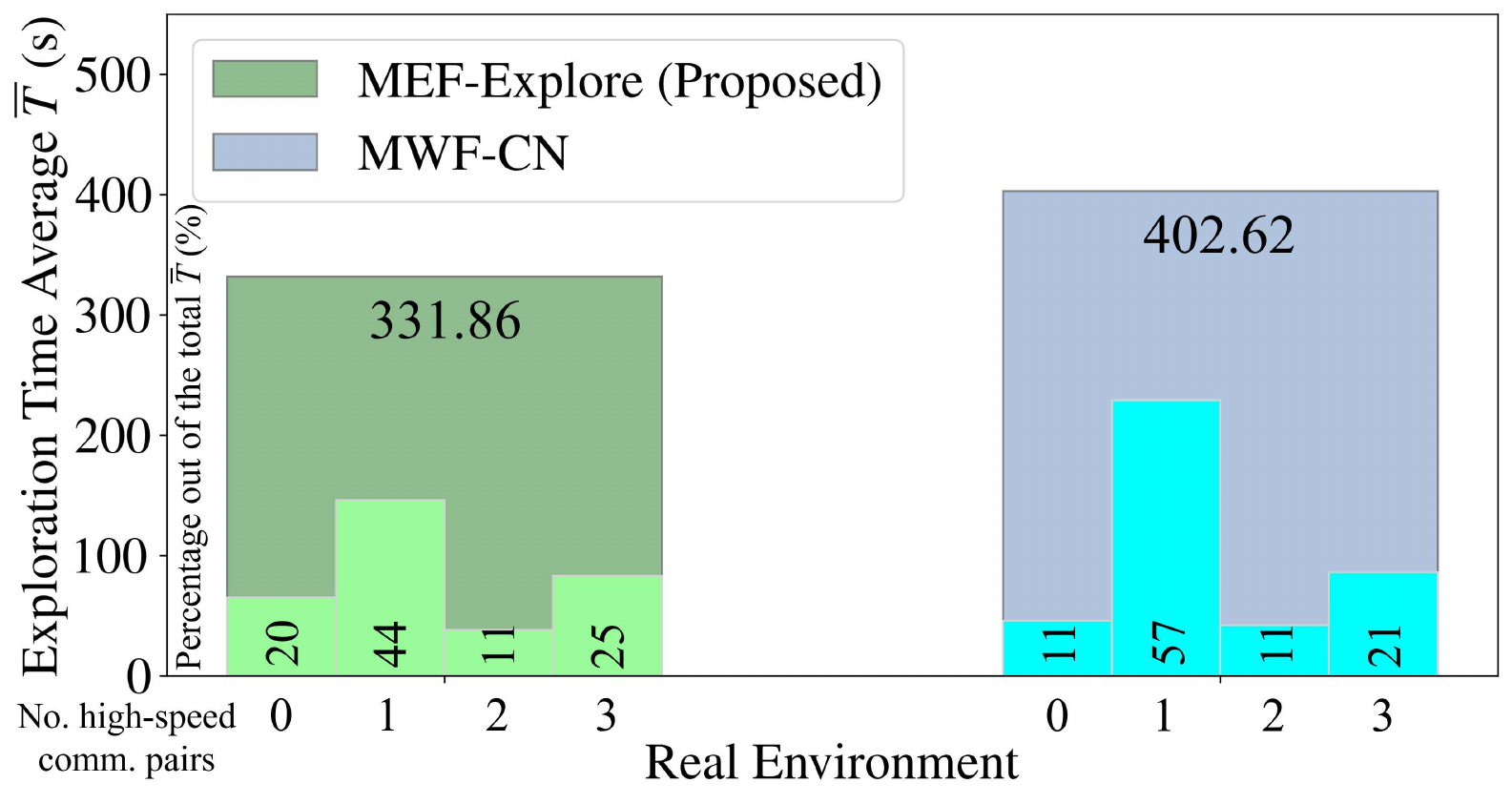}
         \caption{$\bar{T}$ (Real Environment)}
         \label{fig:real_Ttotal}
     \end{subfigure}
     \begin{subfigure}[b]{0.45\linewidth}
         \centering
         \includegraphics[width=\linewidth]{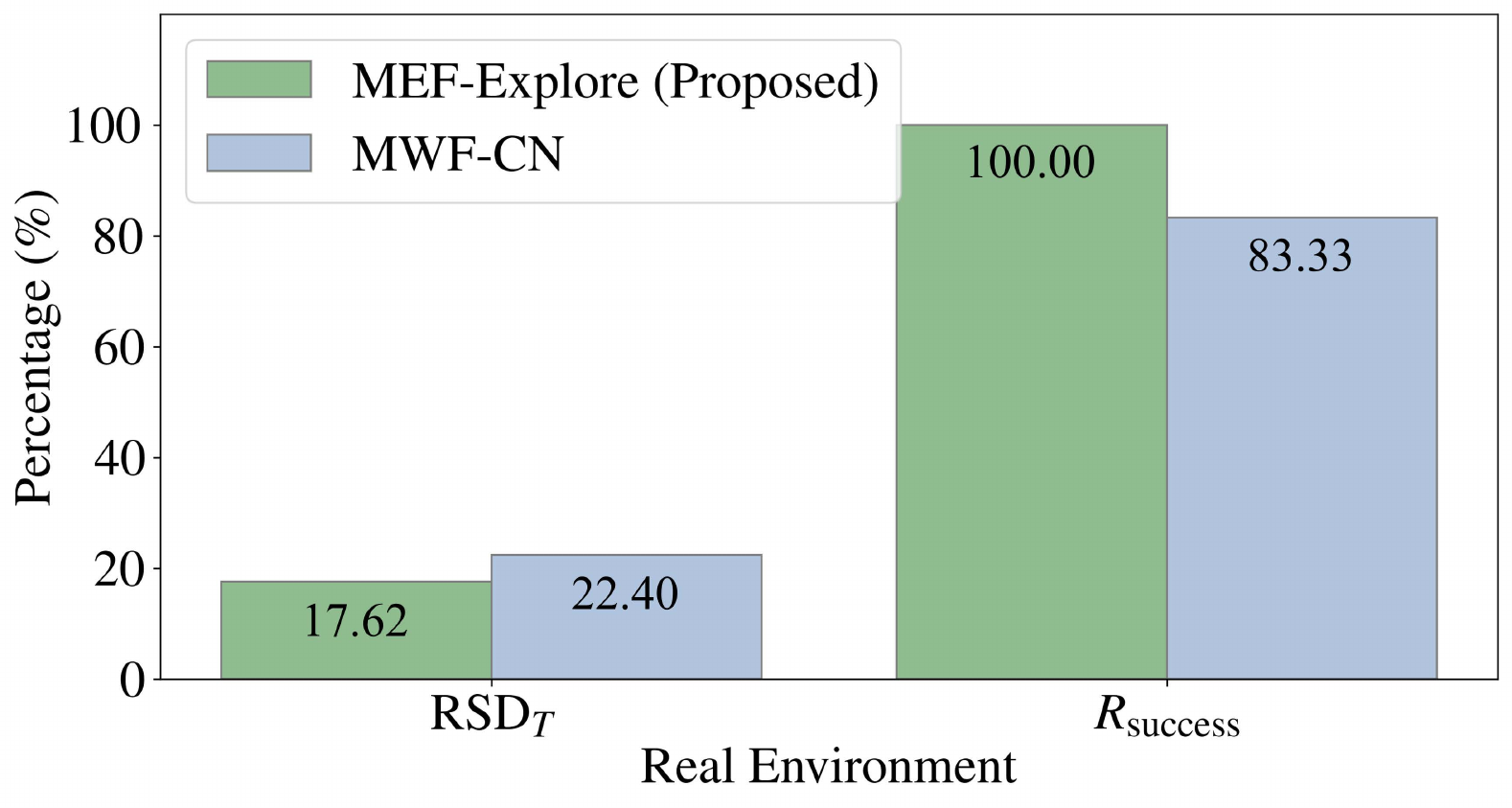}
         \caption{$\mathrm{RSD}_T$ and $R_\mathrm{Success}$ (Real Environment)}
         \label{fig:real_RSD+Success}
     \end{subfigure}
        \caption{Experimental results on Real Environment of our proposed MEF-Explore and the MWF-CN}
        \label{fig:real_results}
\end{figure*}

\subsection{Stability Under Communication Loss}
As robust exploration strategies should maintain stable performance even when inter-robot communication is unavailable, this subsection studies the impact of inter-robot communication loss on exploration stability, evaluating how effectively robots using different methods can adapt and still function under such conditions. As robots cannot communicate during exploration, they will not share their positions and local maps with each other at all. To demonstrate this clearly, we select the sparsest exploration environments in terms of the number of robots according to our previous simulation settings: Map 1 with two robots and Map 2 with three robots. Since the reliance on inter-robot communication in these scenarios is essential, making them stringent tests of whether robots using the proposed MEF-Explore and the benchmarks can function independently and stably if they lack such communication.

For the purpose of evaluation, the map merging is processed separately to check if the environments are fully explored, but this merged map is not accessible to any robots. Also, given that the GVGExp formulates its inter-robot communication based on graph structures from multiple robots, it inherently relies on such communication to support the exploration process. Its ability to coordinate and explore effectively can be significantly disrupted if such communication is unavailable. Therefore, we will focus only on the proposed MEF-Explore, the MWF-CN, and the MMPF. Note that we examine the exploration performance based on 20 successful rounds for each scenario, which means 100 successful rounds for all cases, and the number of unsuccessful rounds is also collected to calculate $R_\mathrm{success}$.

According to Fig. \ref{fig:stability}, the results reveal declines in exploration performance, which are longer exploration time $\bar{T}$ and lower success rate $R_\mathrm{success}$, compared to those with inter-robot communication in Fig. \ref{fig:map1_sim_results} and \ref{fig:map2_sim_results}. The main reason is that non-communicating robots do not share any information, including their positions and maps. Hence, each robot must rely solely on its own local information, which can lead to a lack of fully efficient goal selection. Furthermore, as each robot is entirely unaware of the others' explored areas, this leads to the likelihood of overlaps in exploration efforts, which ultimately escalates the exploration time and the risk of failures.

However, the extent of the degradation is not too excessive for the case of the proposed MEF-Explore, and it still outperforms the benchmarks in all aspects as our reformed entropies integrating both information entropies and potential fields enable more informed and adaptive decision-making compared to the MWF-CN and the MMPF that solely rely on potential fields. As the exploration goals obtained from the MEF-Explore are also assigned to robots more strategically, the MEF-Explore leads to more excellent overall stability.

Robots using the MEF-Explore can finish exploring Map 1 at 12.36\% and 23.00\% faster, as shown in Fig. \ref{fig:stability_Ttotal}, 7.12\% and 13.86\% more consistent, as shown in Fig. \ref{fig:stability_RSD}, and 7.58\% and 19.48\% more successful, as shown in Fig. \ref{fig:stability_Success}, compared to those using the MWF-CN and the MMPF, respectively. While robots using the MMPF cannot have any successful rounds when exploring Map 2, those using the proposed MEF-Explore still manage to explore at 15.05\% faster, 2.73\% more consistent, and 10.91\% more successful than the MWF-CN ones, as shown in Fig. \ref{fig:stability_Ttotal}, \ref{fig:stability_RSD}, and \ref{fig:stability_Success}. These findings demonstrate that the proposed MEF-Explore remains robust under communication loss and can effectively adapt to maintain stability and efficiency in exploration.

\section{Real-World Deployment}
\label{section:real-world}
Based on the performance of each method in simulation, our MEF-Explore and the MWF-CN have better results than the MMPF and the GVGExp. Thus, we select these two methods to deploy further on real robots. As shown in Fig. \ref{fig:turtlebot}, we utilize three TurtleBot2 run on a Kobuki base, each equipped with an Intel NUC i5 as the computing unit and a Velodyne LiDAR sensor (scanning range restricted to 7m). As we intend to emulate the communication-constrained environments, we choose LinkTrack's ultra-wideband node (UWB) to serve low-speed communication and WiFi 2.4Ghz (bounded by $r_\mathrm{comm}=4\text{m}$) to serve high-speed communication. The experimental environment is an indoor area of size 223.0$\text{m}^2$ with boxes and partitions as additional obstacles and walls. At the beginning of each round, the robots are placed at the starting position, as shown in Fig. \ref{fig:real-env}. Note that we use the same parameter set, localization, mapping, and navigation stacks as simulation. We also assess the exploration performance using the same evaluation metrics mentioned in Section \ref{subsection:metrics}.

The results in Fig. \ref{fig:real_results} are calculated based on ten successful exploration rounds, and the number of unsuccessful rounds is collected to measure $R_\mathrm{success}$. The well-performing manners are collectively reflected in the exploration enhancement: 21.32\% faster $\bar{T}$, 4.78\% lower $\mathrm{RSD}_T$, and 16.67\% higher $R_\mathrm{success}$. These outcomes from real-robot experiments closely align with the simulation results, in which our MEF-Explore outperforms the MWF-CN for all the metrics. The robots using our proposed method can explore the environment rapidly and consistently, and the rendezvous strategy is practical for triggering robot meetings to exchange their maps at appropriate times. It is also noticeable that the proposed goal-assigning module chooses a more appropriate goal to traverse compared to the MWF-CN, resulting in faster exploration time. These findings exhibit the strong robustness of our MEF-Explore in real-world environment exploration.

\section{Potential Challenges}
\label{section:challenges}
In this section, we outline potential challenges that could arise and discuss factors that influence the performance. These aspects will provide an understanding of the practical implications and possible areas for further investigation. While our proposed MEF-Explore has demonstrated promising results, there can be some specific challenges that occur during experiments, as follows:
\begin{enumerate}
    \item \textbf{Positioning and UWB ranging errors:} In general, mapping processes depend on sensors, such as LiDARs, whose measurements are relative to robot positions. Due to inherent sensor noise, positioning errors can accumulate over time if the loop closure technique is not applied, as studied in \cite{PosErrors}. In systems employing UWBs like our case, ranging errors further compound the challenges. We notice that UWBs utilized to check distances between robots to enable inter-robot map merging can have inevitable ranging errors due to the occlusions of obstacles in the environment. As a result, the mentioned errors can consequently affect each robot's merged map during exploration. For efficient and reliable exploration performance, further considerations can be made to mitigate these errors' impacts. However, according to our experiments, we have not experienced noticeable downgrades in the maps and overall exploration so far, presumably because the magnitude of the positioning errors is generally slight, and the UWB ranging errors are likely minor: below 0.5m, which aligns with findings from \cite{UWB}.  
    \item \textbf{Congestion due to large numbers of robots:} In this paper, we focus on multi-robot systems rather than swarm systems that typically involve large-scale deployment with reactive behaviors for exploration, as studied in \cite{Swarm}. However, our proposed method is conceptually scalable and consistently outperforms the benchmarks for all cases of numbers of robots in both simulation and real-world deployment. As our realistically simulated robots closely resemble real robots that are rigid bodies with mass, inertia, and volume, if an excessive number of them is deployed, they occasionally face congestion issues due to physical overlaps among robots, as investigated in \cite{Congestion}. Since congestion introduces unintended challenges, which are not the focus of this paper, we observe that the area allocated per robot should be at least 50$\text{m}^2$ across all environments to prevent overcrowding of robots. In other words, we deploy the number of robots proportional to the area of each environment, ensuring that each robot is allocated sufficient space to operate effectively. This approach minimizes overlap between robots and consequently reduces the likelihood of congestion. For future research, determining the optimal number of robots for exploration and conducting an exhaustive analysis of the method's scalability in varying environments also suggest intriguing avenues.
    \item \textbf{Considerations in robots' initial positions:} Even though the proposed method and the benchmarks are generally designed to be able to start exploration at any location in the environment, in reality, there should be considerations regarding where they should be located for both simulation and real-world deployment. In accordance with related studies by \cite{InitPos}, to ensure that the results realistically reflect the actual performance of each method without environment-induced constraints, we should carefully select these positions depending on the environment's characteristics and avoid beginning with highly obstructed areas. The comparison of exploration performance under diverse initial robot positions, including potentially suboptimal configurations, falls outside the scope of this paper. Nonetheless, investigating this aspect is an interesting topic for further research.
\end{enumerate}

\section{Conclusion and Future Work}
\label{section:conclusion}
This paper proposes the MEF-Explore, a novel distributed communication-constrained multi-robot exploration method. It comprises the two-layer inter-robot communication-aware information-sharing strategy and the entropy-field-based exploration strategy. Specifically, the proposed information-sharing strategy allows robots to exchange maps based on available communication. The dynamic graph representation also provides comprehensive information on each robot's communication status and local map. The proposed exploration strategy, which consists of novel entropy forms for frontiers and robots, combines the benefits of both the attractiveness inspired by the potential field and the entropy to deal with uncertainty caused by communication-constrained situations. Our well-constructed entropies also trigger implicit robot rendezvous to boost inter-robot map merging. Moreover, the duration-adaptive goal-assigning module gives robots better goal selection than the benchmark. As demonstrated by simulation and real-robot experiments, the robots using our proposed MEF-Explore outperform the other methods in all scenarios in terms of short and consistent exploration time and high success rate. For further studies, the strategy for heterogeneous multi-robot systems to collaboratively explore unknown environments under communication constraints is also an interesting research topic, as different types of robots can be responsible for various roles based on their capabilities. Together with the comparison with relevant works, this direction provides a more comprehensive evaluation of the method's effectiveness in various settings. Incorporating adaptability analysis would also offer valuable insights into how exploration methods perform under varying communication conditions, such as different communication intensities and devices. This examination will ensure the method's robustness and practicality in different scenarios.

\ifCLASSOPTIONcaptionsoff
  \newpage
\fi

\bibliographystyle{IEEEtran}
\bibliography{IEEEabrv,Bibliography}

{\appendix
In this Appendix, we provide the mathematical properties and proof of our proposed method's components mentioned in the previous sections.

\subsection{Properties of $\boxplus$}\label{appendix:prop}
Recall that $\mathscr{M}_i$, $\mathscr{M}_j$, $\mathscr{M}_k$ are the local maps of robots $i$, $j$, $k$, respectively. We have an operator $\boxplus: \mathbb{R}^2 \times \mathbb{R}^2 \to \mathbb{R}^2$ to represent map merging. For mathematical convenience, we assume that there is no error from merging. Hence, we have the following properties:
\begin{align}
    \text{Idempotence: } &\mathscr{M}_i \boxplus \mathscr{M}_i = \mathscr{M}_i \text{.}\\
    \text{Commutativity: } &\mathscr{M}_i \boxplus \mathscr{M}_j = \mathscr{M}_j \boxplus \mathscr{M}_i \text{.}\\
    \text{Associativity: } &(\mathscr{M}_i \boxplus \mathscr{M}_j) \boxplus \mathscr{M}_k = \mathscr{M}_i \boxplus (\mathscr{M}_j \boxplus \mathscr{M}_k) \text{.}
\end{align}

\subsection{Proof of Rendezvous Behavior} \label{appendix:rendezvous}
We analyze the rendezvous behavior by presenting the following theorem with detailed proof.
\begin{theorem}
    If robot $i$ almost does not have any frontier on its local map, it will rendezvous with other robots within its sensor range.
\end{theorem}

\begin{proof}
    Assume that robot $i$ almost does not have any frontier on its local map, i.e., $N_C \to 1$ and $C_q \to 0$. Using eq. (\ref{eq:H_total}), we have
    \begin{equation}
        \lim_{(N_C,C_q) \to (1,0)} H_\mathrm{total}(i,p,q) = H_r(i,p) + \lim_{(N_C,C_q) \to (1,0)} H_f(p,q)\text{.}
    \end{equation}
    Then, using eq. (\ref{eq:H_f}), we consider
    \begin{align*}
        \lim_{(N_C,C_q) \to (1,0)} H_f(p,q) 
        &= -\frac{k_f(N_r)}{d^*(p,q)}\lim_{C_q \to 0} \frac{\log C_q}{\frac{1}{C_q}} \\
        &= -\frac{k_f(N_r)}{d^*(p,q)}\lim_{C_q \to 0} \frac{\frac{1}{C_q}}{-\frac{1}{C_q^2}} \footnotesize\text{ [L'Hôpital's rule]}\\
        &= \frac{k_f(N_r)}{d^*(p,q)}\lim_{C_q \to 0} C_q = 0\text{.}
    \end{align*}
Hence, we have
\begin{equation}
    \lim_{(N_C,C_q) \to (1,0)} H_\mathrm{total}(i,p,q) = H_r(i,p)\text{,}
\end{equation}
which means using eq. (\ref{eq:g_new}), we have
\begin{equation}
     {g}_\mathrm{new}(i) = \argmin_p H_r(i,p)
\end{equation}
when $N_C \to 1$ and $C_q \to 0$. According to eq. (\ref{eq:H_r}), since $H_r(i,p)$ exists in robot $i$'s surroundings, robot $i$ itself and other robots within its sensor range $d_s$ will then be attracted to travel to each other's positions, i.e., rendezvous.
\end{proof}}

\begin{IEEEbiography}
[{\includegraphics[width=1in,height=1.25in,clip,keepaspectratio]{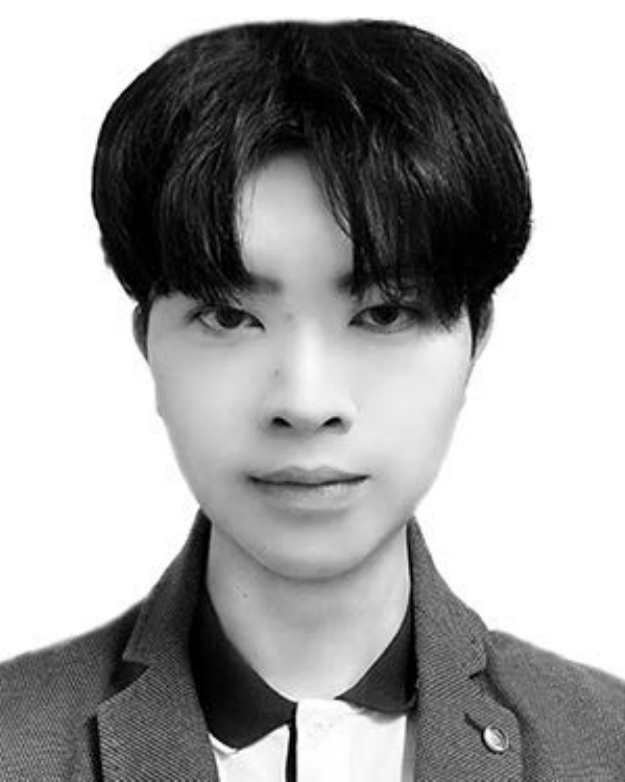}}]{Khattiya Pongsirijinda} received the B.Sc. degree in Mathematics from Silpakorn University, Thailand, in 2019 and the M.Sc. degree in Data Science from the Skolkovo Institute of Science and Technology, Russia, in 2021. He was also a research student at the University of Nebraska–Lincoln, USA, in 2018. He is currently pursuing his Ph.D. degree at the Singapore University of Technology and Design, Singapore, under the supervision of Prof. Chau Yuen. His current research interests include multi-robot exploration and applied mathematics in robotics.
\end{IEEEbiography}

\begin{IEEEbiography}
[{\includegraphics[width=1in,height=1.25in,clip,keepaspectratio]{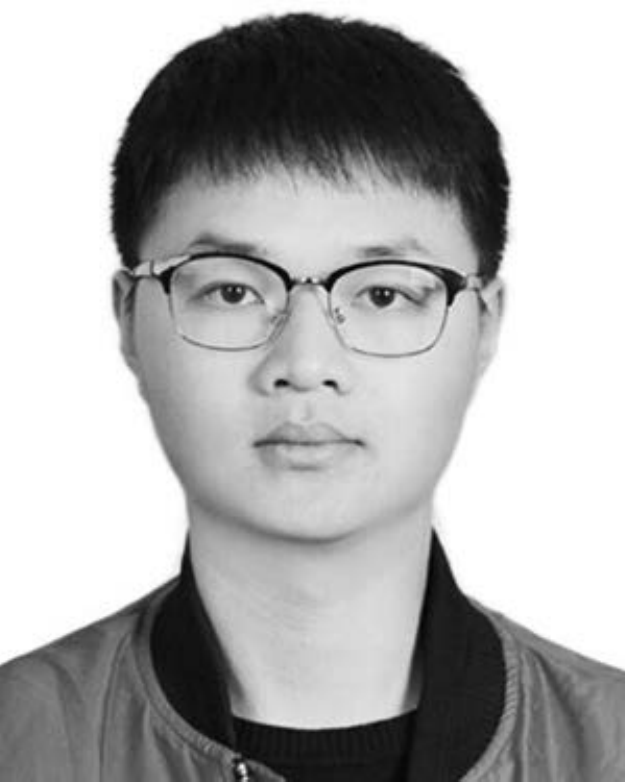}}]{Zhiqiang Cao}
received the B.S. and M.A.Sc. degrees from the Southwest University of Science and Technology, Mianyang, China, in 2019 and 2022, respectively. He is currently pursuing the Ph.D. degree with the Singapore University of Technology and Design, Singapore, under the supervision of Prof. Chau Yuen. His current research interests include multi-robot systems, collaborative localization, and distributed SLAM systems.
\end{IEEEbiography}

\begin{IEEEbiography}
[{\includegraphics[width=1in,height=1.25in,clip,keepaspectratio]{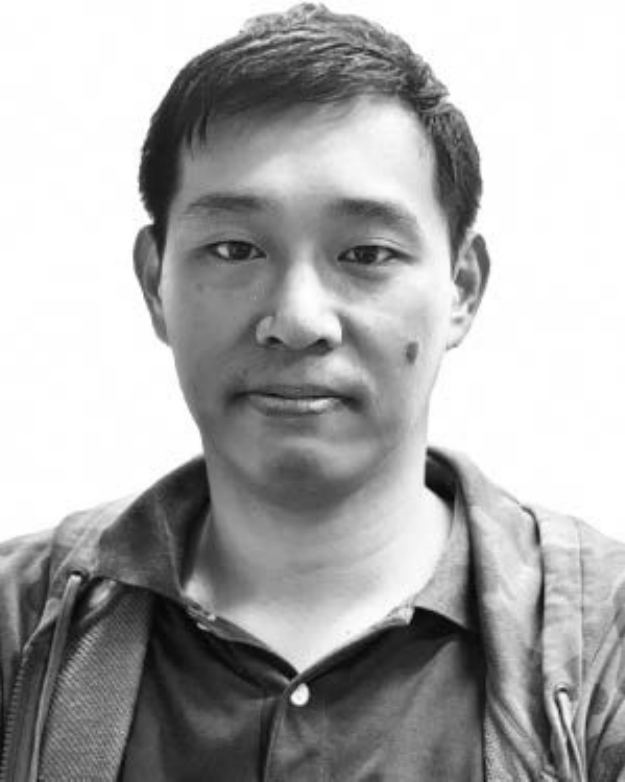}}]{Billy Pik Lik Lau} (Senior Member, IEEE) received the B.Sc. and M.Phil. degrees in computer science from Curtin University, Perth, WA, Australia, in 2010 and 2014, respectively, and the Ph.D. degree from Singapore University of Technology and Design (SUTD), Singapore, in 2021. He is currently working as a Research Fellow with Temasek Laboratory, SUTD. His Ph.D. research work includes smart city, the Internet of Things, data fusion, urban science, and big data analysis. His current research interests include integration of data fusion, urban science, multi-agent system design, multi-robotic collaboration, and robotic exploration.
\end{IEEEbiography}

\begin{IEEEbiography}
[{\includegraphics[width=1in,height=1.25in,clip,keepaspectratio]{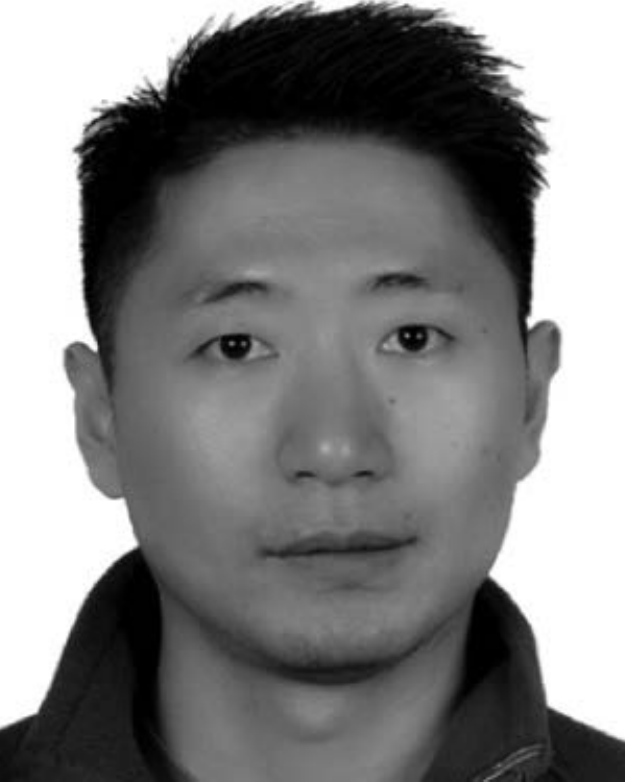}}]{Ran Liu}
(Senior Member, IEEE) received the B.S. degree from the Southwest University of Science and Technology, China, in 2007, and the Ph.D. degree from the University of Tüebingen, Germany, in 2014. From 2015 to 2024, he was a Research Fellow at Singapore University of Technology and Design. He is currently working as a Senior Research Fellow at Nanyang Technological University. His research interests include robotics and SLAM.
\end{IEEEbiography}

\begin{IEEEbiography}
[{\includegraphics[width=1in,height=1.25in,clip,keepaspectratio]{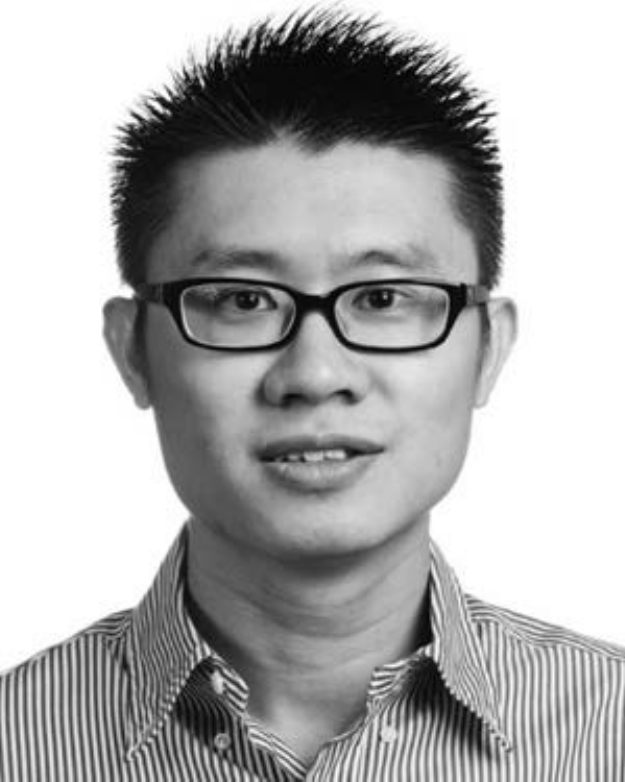}}]{Chau Yuen}
(S02-M06-SM12-F21) received the B.Eng. and Ph.D. degrees from Nanyang Technological University, Singapore, in 2000 and 2004, respectively. He was a Post-Doctoral Fellow with Lucent Technologies Bell Labs, Murray Hill, in 2005. From 2006 to 2010, he was with the Institute for Infocomm Research, Singapore. From 2010 to 2023, he was with the Engineering Product Development Pillar, Singapore University of Technology and Design. Since 2023, he has been with the School of Electrical and Electronic Engineering, Nanyang Technological University, currently he is Provost’s Chair in Wireless Communications, Assistant Dean in Graduate College, and Cluster Director for Sustainable Built Environment at ER@IN.

Dr. Yuen received IEEE Communications Society Leonard G. Abraham Prize (2024), IEEE Communications Society Best Tutorial Paper Award (2024), IEEE Communications Society Fred W. Ellersick Prize (2023), IEEE Marconi Prize Paper Award in Wireless Communications (2021), IEEE APB Outstanding Paper Award (2023), and EURASIP Best Paper Award for JOURNAL ON WIRELESS COMMUNICATIONS AND NETWORKING (2021).

Dr. Yuen current serves as an Editor-in-Chief for Springer Nature Computer Science, Editor for IEEE TRANSACTIONS ON VEHICULAR TECHNOLOGY, IEEE TRANSACTIONS ON NEURAL NETWORKS AND LEARNING SYSTEMS, and IEEE TRANSACTIONS ON NETWORK SCIENCE AND ENGINEERING, where he was awarded as IEEE TNSE Excellent Editor Award 2024 and 2022, and Top Associate Editor for TVT from 2009 to 2015. He also served as the guest editor for several special issues, including IEEE JOURNAL ON SELECTED AREAS IN COMMUNICATIONS, IEEE WIRELESS COMMUNICATIONS MAGAZINE, IEEE COMMUNICATIONS MAGAZINE, IEEE VEHICULAR TECHNOLOGY MAGAZINE, IEEE TRANSACTIONS ON COGNITIVE COMMUNICATIONS AND NETWORKING, and ELSEVIER APPLIED ENERGY.

He is listed as Top 2\% Scientists by Stanford University, and also a Highly Cited Researcher by Clarivate Web of Science from 2022. He has 4 US patents and published over 500 research papers at international journals.
\end{IEEEbiography}

\vfill
\pagebreak

\begin{IEEEbiography}
[{\includegraphics[width=1in,height=1.25in,clip,keepaspectratio]{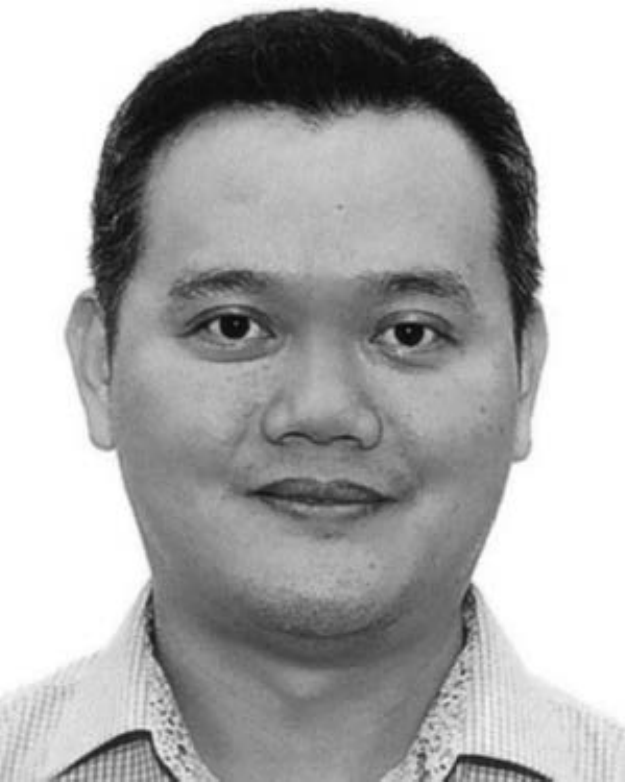}}]{U-Xuan Tan}
(Senior Member, IEEE) received the B.Eng. degree in mechanical and aerospace engineering and Ph.D. degree in mechatronics from Nanyang Technological University, Singapore, in 2005 and 2010, respectively. 

From 2009 to 2011, he did his postdoctoral research at the University of Maryland, College Park, MD, USA. From 2012 to 2014, he was a Lecturer with the Singapore University of Technology and Design, Singapore, where he took up a research intensive role as Assistant Professor in 2014, and has been promoted to Associate Professor since 2021. He is currently the Chair of SUTD Institutional Review Board and is also a Senior Visiting Academician with Changi General Hospital. His research interests include mechatronics, sensing and control, onsite micromanipulation and microsensing, sensing and control technologies for human–robot interaction, and interdisciplinary teaching.

\end{IEEEbiography}

\end{document}